\documentclass[11pt]{article}
\usepackage{fullpage,url,hyperref,lineno,microtype,subcaption,soul}
\usepackage{amssymb,amsmath,amsthm,mathtools}
\hypersetup{
	colorlinks=true,       
	linkcolor=blue,        
	citecolor=blue,        
	filecolor=magenta,     
	urlcolor=blue
}
\usepackage{authblk}
\usepackage{multirow}
\usepackage{optidef}
\usepackage{cite}
\usepackage[margin=1in]{geometry}
\usepackage{cancel}
\usepackage{booktabs}       
\usepackage{color,xcolor}
\usepackage[ruled,linesnumbered]{algorithm2e}
\SetArgSty{textup}

\SetKwInput{KwInput}{Input}
\SetKwInput{KwOutput}{Output}

\newtheorem{theorem}{Theorem}[section]
\newtheorem{proposition}[theorem]{Proposition}

\newtheorem{lemma}[theorem]{Lemma}

\newtheorem{assumption}{Assumption}
\theoremstyle{remark}
\newtheorem{remark}{Remark}[section]

\newtheorem{conjecture*}{Conjecture}
\theoremstyle{plain}

\newcommand{\calL}{\mathcal{L}}
\newcommand{\calH}{\mathcal{H}}

\newcommand{\calF}{\mathcal{F}}
\newcommand{\calV}{\mathcal{V}}

\newcommand{\E}{\mathbb{E}}
\newcommand{\R}{\mathbb{R}}

\newcommand{\SVD}{\text{SVD}}

\newcommand{\Unif}{\rm{Uniform}}

\makeatletter
\renewcommand{\paragraph}{%
  \@startsection{paragraph}{4}%
  {\z@}{1ex \@plus 0.5ex \@minus .2ex}{-1em}%
  {\normalfont\normalsize\itshape}%
}
\makeatother

\begin{document}

\title{Point processes with event time uncertainty}

 \author{
 Xiuyuan Cheng$^{1}$,
 Tingnan Gong$^{2}$,   
 and Yao Xie$^{2,}$\thanks{Email: yao.xie@isye.gatech.edu. Authors listed alphabetically.} 
 \vspace{10pt}
 \\
 \small{$^1$Department of Mathematics, Duke University}\\
 \small{$^2$H. Milton Stewart School of Industrial and Systems Engineering, Georgia Institute of Technology} 
 }


\date{\vspace{-30pt}}

\maketitle

\begin{abstract}

Point processes are widely used statistical models for continuous-time discrete event data, such as medical records, crime reports, and social network interactions, to capture the influence of historical events on future occurrences. In many applications, however, event times are not observed exactly, motivating the need to incorporate time uncertainty into point process modeling. In this work, we introduce a framework for modeling time-uncertain self-exciting point processes, known as Hawkes processes, possibly defined over a network. We begin by formulating the model in continuous time under assumptions motivated by real-world scenarios. By imposing a time grid, we obtain a discrete-time model that facilitates inference and enables computation via first-order optimization methods such as gradient descent and variational inequality (VI). We establish a parameter recovery guarantee for VI inference with an $O(1/k)$ convergence rate using $k$ steps. Our framework accommodates non-stationary processes by representing the influence kernel as a matrix (or tensor on a network), while also encompassing stationary processes—such as the classical Hawkes process—as a special case. Empirically, we demonstrate that the proposed approach outperforms existing baselines on both simulated and real-world datasets, including the sepsis-associated derangement prediction challenge and the Atlanta Police Crime Dataset.

\end{abstract}


\section{Introduction}

Point processes, particularly self-exciting point processes, commonly known as Hawkes processes (see, e.g., \cite{hawkes1971spectra, reinhart2018review}), have been widely adopted for modeling sequential discrete event data across diverse domains, including seismology \cite{ogata1988statistical, ogata1999seismicity, zhuang2011next}, social networks \cite{farajtabar2014shaping}, high-frequency finance \cite{bacry2013some}, and genomics \cite{reynaud2010adaptive}.
Most existing models assume exact knowledge of event timestamps. However, in many practical applications, event times are often observed with uncertainty. For example, in medical contexts such as COVID-19 case reporting or ICU monitoring for sepsis detection, the exact onset of a condition may not be observed directly. Symptoms may arise, but a confirmed diagnosis typically follows after a delay due to lab testing. Similarly, in crime reports, such as burglaries, the exact time of the incident is often unknown because the event is discovered only after the fact. This type of uncertainty in event timing introduces a key modeling challenge, as illustrated in Figure~\ref{fig:illustration-motivation}.
While timing uncertainty is critical, the literature addressing it in point processes is limited. Early efforts, such as Ogata’s Bayesian approach \cite{ogata1999estimating}, introduced prior distributions (e.g., uniform over a fixed window) to model uncertainty in event times. However, these priors are subjective, and these methods did not fully address the computational challenges in inference.

In this paper, we introduce a new framework for modeling point processes under time uncertainty without relying on prior distributions. We adopt a non-Bayesian approach that models the underlying continuous-time Hawkes process directly, avoiding heuristic approximations of event times. Our formulation enables principled statistical inference and leads to efficient parameter estimation algorithms with both convergence guarantees and provable recovery accuracy.
Given a sequence of uncertain event-time windows, each containing at most one event, we first derive the likelihood in a continuous-time setting using a filtration defined by the history of uncertain event-window observations. In many real-world scenarios, the uncertain event windows are naturally aligned with a regular time grid (e.g., hourly medical measurements), which motivates our use of discretized time intervals.
This discretization reduces the continuous-time model to a discrete-time formulation, in which the influence kernel becomes a matrix that can be estimated. This matrix is a ``quantized'' version of the continuous-time kernel, with its size determined by the temporal resolution and influence horizon. The resulting model can be viewed as a Bernoulli process with a link function derived specifically to map historical observations to the event probability at each time point. This structure resembles a Generalized Linear Model (GLM) for discrete-time point processes.

We develop two inference approaches based on this model: one via maximum likelihood estimation (MLE) using gradient descent, and another inspired by solving a monotone operator variational inequality (VI), adapted from \cite{juditsky2019signal}. The latter improves numerical stability and efficiency. We provide theoretical guarantees, including an $O(1/k)$ algorithm convergence rate for parameter recovery.

We further extend the framework to point processes on graphs, in which events are uncertain in both time and location (modeled as nodes on a graph). We begin with a continuous-time formulation and apply a similar discretization to obtain a discrete-time model. In this setting, the influence kernel becomes a tensor, indexed by both time and graph nodes, along with a baseline intensity vector over the nodes. Estimation proceeds similarly using our proposed algorithms.

Empirically, we validate the approach on synthetic data, demonstrating convergence and recovery accuracy, as well as the accuracy of our model in predicting event probabilities in future intervals under both time-only and on-network settings. To evaluate real-world utility, we apply our model to two datasets: (1) the Sepsis-Associated Derangements (SADs) dataset \cite{physionetChallenge} and (2) Atlanta crime data. Both can be viewed as point processes over networks, where nodes represent medical variables or geographical regions. Our model reveals interpretable dynamic influence structures and achieves better predictive performance compared to baselines such as classical GLMs and Hawkes processes.

Notably, classical Hawkes models impose strong assumptions on the influence kernel, often fixing it to a parametric form, such as exponential decay. However, real-world influence structures can be significantly more complex. Here, we overcome this challenge by learning the influence kernel from discrete-time observations with time uncertainty, without assuming a specific parametric form. In addition, our formulation accommodates both non-stationary processes, in which the influence kernel is time-varying and represented as a matrix (or tensor in the network setting), and stationary processes, in which the kernel is time-invariant and reduces to a vector. All theoretical results and inference procedures naturally extend to the stationary case.

In summary, our contributions are:
\begin{itemize}

\item Modeling: We introduce a principled continuous-time point process model with time uncertainty, extendable to networked event data. After assuming that the event uncertainty windows are aligned with a temporal grid, the model yields a discrete-time formulation that is interpretable as a GLM with non-standard link functions.
	
\item Inference: We propose efficient parameter-estimation algorithms based on solving the maximum-likelihood problem via gradient descent, as well as a new approach based on solving a variational inequality with monotone operators. We provide theoretical guarantees for the estimation algorithms, including an $O(1/k)$ convergence rate for parameter recovery.
	
    \item Generality: Our framework handles both non-stationary and stationary influence kernels. All theoretical results apply to both settings.
	
    \item Empirical validation: We demonstrate the effectiveness of our model against baselines on both synthetic and real-world data, including the sepsis prediction dataset and the crime dataset.
\end{itemize}

\begin{figure}[t]
 \centering 
 \includegraphics[width=0.9\linewidth]{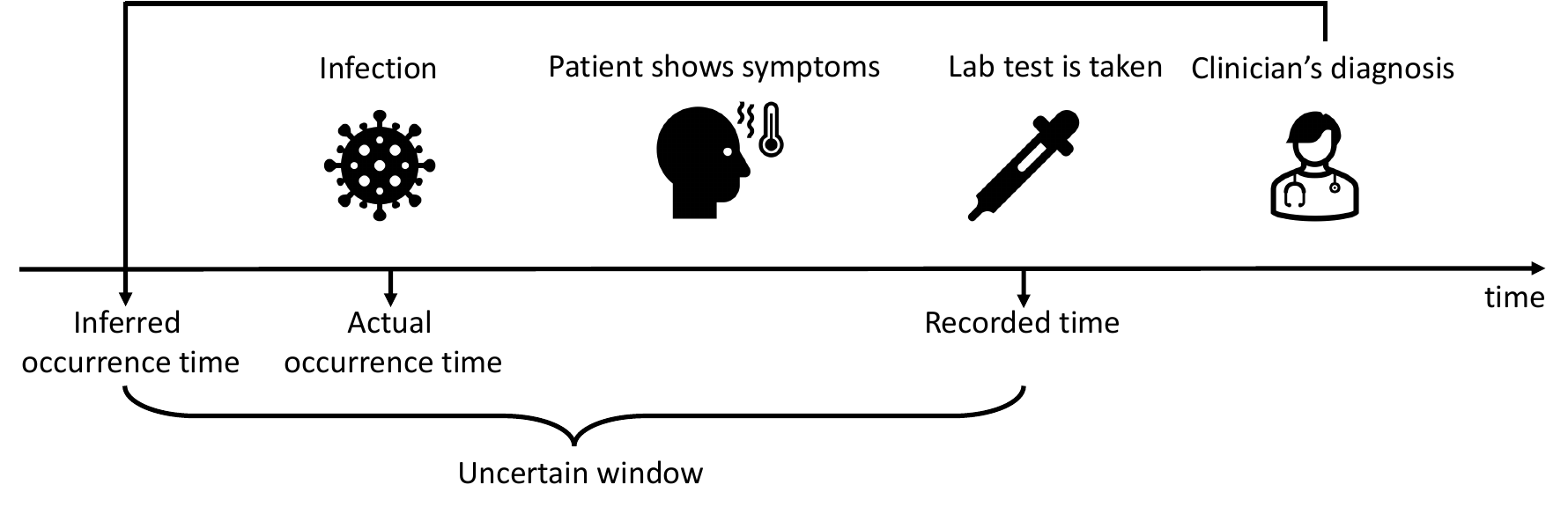}
 \vspace{-10pt}
 \caption{Illustration of a real-world medical application scenario leading to discrete event data with time uncertainty.}
 \label{fig:illustration-motivation}
 \end{figure}

\paragraph{Notations.} The notations used in this work are standard.
For a vector $x \in \R^n$, we use both bracket $x(u)$ or subscript $x_u$ to denote the $u$-th entry of the vector. When to emphasize that $x$ is a vector, we use the boldface symbol $\boldsymbol{x}$; $\otimes$ means vector outer-product, i.e., for $u, v \in \R^n$, $u \otimes n$ is an $n$-by-$n$ matrix.
$a \vee b := \max\{ a, b\}$, and  $a \wedge b := \min \{ a, b\}$.
$\R_+ = (0, \infty)$ denotes the set of strictly positive real numbers.

\subsection{Related works}

\paragraph{Hawkes process.} The Hawkes process, originally studied in \cite{hawkes1971spectra}, is a self-exciting continuous-time point process model, which can capture the excitation effects among events. The Hawkes process \cite{hawkes1971point, hawkes1971spectra, hawkes1974cluster} with a parametric kernel function has given rise to a body of literature for modeling the complex dynamics of heterogeneous time-event data, for example, earthquake occurrences \cite{ogata1988statistical, marsan2008extending}. Due to its good interpretability and well-studied estimation procedures, many applied works are keen on using parametric kernels, including exponential kernels \cite{zhou2013learninga, farajtabar2014shaping, yan2015machine, hall2016tracking} or power-law kernels \cite{zhao2015seismic, zhang2016modeling}. However, the expressiveness of the Hawkes process appears insufficient when the data are temporally non-stationary in their distribution or are not self-exciting. This is due not only to the parametric structure of the kernel in the Hawkes process, but also to the limited number of trainable parameters. Later, non-parametric Hawkes models \cite{hansen2015lasso}, including neural Hawkes models \cite{okawa2021dynamic}, were developed to accommodate the increasing complexity of real-world applications.

\paragraph{Point processes beyond Hawkes.} For events recorded at certain times, more expressive models have been developed in two directions. One is to model the intensity function, namely the rate of arrivals of events. Hansen et al. \cite{hansen2015lasso} aim to find the best linear approximation of the intensity function given a fixed dictionary. Mei and Eisner \cite{mei2017neural} model the intensity function through a continuous-time Long Short-Term Memory neural network. As indicated by \cite{dong2022spatio}, modeling intensity functions may lead to overfitting when the underlying dynamics are relatively simple. To address this issue, an alternative is to model the kernel function.
To learn the kernel function, Zhuang et al. \cite{zhuang2002stochastic} proposed a space-time branching process model (the Epidemic Type Aftershock Sequences model) with a parametric kernel that depends on the magnitudes of ancestors, time lags from offspring to current events, and event locations. Lewis and Mohler \cite{lewis2011nonparametric} proposed a nonparametric EM algorithm to learn coefficients for the kernel density estimator (KDE) of the triggering function. Zhou et al. \cite{zhou2013learningb} improved upon \cite{lewis2011nonparametric} by expressing the triggering kernel as a linear combination of base kernels rather than relying on a single KDE kernel. Bacry et al. \cite{bacry2012non} estimated the triggering kernel shape via the empirical auto-covariance of the counting process. Among these works, few explicitly consider the non-stationarity of the data in kernel design.
Another line of work includes Juditsky et al. \cite{juditsky2020convex}, who introduced a discrete-time Bernoulli process to model multivariate event data, leading to a General Linear Model (GLM) that can be efficiently optimized by first-order methods. Later, Wei and Xie \cite{wei2023causal} applied this model to learn causal networks from time-series data. These discrete-time models can be viewed as implicitly incorporating time uncertainty through discretization, but the modeling of time uncertainty itself is not explicit. In contrast, our model of event-time uncertainty is derived from the probabilistic principles of point processes and is applicable to more general point-process data. We will show that our model also leads to a GLM structure, and we adopt optimization techniques developed in \cite{juditsky2020convex}.

\paragraph{Uncertainty of observations in point processes.} There have been works on uncertainty quantification for model parameters of parametric Hawkes models. Wang et al. \cite{wang2020uncertainty} quantify uncertainty in MLE estimation for multivariate Hawkes processes. Dubey et al. \cite{dubey2021bayesian} use Bayesian neural networks to model event uncertainty via posterior distributions over model parameters. Osama et al. \cite{osama2019prediction} infer confidence intervals for the intensity functions of spatial point processes. Yang et al. \cite{yang2015predicting} assess model selection uncertainty using spatial point processes with applications to wildfire occurrences. These works primarily address uncertainty in model parameters rather than in the event times themselves, which is the focus of our work.

Recent studies have explored a related problem: estimating Hawkes processes from time-interval-censored observations, in which events are recorded as counts within non-overlapping time intervals. For instance, Schneider et al. \cite{schneider2023estimation} developed an expectation-maximization (EM) algorithm for parameter estimation in time-censored Hawkes processes. Rizoiu et al. \cite{rizoiu2022interval} proposed an approach that approximates the Hawkes process using a Mean Behavior Poisson Process (MBPP), which assumes independent increments to simplify model fitting. However, this approximation deviates from the underlying continuous-time stochastic model of the original Hawkes process. Our model also differs from those that consider time-interval-censored observations in \cite{schneider2023estimation, rizoiu2022interval} in that we assume each bin contains exactly one event, a scenario in which events are rare and can be localized to specific time intervals. This setup contrasts with \cite{schneider2023estimation, rizoiu2022interval}, where multiple events can occur within each non-overlapping bin—a more suitable model for denser observations. While our approach currently focuses on sparse event settings, it can, in principle, be extended to accommodate multiple events per bin, which we leave for future work.

\paragraph{Uncertainty of observations in time series.} In time series analysis, uncertainty in observed values has been studied through probabilistic models \cite{yang2019uncertain} or interval-valued observations \cite{maia2008forecasting}. 
For a comprehensive review of uncertain time series, see the survey by Dallachiesa et al. \cite{dallachiesa2012uncertain}. These models are discrete in time and focus on uncertainty in observed values, whereas our work focuses on uncertainty in event timing, with our discrete-time formulation derived from continuous-time uncertainty.

\section{Background}\label{Sec: background}

We start by reviewing some preliminaries of the classical point process without time uncertainty; a more comprehensive review can be found in \cite{daley2003introduction,reinhart2018review}.  

\subsection{Continuous-time Hawkes process}\label{subsec:prelim-hawkes}

Given a sequence of $n$ event times $\{t_1, t_2, \ldots, t_n\}$ on $ (0,T)$,
\[  
0 < t_1 <  t_2 <  \cdots < t_n  < T, 
\]
the (conditional) {\it intensity function} $\lambda(t)$ is defined as the probability of having the next event in $[t, t+dt)$ given the history, where we use $dt$ to denote the infinitesimal time. Formally, 
\begin{equation}\label{eq:def-lambdat-classical}
 \lambda( t ) := \lim _{\Delta t \rightarrow 0} \frac{ \E[ \mathbb{N}( [t, t+\Delta t) ) \mid \mathcal{H}_{t} ]}{\Delta t},   
\end{equation}
where $\calH_t $ is the history of all events up to $t$, 
and  $ \mathbb{N}$ is the counting measure of the events. 
Define the filtration $\calF_k$ to be the $\sigma$-algebra generated by observations up to the $k$-th event, i.e.,
\begin{equation}\label{eq:def-calFk}
    \calF_k :=  \sigma\{ t_1, \cdots, t_k \}.
\end{equation}
The {\it survival function} $S_k(t)$ is defined as
\begin{equation}\label{eq:def-Sk-old-hawkes}
    S_k(t) := \Pr [ t_k > t | {\calF}_{k-1}], \quad t \ge t_{k-1}, 
\end{equation}
and by definition $S_k(t_{k-1}) = 1$ and $S_k(t) > 0$. 

We define the {\it hazard function}, $ h_k(t)$ for the $k$-th event, in the standard fashion
\begin{equation}\label{eq:def-hk-old-hawkes}
    h_k(t) := -  \frac{S_k'(t)}{S_k(t)},  
    \quad t > t_{k-1}, 
\end{equation}
and equivalently, 
\[
S_k(t) = \exp\left\{ - \int_{t_{k-1}}^t  h_k(s) ds\right\},  
\quad  t  \ge t_{k-1}.
\]
Thus, by definition, this gives that for each $k$, 
\begin{equation}\label{eq:Prtk|Fk-old-hawkes}
\Pr [ t_k \in [t, t+dt) | \calF_{k-1}] 
= - d S_k(t)   
= h_k(t) \exp\left\{ - \int_{t_{k-1}}^{t}  h_k(s) ds\right\}  dt.
\end{equation}
From the above derivation, it can be shown that the intensity function $\lambda(t)$ defined in \eqref{eq:def-lambdat-classical}
actually equals the hazard function $h_k(t)$ on the interval $(t_{k-1}, t_k]$. 
Evaluating \eqref{eq:Prtk|Fk-old-hawkes} at $t = t_k$
and applying the argument consecutively for $k=1, \ldots, n$,
we have the expression of the log-likelihood as
\[
\ell := \log \Pr[ t_1, \cdots, t_n, t_{n+1}>T]
= \sum_{k=1}^n \log \lambda(t_k) - \int_0^T \lambda(s) ds.
\]

The self-exciting mechanism  is captured by the parametrization of the intensity $\lambda(t)$ through the 
{\it influence kernel} function $k(t',t)$ in the form as
\begin{equation}\label{eq:lambda-kernel-old-hawkes}
\lambda [k] (t)  = \mu + \sum_{i, \, t_i < t} k(t_i, t),
\end{equation}
where we have assumed that $\mu > 0$ is the constant base intensity of events throughout time. 
It is possible to make $\mu$ time-dependent in general cases.

In the classical Hawkes process \cite{hawkes1971point, hawkes1974cluster}, the influence kernel $k(t', t)$ is typically assumed to be time-invariant and takes the form of an exponential function:
$k(t', t) = \alpha \beta e^{-\beta (t - t')}$
where $\alpha > 0$ and $\beta > 0$ are scalar parameters. This time-invariant kernel leads to a stationary process, meaning that its distribution is invariant under time shifts.
We generalize this setting by allowing $k(t', t)$ to be a time-varying influence kernel, that is, it can depend not only on the time difference $t - t'$, but also on the absolute time $t$. In addition, we allow the influence kernel to take negative values to account for inhibitory effects. This general form enables modeling of how past events at time $t'$ influence future events at time $t > t'$, where the nature of this influence may vary over time.
When the influence is time-invariant, the kernel simplifies to a function of the time difference, i.e., $k(t', t) = k(t - t')$. In contrast, a time-varying kernel implies that the influence of past events varies with their occurrence time $t'$, even for the same lag $t - t'$.
Due to the principle of temporal causality, namely, that only past events can affect future events, the influence kernel is inherently asymmetric, with $k(t', t) \neq k(t, t')$ for $t \ne t'$.

\subsection{Point process on networks}

Hawkes already discussed the multivariate case in \cite{hawkes1971spectra}. 
In a point process on a network $\calV = \{1, \ldots, V \}$ having $V$ nodes, we observe the event data
\begin{equation}\label{eq:ti-ui-on-graph}
    (t_1,u_1), (t_2,u_2), \ldots, (t_n,u_n),
\end{equation}
over the time horizon  $[0,T]$,
$0 < t_1 < \cdots < t_n < T$;
 $t_i$ is the occurrence time of $i$-th event, and
  $u_i \in \calV $ denotes the location of the event on the network. 
We can similarly define the filtration $\calF_k$ of precise timing and node information of events as
\begin{equation}
\calF_k :=  \sigma\{ (t_1, u_1) , \cdots, (t_k, u_k) \}.
\label{filtration_1}
\end{equation}
The event history $\calH_t$ and the counting measure $\mathbb{N}$ contain information about both event time and location. 
The conditional intensity function at time $t$ and location $u$ is defined as 
\[
    \lambda (t, u) := \lim _{\Delta t \rightarrow 0} 
     \frac{\E [ \mathbb{N}( [t, t+\Delta t), u)  \mid \calH_t ]}{\Delta t}, \quad t >0, \, u \in \calV,
\]
where $\calH_t $ is the history of all events up to $t$. 
The log-likelihood of $n$ events has the expression as \cite{reinhart2018review}
\[
\ell = \sum_{k=1}^n \log \lambda(t_k, u_k) - \int_0^T  \sum_{v \in \calV } \lambda(s, v) ds.
\]

On a network, the influence kernel function $k(t', t, u', u)$ depends on both time and location. The conditional intensity function  $\lambda( t,u)$ is parameterized  as
\begin{equation}
    \lambda [ k ] (t, u )= \mu(u) + \sum_{i, \, t_i < t } k(t_i,t, u_i, u), \label{eq:classicHawkes}
\end{equation}
where $\mu(u)>0$ is the base intensity at node $u$.
 The kernel function $k (t', t, u', u )$, $t>t'$, $u, u' \in \mathcal V$ represents the influence, that is, the triggering or inhibiting effect of the past event at time and location $(t',u')$ 
 on the probability of a future event at $(t,u)$.

The model to be introduced in Section \ref{sec:model-with-graph} is more expressive than those previously considered in the literature (see, e.g., \cite{reinhart2018review}), including the spatial-temporal factorized model $k(t', t, u', u) = k_1 (t', t)k_2(u',u)$. As a result, our model can represent more complex spatial-temporal influence patterns of point processes on graphs.

\begin{figure}[t]
\centering 
\begin{subfigure}[h]{0.45\linewidth}
\includegraphics[width=\linewidth]{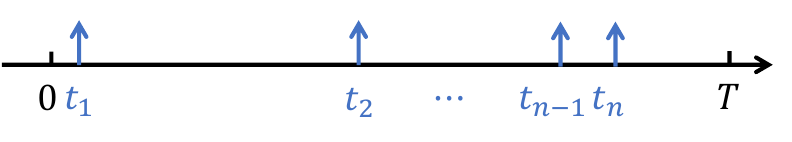}
\caption{Continuous-time: exact event times}
\end{subfigure}
\hspace{10pt}
\begin{subfigure}[h]{0.45\linewidth}
\includegraphics[width=\linewidth]{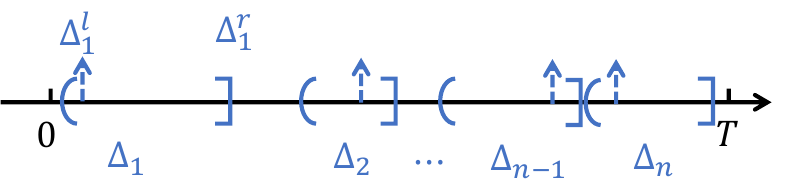}
\caption{Continuous-time: arbitrary event uncertainty}
\label{fig:time-diag-b}
\end{subfigure}

\begin{subfigure}[h]{0.45\linewidth}
\includegraphics[width=\linewidth]{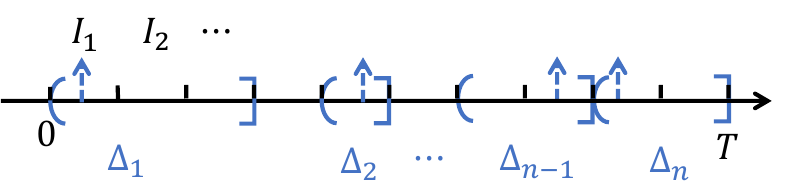}
\caption{Discrete-time: arbitrary event uncertainty}
\label{fig:time-diag-c}
\end{subfigure}
\hspace{10pt}
\begin{subfigure}[h]{0.45\linewidth}
\includegraphics[width=\linewidth]{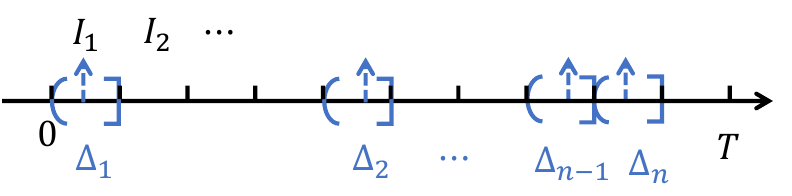}
\caption{Discrete-time: unit-time event uncertainty}
\end{subfigure}
\caption{
\label{fig:time-diag}
Illustration of the model: 
(a) Event times without uncertainty in the standard point process model.
(b) Continuous-time model with arbitrary-length event window, see Section \ref{subsec:cont-time-time-only}.
(c) Discrete-time model with arbitrary-sized event uncertainty, see Section \ref{subsec:discrete-general-uncertainty}.
(d) Discrete-time model with unit-sized event uncertainty, see Section \ref{subsec:bernoulli-time-only}.}
\end{figure}

\section{Events with time uncertainty}\label{sec:model-time-only}

In this section, we establish a time-uncertainty model using only event-time data, treating it as a one-dimensional point process. We first detail how to express the likelihood of a sequence of events subject to time uncertainty under a continuous-time point process model governed by a general influence kernel, and then explore a specific scenario in which the uncertainty windows align with a uniform grid on [0, T]. 
Furthermore, if the uncertainty window spans a single unit of time, the model is equivalent to a Bernoulli process with a specific link function that maps past observations to the current observation. 
This model will be extended to encompass network events, accounting for both time and location, in Section \ref{sec:model-with-graph}. All proofs are given in Appendix \ref{app:proofs-sec:3}.

\subsection{Continuous-time model with arbitrary-length event window}\label{subsec:cont-time-time-only}

\paragraph{Uncertainty window.}
The classical Hawkes model assumes that we observe events exactly when they occur; however, in practice, we often do not know the exact time of events.
Instead, we may know when the $i$-th event happens up to certain {\it uncertainty window} $\Delta_i$, namely,
\begin{equation}
    t_i \in \Delta_i := (\Delta_i^l, \Delta_i^r], \quad i = 1, 2, \ldots, \label{event_uq}
\end{equation}
where $\Delta_i^l< \Delta_i^r$ are the endpoints of the interval for each event. In other words, we do not know exactly when the event happens, but we only know that it happens within the interval.

We also make the following assumption that each interval only contains one event, and the intervals do not overlap. 
In other words, the observation has enough time resolution to identify an individual event up to its time uncertainty.

\begin{assumption}[Non-overlapping windows]\label{assump:disjoint_Delta} 
(A1)  The consecutive intervals $\Delta_i$ do not overlap, i.e., $\Delta_{i-1}^r \le \Delta_{i}^l$ for all $i$.
\end{assumption}
We would like to remark that the extension to allow multiple events in one uncertainty window may lead to a  Poisson process, which can be handled by other mathematical tools and is not the focus of this work. 
The time uncertainty window is illustrated in Figure \ref{fig:time-diag-b}.

\paragraph{Likelihood model.}
We are ready to derive the likelihood of the observation
\begin{equation}\label{eq:def-calL-time-uncert}
\calL := \Pr [  t_1 \in \Delta_1, \cdots, t_n \in\Delta_n, t_{n+1} > T],
\end{equation}
which will enable model estimation.
Our derivation will be based on first principles, and we will define the survival and hazard functions in the new setting in which the history contains time uncertainty. 
Specifically, instead of using the filtration $\calF_k$ as in \eqref{eq:def-calFk} by precise timing information, we define the filtration to be generated by historical events with time uncertainty, i.e.,
\begin{equation}
\tilde{\calF}_k :=  \sigma\{ t_i \in \Delta_i, \, i=1, \cdots, k \}.
\end{equation}
For notation brevity, below we use $\Delta_i$ to stand for the probability event of $t_i \in \Delta_i$ when there is no confusion. 
The survival function $S_k(t)$
is defined similarly as in the classical case in Section \ref{subsec:prelim-hawkes},  c.f. \eqref{eq:def-Sk-old-hawkes}, but conditioning on the uncertain history, that is
\begin{equation}\label{eq:def-Sk-time-only}
    S_k(t) := \Pr [ t_k > t | \tilde{\calF}_{k-1}], 
    \quad t \ge \Delta_{k-1}^r. 
\end{equation} 
Because $t_k $ must happen after $\Delta_{k-1}^r$ by Assumption (A1), 
we have $ S_k( \Delta_{k-1}^r ) = 1$.
The hazard function $h_k$ is defined from $S_k$, same as before, c.f. \eqref{eq:def-hk-old-hawkes},
\begin{equation}\label{eq:def-hk-time-uncert}
    h_k(t) := -  \frac{S_k'(t)}{S_k(t)},  
    \quad t > \Delta_{k-1}^r, 
\end{equation}
and note that the time is from $t > \Delta_{k-1}^r$.
As a result, $S_k(t)$ can be represented by integrating $h_k$ as 
\begin{equation}\label{eq:Sk-expression-by-hk}
S_k(t) = \exp\left\{ - \int_{\Delta_{k-1}^r}^t h_k(s) ds\right\},  \quad  t \ge \Delta_{k-1}^r.
\end{equation}
We further define the intensity function $\lambda(t)$  by {\it piecing together} the hazard function $h_k$ from each interval as
\begin{equation}\label{eq:def-lambdat-time-only}
\lambda(t) := h_k(t), \quad t \in (\Delta_{k-1}^r, \Delta_{k}^r], \quad k = 1, 2, \cdots, 
\end{equation}
and for $k=n+1$ it will from the interval $(\Delta_n^r, T]$; note that in this way, $\lambda(t)$ is defined everywhere on the time line.

\begin{lemma}\label{lemma:logL-time-only}
The log-likelihood of $\calL$ as in \eqref{eq:def-calL-time-uncert} has the expression
    \begin{equation}\label{eq:l-lambda-cont-time-uncert}
\ell: =
\log \calL
=   \sum_{k=1}^n
  \log \left( \exp\left\{ \int_{\Delta_{k}^l}^{\Delta_{k}^r}  \lambda(s) ds\right\}  - 1 \right) 
  - \int_{0}^{T}  \lambda(s) ds.
\end{equation}
\end{lemma}

\paragraph{Influence kernel.}

In the classical Hawkes model, the intensity function  $\lambda(t)$ is further parametrized by the influence kernel function  $k(t',t)$, see \eqref{eq:lambda-kernel-old-hawkes}.
Here, we also parametrize the intensity function \eqref{eq:def-lambdat-time-only} by a linear superposition of continuous time kernel $k(t',t)$, and we adopt the following expression to incorporate time uncertainty
\begin{equation}\label{eq:lambda-phi-cont-time-uncert}
    \lambda[k](t) = \mu + \sum_{i, \Delta_i^r < t } \frac{1}{|\Delta_i|} \int_{\Delta_i^l}^{\Delta_i^r} k( t', t) dt'.
\end{equation}
The notation has emphasized the dependence of $\lambda(t)$ on the continuous time kernel function $k$. 

\begin{remark}[Consistency of conditional intensity function definition]
We choose the form of the conditional intensity function in \eqref{eq:lambda-phi-cont-time-uncert}
 because it is consistent with the earlier Hawkes conditional intensity model specified in \eqref{eq:lambda-kernel-old-hawkes}.
Specifically,  if the uncertainty window $\Delta_i$ shrinks to a single point $t_k$,
i.e.,  $|\Delta_i|\rightarrow 0$ for all $i$, 
then,  by that  $t_i \in (\Delta_i^l, \Delta_i^r]$ according to \eqref{event_uq} (and assuming the continuity of the function $k$),
the expression \eqref{eq:lambda-phi-cont-time-uncert} is reduced to \eqref{eq:lambda-kernel-old-hawkes}.
In this limit, the time-uncertain model for the conditional intensity function $\lambda(t)$ recovers the classic ``time-certain'' Hawkes model. 
\end{remark}

\begin{remark}[Time-invariant kernel]\label{rk:stationary-kernel}
In \eqref{eq:lambda-phi-cont-time-uncert}, we have assumed the most general form of kernel $k(t',t)$, which is allowed to be non-stationary. 
The time-invariant kernel is a special case when $k(t',t) = \psi( t- t' ) $ for some function $\psi$.  
Later in Section \ref{sec:inference-GD-VI}, we will first present the general case of a time-varying kernel (corresponding to the recovery of a matrix), 
and then discuss the special case of the time-invariant kernel (corresponding to the recovery of a vector)
 in Section \ref{subsec:special-structure-psi}.
\end{remark}

\subsection{Discrete time model: Arbitrary event uncertainty}\label{subsec:discrete-general-uncertainty}
In many applications, the observed uncertainty windows are some regular time intervals, e.g., a window having a length of one or more hours 
 that starts at integer hours. 
For example, in the PhysioNet challenge data for sepsis prediction \cite{physionetChallenge}, event (patient test results) records are recorded hourly.
We show that under this setting, the continuous-time uncertainty model above can be simplified into a special case of a discrete-time model, 
where the uncertainty is up to a certain time unit.

\paragraph{Time grid.}
We evenly divide the time horizon $(0,T)$ into $N$ intervals, each having length $h$, 
\begin{equation}\label{eq:def-h-time-step}
    h := T/N,
    \quad  I_j: = \left( (j-1) h, j  h \right], 
    \quad j  \in [N] :=\{ 1,\cdots,N \}.
\end{equation} 
In the discrete-time models, we use $h$ to stand for the size of the time grid, not to confuse the notation with the hazard function $h_k$,
which was denoted by $\lambda(t)$ from \eqref{eq:def-lambdat-time-only} on.

To proceed, we assume that the event time uncertainty intervals $\Delta_i$ are ``aligned'' with the discrete-time grids, as illustrated in Figure~\ref{fig:time-diag-c}.
 We index the $N$ intervals by the set $[N]:= \{1, \cdots, N \}$.  
\begin{assumption}[Alignment to time grid]\label{assump:grid}
(A2) The uncertain event time window $\Delta_i$ can only have the end points $\Delta_i^l$ and $\Delta_i^r$ 
taking value on the evenly spaced discrete time grid points $\{ j h: j \in [N]\}$.
\end{assumption}  

Under (A2), the event-time window $\Delta_i$ for each event $i$ occupies some consecutive intervals of $I_j$.
When $\Delta_k$ occupies from 
the $l$-th interval till the $r$-th interval, 
 we denote the index $l$ as $t_k^l$,
 the index $r$ as $t_k^r$, that is,
 \begin{equation}\label{eq:Deltak-discrete-time-multiple-Ij}
 \Delta_k = \cup_{ t_k^l \le j \le t_k^r } I_j,
 \quad t_k^l, t_k^r \in [N], 
 \quad t_k^l \le t_k^r, 
 \quad  t_{k}^r < t_{k+1}^l.
 \end{equation}
The last inequality $t_{k}^r < t_{k+1}^l$ is by that the uncertainty windows $\Delta_k$ do not overlap as assumed in (A1).
The assumption (A2) provides a discrete-time formulation
that will  allow us to simplify the model \eqref{eq:l-lambda-cont-time-uncert} and \eqref{eq:lambda-phi-cont-time-uncert} by writing everything in discrete time indexed by $j \in [N]$.

\paragraph{Likelihood and influence model.}
Under (A2), the log-likelihood of the windowed observations \eqref{eq:l-lambda-cont-time-uncert} only depends on the $N$ quantities which are the integral of $\lambda(t)$ on the subintervals $I_j$. 
 To be specific,  we define
\begin{equation}\label{eq:def-Lambdaj-discrete-time}
\Lambda_j := \frac{1}{ h }  \int_{I_j} \lambda(s) ds, \quad j \in [N],
\end{equation}
which can be interpreted as the ``average intensity" over a discrete time interval $I_j$.
With the definition of $\Lambda_j$, \eqref{eq:l-lambda-cont-time-uncert} is reduced to 
\begin{equation}\label{eq:l-lambda-dis-time-uncert}
\ell 
= \sum_{k=1}^n \log \left( e^{ h \sum_{ j' = {t_k^l}}^{t_k^r}  \Lambda_{j'} }   - 1 \right)
-  h \sum_{j=1}^N  \Lambda_j.
\end{equation}

Next, we derive the representation of intensity $\Lambda_j$ by the influence kernel.
Define the matrix 
\begin{equation}\label{eq:def-Kij-discrete-time}
K_{i,j} : = \frac{1}{ h^2  }\int_{I_{i}} \int_{I_j} k( t', t) dt' dt, \quad i < j, \quad i,j \in [N], 
\end{equation}
which can be interpreted as the ``average influence" casted on the interval $I_j$ from an earlier interval $I_i$.
We will show that only the $i< j$ entry of $K_{i,j}$ is used. We call $K$ the {\it influence kernel matrix}, which can be viewed as a quantization of the kernel function $k(t',t)$.

The following lemma provides the representation of $\Lambda_j$ by the kernel matrix $K_{i,j}$.
\begin{lemma}\label{lemma:Lambdaj-Kij-lemma}
Under Assumption (A2),
for $j \in [N]$, 
\begin{equation}\label{eq:Lambdaj-Kij-lemma}
\Lambda_j[K] = 
     \mu+ \sum_{\{k, \, t_k^r < j\} } \frac{1}{ t_k^r - t_k^l+1} \sum_{i=t_k^l}^{t_k^r} K_{i,j}.
\end{equation}
\end{lemma}
In the notation, we have emphasized the dependence of $\Lambda_j$ on the kernel matrix $K$.
The lemma shows that the dependence of the discrete-time intensity function $\Lambda_j$ 
from the influence kernel function $k(t',t)$ is all encoded into the matrix $K_{i,j}$.

The equations \eqref{eq:l-lambda-dis-time-uncert} and \eqref{eq:Lambdaj-Kij-lemma} provide a discrete-time model which only involves discrete-time objects $\Lambda_j$ and $K_{i,j}$. 
This will facilitate computation, as we can thus reduce the functional optimization problem with respect to  $k(t',t)$ 
to a matrix-valued optimization problem with respect to $K_{i,j}$.

\subsection{Discrete time model: Unit-time event uncertainty}\label{subsec:bernoulli-time-only}

The discrete-time formulation in Section \ref{subsec:discrete-general-uncertainty} has a simple and important special case,
which we call ``unit uncertainty,'' as follows.

\begin{assumption}[Unit-time uncertainty]\label{assump:A3}
(A3) The uncertainty window $\Delta_i$ falls on exact one of the $N$ intervals, i.e., $\Delta_i = I_{j(i)}$ for some $j(i) \in [N]$.
In other words, $t_i^l = t_i^r = j(i)$.
\end{assumption}

\paragraph{Likelihood and influence model.}
Under Assumption (A3),  the discrete-time model  \eqref{eq:l-lambda-dis-time-uncert} and \eqref{eq:Lambdaj-Kij-lemma} is further simplified to
\begin{align}\label{eq:model-base-uncert}
\ell 
 = \sum_{k=1}^n \log \left( e^{ h     \Lambda_{j(k) } }   - 1 \right)
    -   h   \sum_{j=1}^N  \Lambda_j,
    \quad 
\Lambda_j  
  =  \mu+ \sum_{\{k: \, j(k) < j\} }  K_{j(k), j}.
\end{align}
Like before, $\Lambda_j = \Lambda_j[K]$ which depends on the kernel matrix $K$.

Meanwhile, when the baseline intensity $\mu$ is unknown, in principle it can also be inferred from data. 
Thus, one can treat the scalar $\mu$ as an estimable parameter. 
We write the parameters $\{\mu ,K \} $ together as $\theta$, 
and write $\Lambda_j$ as $\Lambda_j (\theta)$ to emphasize the dependence on the parametrization. 
In this work,  our inference algorithms and experiments will focus on the unit-time uncertainty case under (A3).

In the rest of this subsection, we first show  that  the model \eqref{eq:model-base-uncert} has an equivalent formulation in the form of a Bernoulli process,
which will give a close connection to the general linear model for discrete-time point process \cite{juditsky2020convex} 
and also facilitate notations in future sections.
In addition, we show that, once the parameters are estimated, our discrete-time model allows us to predict the event chance on a future interval.

\paragraph{Bernoulli process formulation.}
Under the unit-uncertainty setting (A3), our discrete-time model has an equivalent form as a Bernoulli process with an autoregressive structure and a particular link function. Below, we will use $i$, $j$, and also $t \in [N]$ to denote the time-grid index, not to be confused with the former continuous time $t$.

Specifically,  each trajectory of the event time uncertainty window data can be written as a  binary sequence $y = (y_1, \cdots, y_N)$, 
where $y_t = 1$ if $I_t = I_{j(k)} $ for some $k$-th event (i.e., there is an event happening inside $I_t$), 
and $y_t = 0$ otherwise. 
Using the notation of $y_j$, 
\eqref{eq:model-base-uncert} can be equivalently written as
\begin{equation}\label{eq:model-base-uncert-bernoulli}
\ell
= \sum_{j=1}^N \left( 
- (1-y_j ) h  
\Lambda_j +  y_j  \log ( 1 - e^{ - 
 h
\Lambda_j} )
\right), \quad
\Lambda_j ( \theta )
=  \mu+ \sum_{ i < j}  y_i  K_{i, j}.
\end{equation}

Furthermore, the following lemma characterizes the conditional expectation of $y_t$ given the time-uncertain history. In this lemma, $\Lambda_t $ refers to the true intensity $\Lambda_t^*$, namely induced by the true parameter under which the sequence $\{ y_t \}_t$ is generated. 
The lemma provides a general linear model to predict $y_t$ using $\Lambda_t$.
We note that the link function $\phi$ here takes a particular form under the framework of \cite{juditsky2020convex}, which considers a discrete-time autoregressive Bernoulli process; however, with a standard logistic link function.

\begin{lemma}\label{lemma:time-only-bernoulli-glm}
Under the unit-uncertainty setting Assumption \ref{assump:A3}(A3) and using the formulation of the Bernoulli process $y_t$, 
for any $t\in [N]$, $\Lambda_t \in \sigma\{ y_i, i\le t-1\}$ and 
\begin{equation}\label{eq:time-only-bernoulli-glm}
\E [ y_t |  y_i, i\le t-1 ] = \phi( h   \Lambda_t ), 
\quad \phi(x):= 1-e^{-x}.
\end{equation}
\end{lemma}
The function $\phi$ is monotonically increasing on $\R$.
The relationship \eqref{eq:time-only-bernoulli-glm} will be used in deriving the estimation algorithm.

\paragraph{Prediction on future interval.}
Once the parameters $\theta$ are learned from data, we can compute $\Lambda_j$ based on \eqref{eq:model-base-uncert} and equivalently \eqref{eq:model-base-uncert-bernoulli}.
This enables us to predict the likelihood of an event occurring during a future interval.
To derive the expression, we take a step back to recall the notations of the continuous-time setting in Section \ref{subsec:cont-time-time-only}.
Given $n$ historical events with uncertainty, 
by the definition of $S_k$ in \eqref{eq:def-Sk-time-only}, the chance of having the next event on a given interval $( l ,r]$, $r > l \ge  \Delta_{n}^r$, 
can be written as
\begin{equation}\label{eq:predict-derivation-1}
\Pr [t_{n+1} \in (l, r] | \tilde{\calF}_{n} ]
=  S_{n+1}(l) - S_{n+1}(r).
\end{equation}
Under the discrete-time model and using the notation in Section \ref{subsec:discrete-general-uncertainty}, suppose $l$ and $r$ are endpoints of the time grid that is,
\[
l = j_l h , \quad r = j_r h , \quad j_r > j_l \ge t_n^r.
\]
By \eqref{eq:predict-derivation-1}, \eqref{eq:Sk-expression-by-hk} and the definition of $\Lambda_j$ in \eqref{eq:def-Lambdaj-discrete-time},
we have
\begin{equation}\label{eq:predict-future-1}
\Pr [t_{n+1} \in ( j_l h, j_r h] | \tilde{\calF}_{n} ]
= e^{ - h \sum_{j' = t_n^r +1}^{j_l} \Lambda_{j'}} \left( 1 -  e^{ - h \sum_{j' =j_l+1}^{j_r} \Lambda_{j'}} \right).
\end{equation}
For the probability of having no event happening till time $l = j_l h$, we have
\begin{equation}\label{eq:predict-future-2}
\Pr [t_{n+1} > j_l h  | \tilde{\calF}_{n} ]
= e^{ - h \sum_{j' =t_n^r +1}^{j_l} \Lambda_{j'}}.
\end{equation}
In particular, under unit-time uncertainty (A3), 
$\Delta_{n} = I_{ j(n) }$ and $t_n^r = j(n)$.  
One can compute \eqref{eq:predict-future-1} and \eqref{eq:predict-future-2} with $t_n^r = j(n)$ in the expressions.

\section{Estimation of influence kernel}\label{sec:inference-GD-VI}

\begin{figure}[t]
\centering 
\begin{subfigure}[h]{0.26\linewidth}
\includegraphics[width=\linewidth]{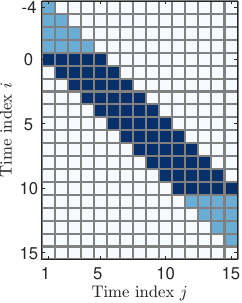}
\caption{Kernel $K$.}
\label{fig:time-only-K-illus}
\end{subfigure}
\hspace{+1.2in}
\begin{subfigure}[h]{0.1525\linewidth}
\includegraphics[width=\linewidth]{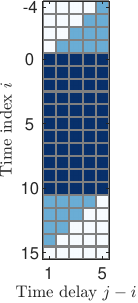}
\caption{Kernel $\Psi$.}
\label{fig:time-only-Phi-illus}
\end{subfigure}
\caption{
Kernel matrix $K$ in \eqref{eq:def-Kij-discrete-time}
and the reparameterized form $\Psi$  in \eqref{eq:def-Psi-matrix}. 
 We set $N' = 5$ and $N=15$. 
The blue (including light and dark blue) elements represent $j > i$, the entries of the kernel matrix that are estimable in $\theta_K$, see \eqref{eq:def-theta-timeonly-Omega};
The white elements are non-causal (i.e., $j\leq i$) and thus are excluded from $\theta_K$.
We apply the low-rank constraint to the submatrix formed by the dark blue entries in $\Psi$.}
\label{fig:time-only-K-diag}
\end{figure}

Given an observed sequence of event data with time uncertainty, our goal is to estimate the influence kernel from data up to the time uncertainty.
The derivations in Section \ref{sec:model-time-only} provide a discrete-time model, where the parameters are the influence kernel matrix $\{K_{i,j}\}$ (and the scalar base intensity $\mu$). 
In this section, we develop estimation approaches to recover the kernel matrix from data, focusing on unit uncertainty setting under Assumption \ref{assump:A3}(A3). 

\subsection{Kernel matrix  and log-likelihood}\label{subsec:kernel-est-parametrization}

Recall that  $\theta$ denotes the parameters in the model, 
which consists of entries in the influence kernel matrix $K$ and $\mu$. 
In the model \eqref{eq:model-base-uncert} and equivalently \eqref{eq:model-base-uncert-bernoulli}, 
only $K_{i,j}$ for $j > i$ are contributing to the likelihood. 
Thus, only $K_{i,j}$ for $j > i$ are included into $\theta$.
We further assume that the influence has a finite memory length, that is, 
\[
k(t', t) = 0, \quad \text{if $t-t' > \tau_{\rm max}$.}
\]
Without loss of generality, assume that $\tau_{\rm max} = N' h$ for some integer $N'$. 
By definition \eqref{eq:def-Kij-discrete-time}, we have 
\[
K_{i,j} = 0, \quad \text{if $j - i > N'$.}
\]
Thus, we only include $K_{i,j}$ for $  i < j \le i + N' $ into $\theta$, as illustrated in Figure \ref{fig:time-only-K-illus}.

The parameters are then
\begin{equation}\label{eq:def-theta-timeonly-Omega}
    \theta 
    = \{ \mu \} \cup
    \{ K_{i,j}, \,  0 <j-i \le N', \,  -N' < i \le N, \, 0 < j  \le N \}
    =: \{ \mu \} \cup \theta_K, 
\end{equation} 
where $\theta_K$ has $N N'$ estimable entries and thus is in the vector space of $\R^{NN'}$. 
Here, for convenience, we have assumed negative time grids indexed from $-N'+1$ to 0 with possible observed event data (that is, the Bernoulli process model in Section \ref{subsec:bernoulli-time-only} starts at $-N'+1$ instead of 1). 
These historical data are used to compute $\Lambda_j$ for $j >0$, while the likelihood is computed on $1 \le j \le N$.  
The vector $\theta_K$ can be re-arranged into an array indexed by $ -N' < i \le N$  and $0 < j-i \le N'$, 
which we call the kernel matrix $\Psi $,
\begin{equation}\label{eq:def-Psi-matrix}
\Psi_{i,l} = K_{i, i+l}, \quad l=1, \cdots, N',
\end{equation}
 and $\Psi$ is of size $(N+N')\times N'$, see Figure \ref{fig:time-only-Phi-illus}.
 In many cases, it is reasonable to assume that the rearranged kernel matrix $\Psi$ is low-rank, possibly due to the similar influence pattern over time, see more in Section \ref{subsec:special-structure-psi}. 

Using the notations above, our likelihood model \eqref{eq:model-base-uncert-bernoulli} can be written as
\begin{equation*}
\ell(\theta)
= \sum_{t=1}^N \left( 
- (1-y_t ) h \Lambda_t(\theta) +  y_t  \log ( 1 - e^{ - h \Lambda_t(\theta)} )
\right), 
\quad 
\Lambda_t ( \theta ) = \mu + \langle \eta_t , \theta_K \rangle,
\end{equation*}
and  $\eta_t \in \R^{NN'}$ is a vector 
consisting of binary-valued entries determined by the process $y_t$ such that 
\begin{equation}\label{eq:def-innerprod-etat-K}
\langle \eta_t , \theta_K \rangle = \sum_{ i = t- N' }^{t-1}  y_i  K_{i, t}.
\end{equation}

In practice, we usually estimate the model parameters from multiple trajectories. Given $M$  independent (training) trajectories  $\{ y_t^{(m)}\}_{m=1}^M $, the log-likelihood is summed over the $M$ trajectories.
The trajectories may potentially differ in length, but the likelihood model remains well-defined. For simplicity, we assume the $M$ trajectories are of the same length $N$, and the trajectories $y^{(m)}$ are i.i.d. across $m$. 
We introduce the superscript $^{(m)}$ to stand for the $m$-th trajectory and apply to the notations $y_t$, $\eta_t$, $\Lambda_t$ and log-likelihood $\ell$. 
We denote the (averaged) summed log-likelihood as $L(\theta)$, 
\begin{align}
L(\theta) 
& = \frac{1}{M}\sum_{m=1}^M \ell^{(m)}(\theta),  \label{eq:def-Ltheta-timeonly} \\
\ell^{(m)} (\theta)
& = \sum_{t=1}^N \left( 
- (1-y_t^{(m)} ) h \Lambda_t^{(m)}(\theta) +  y_t^{(m)}  \log ( 1 - e^{ - h \Lambda_t^{(m)}(\theta)} )
\right),  \label{eq:lm-theta-timeonly}\\
 \Lambda_t^{(m)}(\theta) 
& = \mu + \langle \eta_t^{(m)} , \theta_K \rangle. \label{eq:lambdatm-theta-timeonly}
\end{align}

Given the likelihood model, it is natural to estimate the parameters using MLE by, e.g., Gradient Descent (GD).
We will see that the MLE problem is (strongly) convex under certain technical conditions, which is proved in expectation and should hold when there are enough trajectories. However, it is known that gradient descent algorithms for estimating GLMs converge slowly, particularly near the true parameters \cite{he2020point}. Thus, we also consider an alternative approach to estimate model parameters using Variational Inequality (VI) optimization \cite{juditsky2019signal}. We obtain a parameter recovery guarantee for VI (in expectation) under a stochastic optimization setting that resembles the guarantee of stochastic first-order optimization for convex problems.
All the proofs are postponed to Appendix \ref{app:proofs-sec:4}.

\subsection{Gradient Descent (GD) for MLE}\label{subsec:GD-shceme-timeonly}

Given the likelihood model \eqref{eq:def-Ltheta-timeonly}\eqref{eq:lm-theta-timeonly}\eqref{eq:lambdatm-theta-timeonly}, the MLE solves for parameter $\theta$ by maximizing $L(\theta)$.
Again, we assume that $\mu$ is known and focus on solving for the kernel matrix parameters $ z= \theta_K \in \R^{ N N'}$. Define $L(z) := L([\mu, z])$, the GD dynamic follows the gradient field
\[
F(z) = - \partial_z L(z) =  \frac{1}{M} \sum_{m=1}^M \hat F^{(m)}(z),
\quad 
\hat F^{(m)}(z) = - \partial_z l^{(m)}([\mu, z]),
\]
and the GD update is by $z_{k} = z_{k-1} - \gamma_k F(z_{k-1})$ at step size $\gamma_k > 0$.
One can also implement stochastic optimization similar as in \eqref{eq:VI-SGD-scheme},  replacing $\hat G^{(k)}$ with $\hat F^{(k)}$. 

Direct computation gives the expression of $\hat F^{(m)}$ as 
\begin{equation}\label{eq:expression-hatFm}
\hat F^{(m)}(z) = 
\sum_{t=1}^N  \frac{ h}{ \phi( h  \Lambda_t^{(m)}( z) )}   ( \phi( h  \Lambda_t^{(m)}( z) ) - y_t^{(m)} )  \eta_t^{(m)},
\end{equation}
as well as the Hessian of $L(z)$. The following proposition shows that the objective $-L$ is strongly convex in expectation on the domain $\Theta_K$.

\begin{proposition}\label{prop:GD-strongly-convex}
Under Assumption \ref{assump:c1-Lambda-on-data}, 
$- \E L(z)$ is $\kappa'$-strongly convex on $\Theta_K$,
where $\kappa'  :=  e^{-2hb} (b/B)^2 h^2 e^{-B \tau_{\rm max}}$.
\end{proposition}

While Proposition \ref{prop:GD-strongly-convex} gives an in-expectation result, one may use the i.i.d. or possibly relaxed independence assumptions on trajectories $y^{(m)}$ to derive a concentration argument and show that when there are sufficiently many training trajectories, $-L$ is strongly convex and the MLE has a unique solution. 
Meanwhile, it is also possible to prove a recovery guarantee (in expectation) for a stochastic GD scheme similar to the VI result in Theorem \ref{thm:VI-convergence}.
We omit these extensions.

\subsection{Monotone Variational Inequality (VI) for parameter recovery}\label{subsec:VI-recovery-timeonly}

The VI approach leverages the GLM model in \eqref{eq:time-only-bernoulli-glm} with link function $\phi: \R \to \R$, and casts the parameter estimation problem as finding an equilibrium that corresponds to the true parameters using an iteration defined by the monotone operator \cite{juditsky2019signal,juditsky2020convex}.
We conceptually consider the $L^2$ loss defined as
 $L_{\rm VI}(\theta)   = \frac{1}{M} \sum_{m=1}^M \ell_{\rm VI}^{(m)}(\theta)$,
 where 
 \[
 \ell_{\rm VI}^{(m)}(\theta)
   = \frac{1}{2 h }\sum_{t=1}^N  ( \phi( h \Lambda_t^{(m)}(\theta)) - y_t^{(m)} )^2.
 \]
Recall that $\theta = \{ \mu, \theta_K\}$.
For simplicity and without loss of generality, in the rest of this section, we always assume that $\mu$ is a known constant and focus on estimating the influence kernel $\theta_K \in \R^{ N N'}$.  We will explain the algorithmic details to recover $\mu$ from data in Section \ref{sec:algo}.

\paragraph{VI stochastic scheme.}
For notation convenience, we denote $\theta_K$ as $z \in \R^{N N'}$, and also write 
$\Lambda_t( z) = \Lambda_t( [\mu, z])$, i.e., with $\theta = [\mu, z]$.
Note that $\Lambda_t^{(m)}$ is linear in $z=\theta_K$
and $\partial_{\theta_K} \Lambda_t^{(m)}(\theta) = \eta_t^{(m)}$ by \eqref{eq:lambdatm-theta-timeonly}.
Following the VI monotone operator construction for GLM \cite{juditsky2020convex}, 
we have the VI monotone operator 
(as a vector field in the space of $z$ and it is {\it not} the gradient to minimize the $L^2$ loss above, which is a non-convex problem) 
written as 
\begin{equation}\label{eq:formula-Gtheta-timeonly}
\hat G( z ) =  
 \frac{1}{M} \sum_{m=1}^M 
 \hat G^{(m)}( z ),
 \quad 
 \hat G^{(m)}( z ) = 
\sum_{t=1}^N \left(\phi( h \Lambda_t^{(m)}( z )) - y_t^{(m)} \right) \eta_t^{(m)}.
\end{equation}

We use $z_k$ to stand for the parameter at the $k$-th iteration, $k=0,1, \ldots$.
The stochastic updating scheme reads 
\begin{equation}\label{eq:VI-SGD-scheme}
z_k = {\rm Proj}_{\Theta_K} [ z_{k-1} - \gamma_k \hat G^{(k)} (z_{k-1}) ],
\end{equation}
where  $\gamma_k > 0 $ is the step size,
$\Theta_K$ is a convex compact domain in $\R^{N N'}$ satisfying certain technical assumptions to be specified below, and ${\rm Proj}_{\Omega}$ is the metric projection operator onto convex set $\Omega$ defined as 
\begin{equation}\label{eq:def-proj-operator}
{\rm Proj}_{\Omega} (z) : = \arg\min_{ u \in \Omega} \| u - z\|_2^2.
\end{equation}
We assume that the iteration \eqref{eq:VI-SGD-scheme} starts from $z_0 \in \Theta_K$.
For theoretical simplicity, we also assume there are enough trajectories to compute \eqref{eq:VI-SGD-scheme} without repeating.

\paragraph{Kernel recovery guarantee.}
We will prove the convergence of the scheme \eqref{eq:VI-SGD-scheme} such that $z_k$ recovers the true kernel parameter $z^*$.
The convergence bounds the 2-norm error in expectation, namely $\E \| z_k - z^*\|_2^2 $, where the expectation $\E$ is over the randomness of data observations $y_t^{(m)}$. 
A central property of analyzing VI optimization is the monotonicity of the VI vector field. We define
\begin{equation}\label{eq:def-Gz-VI}
{G}( z)
: =  \E \hat G^{(m)}( z ) 
= \E \sum_{t=1}^N \left(\phi( h \Lambda_t (z)) - y_t \right) \eta_t,
\end{equation}
and when $z$ is the true kernel $z^*$ we have
\begin{equation}\label{eq:G(zstar)=0}
G(z^*) = 0,
\end{equation}
as a result of Lemma \ref{lemma:time-only-bernoulli-glm} (in \eqref{eq:def-Gz-VI}, first taking conditional expectation conditioning on $y_i,\, i\le t-1$).
We will see that the strong monotonicity of $G(z)$ can be induced when $\E \sum_{t= 1}^N \eta_t \otimes \eta_t $ is strictly positive definite. 

To proceed, we introduce the following assumption on the domain $\Theta_K$ of $z$.

\begin{assumption}[Physical intensity]\label{assump:c1-Lambda-on-data}
There exist constants  $B, b > 0 $ s.t.

(i) The true conditional intensity $\Lambda_t^*$ satisfies 
\begin{equation}\label{eq:assump-Lamdbat-star} 
b \le \Lambda_t^* \le B, \quad \forall t= -N'+1, \cdots, N, \quad \text{a.s.}
\end{equation}

(ii)  There exists a non-empty, compact, and convex set $\Theta_K \subset \R^{NN'}$
such that the true kernel parameter  $z^* \in \Theta_K $, and for any $z \in \Theta_K$,
\begin{equation}\label{eq:assump-Lamdbat-b} 
b \le \Lambda_t( z )  \le B, \quad 
\forall t= 1, \cdots, N, \quad \text{a.s.}
\end{equation}

\end{assumption}

To clarify the relationship of $\Lambda_t^*$ and $z^*$, note that $z^*$ is the true kernel matrix restricted to the $NN'$ entries in \eqref{eq:def-theta-timeonly-Omega}. Thus $\Lambda_t^*  = \Lambda_t(z^*) $ when $ 1 \le t \le N $, 
and $\Lambda_t^*$ for $ -N' < t \le 0 $ also involves entries of the true kernel outside the $N N'$ index set.

We say that the intensity $\Lambda_t (z) > 0$ is ``physical''. 
Assumption \ref{assump:c1-Lambda-on-data}(ii)  imposes that the model intensity
$\Lambda_t (z) $ is upper and lower bounded by positive constants. 
For example, when $z = 0$, $\Lambda_t \equiv \mu$.
Then, as long as $[b, B]$ contains $\mu$, we have \eqref{eq:assump-Lamdbat-b} hold at $z=0$. 
Thus, one can initialize the iteration \eqref{eq:VI-SGD-scheme} from $z_0 = 0$ assuming that $0 \in \Theta_K$.
On the other hand, Assumption \ref{assump:c1-Lambda-on-data}(i) actually implies the following lemma:

\begin{lemma}\label{lemma:c1-Lambda-on-data-implication}
Under Assumption \ref{assump:c1-Lambda-on-data}(i), 
we have that $\mu \in [b,B]$, and any realization of the binary sequence $\{y_t, -N' <  t \le N\}$ happens w.p. $> 0$. Consequently, it always holds that $b \le \Lambda_t^*  \le B$, $\forall t= -N'+1, \cdots, N$. 
\end{lemma}

As a result, the requirement that \eqref{eq:assump-Lamdbat-b} holds a.s. means that the model intensity $\Lambda_t(z)$ also needs to stay in $[b,B]$ for any realization of the history sequence $\{ y_i, \, t - N'\le i \le t-1 \}$. This poses finitely many linear constraints on $z$ as a vector. 

We are ready to prove the strict positive definiteness of  $\E  \sum_{t= 1}^N \eta_t \otimes \eta_t $:

\begin{lemma}\label{lemma:eta-eta-positivity}
Under Assumption \ref{assump:c1-Lambda-on-data}(i),
\[
\E     \sum_{t= 1}^N \eta_t \otimes \eta_t   
\succeq  \rho I_{NN'},
\quad \rho := (1-e^{-hb})e^{-hB (N'-1)}.
\]
\end{lemma}

\begin{lemma}[Strong monotonicity of $G(z)$]\label{lemma:G-kappa-monotone}
Under Assumption \ref{assump:c1-Lambda-on-data}, 
$G(z)$ is $\kappa$-monotone on $\Theta_K$, where
\begin{equation}\label{eq:def-kappa-VI}
\kappa :=  e^{-hb} b h^2  e^{- B \tau_{\rm max}}.
\end{equation}
\end{lemma}

\begin{theorem}[Kernel recovery by VI]\label{thm:VI-convergence}
Under Assumption \ref{assump:c1-Lambda-on-data}, 
suppose $y^{(k)} = \{ y_t^{(k)} \}_t$ are i.i.d. trajectories across $k$,
and $z_k$ is computed from \eqref{eq:VI-SGD-scheme} starting from $z_0 \in \Theta_K$,
with step size
\[
\gamma_k = \frac{1}{\kappa (k+1)},
\]
where $\kappa$ is as in \eqref{eq:def-kappa-VI}.
Then, with $C:= B(T+ \tau_{\rm max} )  \wedge (N + N')$, 
we have
\[
\E \|z_k - z^* \|_2^2 \le \frac{4 C^2 }{\kappa^2} \frac{1}{k+1}, \quad k = 0, 1,  \cdots.
\]
\end{theorem}

As will be shown in the proof, the i.i.d. trajectory assumption is only used to ensure that  $\E [ \hat G^{(k)}(z_{k-1}) | y^{(m)}, m=1, \cdots k-1 ]  = G(z_{k-1})$ and thus can potentially be relaxed. 
Meanwhile, Theorem \ref{thm:VI-convergence} gives an ``in-expectation'' parameter recovery bound, and to derive a recovery guarantee beyond the in-expectation type result, one can use the i.i.d. or possibly relaxed independence assumptions on samples to induce concentration of the error $\| z_k - z^*\| $ around its expectation.
Such theoretical extensions are omitted here.

\begin{remark}[Scaling in the continuity limit]\label{rk:cont-limit-scaling}
We consider the limit when the length of the time grid interval $h \to 0$,
where the kernel matrix $K$ approaches the kernel function $k(t',t)$
and the time-uncertainty model (the unit-uncertainty setting) approaches the time-certain case.
We assume that the time horizon $T = h N$ and the influence time lag $\tau_{\rm max} = h N' $ stay as $O(1)$ constants in the limit, and we call this the {\it continuity limit}. 
In other words, $N \sim N' \sim 1/h$.
We also assume that the constants $b$ and $B$ in Assumption \ref{assump:c1-Lambda-on-data} are fixed positive constants. 
In this asymptotic, for large $N$ (and $N'$) and equivalently small $h$, $C \sim B(T + \tau_{\rm max}) $ is $O(1)$,
and $\kappa \sim h^2 e^{-B \tau_{\rm max}}$,
and thus the factor $C/\kappa$ in the convergence bound in Theorem    \ref{thm:VI-convergence}  scales as $e^{B\tau_{\rm max}}/h^2$, which $\sim h^{-2}$ and we kept the exponential factor in the constant. 
\end{remark}

In practice, we compute the VI using batch-based stochastic updates, and the constraint of $\Lambda_t$ being physical,
corresponding to the lower bound in Assumption \ref{assump:c1-Lambda-on-data}(ii),
is enforced by a barrier penalty. The details will be introduced in Section \ref{sec:algo}.

We give a few comments on the comparison of the GD and VI optimization dynamics:
First, consider the continuity limit (Remark \ref{rk:cont-limit-scaling}), 
since $b$ and $B$ are fixed $O(1)$ constants,
both the (in expectation)
VI monotonicity modulus $\kappa$ (Lemma \ref{lemma:G-kappa-monotone})
 and the GD convexity $\kappa'$ (Proposition \ref{prop:GD-strongly-convex})
 scale as $h^2 e^{-B \tau_{\rm max}}$, here we keep the $O(1)$ factor $e^{-B \tau_{\rm max}}$ which is exponential. 
 In short, $\kappa$ and $\kappa'$ both $\sim h^2$ which are comparable. 
 This suggests that, in theory, the in-expectation convergence behavior of the two optimization approaches is comparable; for example, the convergence rate should be the same up to a constant factor. 

 However, comparing the GD gradient field \eqref{eq:expression-hatFm} to the VI gradient field \eqref{eq:formula-Gtheta-timeonly}, we see that $\hat F^{(m)}$ has an extra multiplicative factor ${ h}/{ \phi( h  \Lambda_t^{(m)}( z) )} $ for each time $t$. While $\phi( h  \Lambda_t^{(m)}( z) ) \sim h$ under Assumption \ref{assump:c1-Lambda-on-data}, and thus this factor does not change the order of magnitude of the gradient field, the term $\phi( h  \Lambda_t^{(m)}( z) )$ on the denominator is a random variable and can introduce additional fluctuation especially when the conditional intensity at $t$ is small. 
This is consistent with our empirical observation that VI dynamics are more stable than GD, see Appendix \ref{apdx:GD-instability}.

\subsection{Special structures in influence kernel}\label{subsec:special-structure-psi}

\paragraph{Time-invariant kernel.}
When the influence kernel $k(t',t)$ is time-invariant, see Remark \ref{rk:stationary-kernel}, we know that the kernel matrix $K$ has the pattern that $K_{i,j} = \Psi_{i, j-i} = \psi_{j-i}$, where $\{ \psi_l, l=1,\cdots, N' \}$ is a vector of length $N'$.  Thus, the parameters in the kernel matrix are the vector $\psi \in \R^{N'}$, assuming that $\mu$ is known and fixed.
The likelihood model  \eqref{eq:def-Ltheta-timeonly}\eqref{eq:lm-theta-timeonly} still holds, where the kernel parametrization of $\Lambda_t^{(m)}$ is 
\begin{equation}\label{eq:lambda-psi-para-stationary}
\Lambda_t^{(m)}(\psi)= \mu + \langle \xi_t^{(m)},  \psi \rangle,
\quad \text{where }
\langle \xi_t, \psi \rangle = \sum_{i = t-N'}^{t-1} y_i \psi_{t-i}.
\end{equation}
The vector $\xi_t \in \R^{N'}$ consists of binary-valued entries determined by the history of $y_t$ from $t-N'$ to $t-1$.
The estimation by VI and GD applies here, and the optimization is reduced to solving for the vector $\psi$ of dimension $N'$ instead of the kernel parameter $\theta_K $ of dimension $NN'$  as in \eqref{eq:def-theta-timeonly-Omega}.
We provide additional details on VI recovery below.

Similar to \eqref{eq:formula-Gtheta-timeonly}, the VI vector field is 
\begin{equation}\label{eq:formula-Gtheta-timeonly-stationray}
\hat G_s( \psi ) =  
 \frac{1}{M} \sum_{m=1}^M 
 \hat G_s^{(m)}( \psi ),
 \quad 
 \hat G_s^{(m)}( \psi ) = 
\sum_{t=1}^N \left(\phi( h \Lambda_t^{(m)}( \psi )) - y_t^{(m)} \right) \xi_t^{(m)},
\end{equation}
and we use the subscript $_s$ to stand for ``stationary''. Define
\begin{equation}\label{eq:def-Gz-VI-stationary}
G_s( \psi )
: =  \E \hat G_s^{(m)}( \psi ) 
= \E \sum_{t=1}^N \left(\phi( h \Lambda_t (\psi)) - y_t \right) \xi_t,
\end{equation}
and again we have $G_s(\psi^*) = 0$ where $\psi^*$ is the true time-invariant kernel.
We assume that the true conditional intensity $\Lambda_t^*$ 
satisfies Assumption \ref{assump:c1-Lambda-on-data}(i), and similar to (ii) we assume that 

\vspace{5pt}
(ii')  There exists a non-empty, compact and convex set $\Theta_\psi \subset \R^{N'}$
such that the true kernel parameter  $\psi^* \in \Theta_\psi $, and for any $\psi \in \Theta_\psi$,
$b \le \Lambda_t( \psi )  \le B$, 
$\forall t= 1, \cdots, N$, a.s.
\vspace{5pt}

Lemma \ref{lemma:c1-Lambda-on-data-implication} still applies. The following lemma is a counterpart of Lemma \ref{lemma:G-kappa-monotone}.

\begin{lemma}[Strong monotonicity of $G_s(z)$]\label{lemma:G-kappa-monotone-stationary}
Suppose the influence kernel is time-invariant,
under Assumption \ref{assump:c1-Lambda-on-data}(i) and the assumption (ii') on $\Theta_\psi$  as above,
$G_z(\psi)$ is $\kappa_s$-monotone on $\Theta_\psi$, where
\begin{equation*}
\kappa_s :=  e^{-hb} bT  h  e^{- B \tau_{\rm max}}.
\end{equation*}
\end{lemma}

Based on Lemma \ref{lemma:G-kappa-monotone-stationary}, one can prove a VI recovery result similar to Theorem \ref{thm:VI-convergence}. Specifically, let $\psi_k$ be the sequence computed from the stochastic VI scheme; we have
\[
\E \| \psi_k - \psi^*\|_2^2 \le \frac{4C_s^2}{ \kappa_s^2} \frac{1}{k+1}, \quad k = 0, 1, \cdots,
\]
where $C_s$ is an $O(1)$ constant involving $B$, $T$ and $\tau_{\rm max}$ and derived using a similar argument as Lemma \ref{lemma:boundedness-VI}. Our asymptotic notation here refers to the continuity limit per Remark \ref{rk:cont-limit-scaling}. 
Note that $\kappa_s \sim h e^{-B \tau_{\rm max}}$ while previously with time-varying kernel $\kappa \sim h^2 e^{-B \tau_{\rm max}}$.
As a result, the factor $C_s/\kappa_s$ in the convergence bound $\sim h^{-1}$, while previously in Theorem \ref{thm:VI-convergence}, the factor $C/\kappa \sim h^{-2}$.  
This improvement is essentially because with a time-invariant kernel, we have $N' \sim h^{-1}$  many parameters to recover, while with a time-varying kernel, there are $N N' \sim h^{-2}$ many parameters.

\paragraph{Low-rank structure.}
There are different possible ways to impose low-rank structures on the influence kernel, 
and in our discrete-time model, naturally on the kernel matrix. Recalling the representation of the kernel matrix as $K$ and $\Psi$ in Section \ref{subsec:kernel-est-parametrization} and Figure \ref{fig:time-only-K-diag}.
We think it is natural to impose the low-rank structure on $\Psi$, as previously adopted in \cite{dong2022spatio} for influence kernel functions. 
In particular, the time-invariant kernel corresponds to the case where
$\Psi = {\bf 1}_N \psi^T$ (after filling the entries outside the region of $NN'$ entries to be recovered from data),
that is, $\Psi$ is a rank-1 matrix.

Theoretically, the kernel matrix is constructed by a grid average of the kernel function, see \eqref{eq:def-Kij-discrete-time}. Thus, the low rankness of the kernel matrix can be a result of the discretization of a continuous kernel function $k(t',t)$, assumed to have certain regularity with respect to the varying times $t'$ and $t$. 
In Section \ref{sec:experiments}, we experimentally study several simulated examples where the kernel matrix is induced from some continuous function $k$, and numerically, the matrix $\Psi$ can be approximated by a low-rank matrix.

We need some additional tricks when applying the low-rank constraint in practice. 
As has been explained in Section \ref{subsec:kernel-est-parametrization}, the $NN'$ learnable parameters in $\theta_K$ form a parallelogram-shaped region in the matrix $\Psi$ which is of size $(N+N')\times N'$, see Figure \ref{fig:time-only-Phi-illus}. Thus, even the underlying $\Psi$ has low-rankness (after filling the $(N+N')\times N'$ matrix), the inferred $\Psi$ by optimization has no values (or zero values) outside the parallelogram region, and this interferes with the low-rankness. 
To overcome this issue, observe that when $\Psi$ is rank-$r$, then any $m\times N'$ submatrix $\Psi'$ of $\Psi$ (by retrieving rows) is at most rank-$r$,  and the rows of $\Psi'$ are spanned by the $r$ right singular vectors of $\Psi$. We choose the submatrix $\Psi' $ to be the ``middle chunk'' of $\Psi$ that has full rows inside the parallelogram region, indicated by dark blue entries in Figure \ref{fig:time-only-Phi-illus}. In practice, we impose low rankness on $\Psi'$ using a truncated SVD, and use the right singular vectors of $\Psi'$ to apply a projection of the inferred full matrix $\Psi$, see more in Section \ref{sec:algo}.

\section{Time-uncertainty point process on network}\label{sec:model-with-graph}

In this section, we extend the time uncertainty model to the network setting, namely when the event data also contains location information $u_i$ for the $i$-th event and $u_i$ is a node on a graph. 
All the proofs can be found in Appendix \ref{app:proofs-sec:5}.

\subsection{Continuous-time formulation}\label{subsec:model-graph-cont-time}

\paragraph{Filtration and conditional intensity.}
We consider the event data on a network $\calV = \{1, \cdots, V \}$ where the event time $t_i$ has uncertainty, and the event location $u_i \in \calV$.
The classical scenario with exact event time is provided in \eqref{eq:ti-ui-on-graph}, as a comparison. 
Following the formulation in Section \ref{sec:model-time-only}, we consider $t_i \in \Delta_i$ same as before. 
Again, let $\Delta_i$ also denote the event that $\{ t_i \in \Delta_i \}$, the filtration is defined as 
\[
\tilde{\calF}_k :=  \sigma\{ ( \Delta_i, u_i), \, i=1, \cdots, k \}.
\]
We define the survival function $S_k(t)$ and hazard function $h_k(t)$ same as in \eqref{eq:def-Sk-time-only} and \eqref{eq:def-hk-old-hawkes}, respectively, that is, $S_k(t)$ and $h_k(t)$ only addresses the time of the event regardless of the event location.
We then have that \eqref{eq:def-hk-time-uncert}\eqref{eq:Sk-expression-by-hk} also hold.

To be able to incorporate event location information $u_i$, we define $\tilde{f}_k(u|t)$ by
\begin{equation}\label{eq:def-tilde-fk}
\tilde{f}_k(u|t)  
:= \Pr [ u_k = u| t_k \in [t, t+dt), \, \tilde{\calF}_{k-1} ],
\end{equation}
and then, for $t > \Delta_{k-1}^r$,
\begin{align}
\Pr[ t_k \in [t, t+dt), u_k =u | \tilde{\calF}_{k-1} ]
& = \Pr[ t_k \in [t, t+dt)  | \tilde{\calF}_{k-1} ]\tilde{f}_k(u|t)  \nonumber \\
& = -S_k'(t) \tilde{f}_k(u|t) dt  \nonumber  \\
& = S_k(t) h_k(t) \tilde{f}_k(u|t) dt 
    \quad \text{(by \eqref{eq:def-hk-time-uncert})}\nonumber  \\
& = S_k(t) \lambda( t, u) dt, \label{eq:Pr-derivation-on-graph-1}
\end{align}
where we define 
\begin{equation}\label{eq:def-lambda-tu}
\lambda(t,u) := h_k(t) \tilde{f}_k(u|t).
\end{equation}
By that $\tilde{f}_k(u|t)$ is a conditional distribution of $u \in \calV$ and thus
$\sum_{u \in\calV} \tilde{f}_k(u|t) = 1 $,  we have
\begin{equation}\label{eq:equation-bar-lambda(t)}
\sum_{u \in\calV} \lambda(t,u) = h_k(t) =: \bar{\lambda}(t),
\end{equation}
where we define $\bar{\lambda}(t)$ to equal $h_k(t)$ piece-wisely on $(\Delta_{k-1}^r, \Delta_{k}^r]$, similarly as in 
\eqref{eq:def-lambdat-time-only}.
Back to \eqref{eq:Pr-derivation-on-graph-1},
by integrating $dt$ over the interval $\Delta_k$ on both sides,
we have
\begin{align}
\Pr [ t_k \in \Delta_k, \, u_k = u | \tilde{\calF}_{k-1}]
& = \int_{\Delta_k^l}^{\Delta_k^r} S_k(t) \lambda( t, u) dt  \nonumber  \\
& = \int_{\Delta_k^l}^{\Delta_k^r} \lambda( t, u) e^{ - \int_{\Delta_{k-1}^r}^t h_k(s) ds} dt  \quad \text{(by \eqref{eq:Sk-expression-by-hk})}
 	 \label{eq:Pr-tk-cont-time-uncert-on-graph-form-a}\\
& = e^{ - \int_{\Delta_{k-1}^r}^{\Delta_{k}^r}  \bar{\lambda} (s) ds}  
	\int_{\Delta_k^l}^{\Delta_k^r} \lambda( t, u) e^{ \int_t^{\Delta_{k}^r}  \bar{\lambda} (s) ds} dt.
	\label{eq:Pr-tk-cont-time-uncert-on-graph}
\end{align}

\paragraph{Likelihood model.}
With these notions in hand, in the following lemma, we derive the expression of the likelihood 
\[
\calL :=  \Pr [   (\Delta_1, u_1),  \cdots, ( \Delta_n, u_n), \, t_{n+1} > T] 
\]
\begin{lemma}\label{lemma:logL-cont-time-uncert-on-graph}
Recall that $\bar{\lambda}(t) = \sum_{u \in\calV} \lambda(t,u)$, 
\begin{equation}\label{eq:logL-cont-time-uncert-on-graph}
\log \calL = \sum_{k=1}^n \log \left( \int_{\Delta_k^l}^{\Delta_k^r} \lambda( t, u_k) e^{ \int_t^{\Delta_{k}^r}  \bar{\lambda} (s) ds} dt
		\right) - \int_0^T \bar{\lambda}(s) ds.
\end{equation}
\end{lemma}

\begin{remark}[Reduction to classical scenario without time uncertainty]
We consider the limit where the length of the window $\Delta_i$ shrinks to zero, 
that is, $\Delta_i^{l}, \, \Delta_i^{r} \to t_i$  the certain event time for all $i$.
We claim that in this case, 
\begin{equation}\label{eq:Pr-tk-cont-time-uncert-on-graph-limit-tk}
\frac{1}{|\Delta_k|}\Pr [ t_k \in \Delta_k, \, u_k = u | \tilde{\calF}_{k-1}]
\to \lambda( t_k, u)  e^{ - \int_{t_{k-1}}^{t_{k}}  \bar{\lambda} (s) ds}.
\end{equation}
This is can be derived from \eqref{eq:Pr-tk-cont-time-uncert-on-graph-form-a}, which gives that the l.h.s. can be written as 
\[
\frac{1}{|\Delta_k|} \int_{\Delta_k^l}^{\Delta_k^r} \lambda( t, u) e^{ - \int_{\Delta_{k-1}^r}^t \bar{\lambda}(s) ds} dt, 
\]
and then using the continuity of $\lambda(t,u)$ (and $\bar{\lambda}(t)$) with respect to $t$.
The equation \eqref{eq:Pr-tk-cont-time-uncert-on-graph-limit-tk} leads to the expression
\[
\Pr [ (t_k, u) | \calF_{k-1}] = \lambda( t_k, u)  e^{ - \int_{t_{k-1}}^{t_{k}}  \bar{\lambda} (s) ds},
\]
where we recall that in the classical (time-certain) model, 
$\calF_{k}:=\sigma\{ (t_i, u_i), \, i=1,\cdots, k \}$. 
Then, the log-likelihood of the $n$ events can be shown to be
\[
\log \calL = \sum_{k=1}^n \log \lambda( t_k, u_k ) - \int_0^T  \bar{\lambda} (s) ds,
\]
which recovers the classical model without time uncertainty, see e.g. Eqn (8) in \cite{reinhart2018review}. 
\end{remark}

\begin{remark}[Reduction to time-only model]
Consider the special case where the network $\calV$ only has one node $u$.
In this case, $\bar{\lambda}(t)= \lambda(t,u)$ and define it to be $\lambda(t)$.
Then, one can verify that
\[
\int_{\Delta_k^l}^{\Delta_k^r} \lambda( t) e^{ \int_t^{\Delta_{k}^r}  {\lambda} (s) ds} dt
= e^{ \int_{\Delta_k^l}^{\Delta_{k}^r}  {\lambda} (s) ds} -1.
\]
This equality reduces the conditional probability \eqref{eq:Pr-tk-cont-time-uncert-on-graph} to \eqref{eq:Pr-Deltak-time-only},
and the log-likelihood \eqref{eq:logL-cont-time-uncert-on-graph} to \eqref{eq:l-lambda-cont-time-uncert}.
That is, the model is reduced to the time-only case in Section \ref{subsec:cont-time-time-only}.
\end{remark}

\paragraph{Influence kernel.}
We introduce the spatial-temporal influence kernel function $k(t', t; u', u)$ to parametrize the intensity function $\lambda(t,u)$.
By incorporating the spatial dependence in \eqref{eq:lambda-phi-cont-time-uncert}, we have
\begin{equation}\label{eq:lambda-phi-cont-time-on-graph}
    \lambda[k](t, u) = \mu(u) + \sum_{i, \Delta_i^r < t } \frac{1}{|\Delta_i|} \int_{\Delta_i^l}^{\Delta_i^r} k( t', t; u_i, u) dt',
\end{equation}
where $\mu(u) > 0$ for $u \in \calV$ is the base intensity on the network $\calV$.
Based on the continuous-time model \eqref{eq:logL-cont-time-uncert-on-graph} and \eqref{eq:lambda-phi-cont-time-on-graph}, we will derive a discrete-time model in the next subsection under additional assumptions.

\subsection{Discrete-time model and discrete parametrization}

We consider the time grid as in Section \ref{subsec:discrete-general-uncertainty} which has $N$ intervals $I_j$, $|I_j| = h=T/N$. 
For simplicity, we only consider the unit uncertainty setting as in Assumption \ref{assump:A3}(A3).

\paragraph{Time-quantized influence kernel.}
Suppose $\Delta_k = I_{j(k)}$, the log-likelihood \eqref{eq:logL-cont-time-uncert-on-graph} can be written as
\begin{equation}\label{eq:logL-cont-time-uncert-on-graph-A3}
\log \calL = \sum_{k=1}^n \log \left( \int_{I_{j(k)}^l}^{I_{j(k)}^r} \lambda( t, u_k) e^{ \int_t^{I_{j(k)}^r}  \bar{\lambda} (s) ds} dt
		\right) - \int_0^T \bar{\lambda}(s) ds.
\end{equation}
This, however, still cannot be expressed by the average of $\lambda(t,u)$ on $t\in I_j$ only,
due to that $\lambda( t, u) e^{ \int_t^{I_{j(k)}^r}  \bar{\lambda} (s) ds}$ takes different value over location $u$ and $t \in I_{j(k)}$.
To reduce to a discrete problem, we introduce the following assumption on the influence kernel.

\begin{assumption}[Time-quantized influence kernel on network]\label{assump:A4}
For any $u', u \in \calV$ and $ 1 \le i < j \le N $, 
$k(t', t; u', u)$ is constant on $(t',t) \in I_i \times I_j$, and we denote the value as $K_{i,j}(u', u)$.
\end{assumption}
A useful consequence of the assumption is that the influence function $\lambda(t,u)$ is also quantized in time, which is proved in the following lemma.

\begin{lemma}[Discrete-time intensity]\label{lemma:quant-time-lambda(t,u)}
Under (A3) and Assumption \ref{assump:A4}, 
for any $u \in \calV$, $\lambda(t,u)$ as defined in \eqref{eq:lambda-phi-cont-time-on-graph} is constant over $t \in I_j$.
We denote the value as $\Lambda_j(u)$, and
\begin{equation}\label{eq:Lambdaj-base-on-graph}
\Lambda_j  (u) = \mu(u) +  \sum_{\{k: \, j(k) < j\} }  K_{j(k), j}(u_k, u).
\end{equation}
\end{lemma}

\paragraph{Likelihood model and learnable parameters.}
According to \eqref{eq:equation-bar-lambda(t)}, we define
\begin{equation}\label{eq:def-bar-Lambda}
\bar{\Lambda}_j = \sum_{u \in \calV} {\Lambda}_j(u).
\end{equation}
The following lemma derives the log-likelihood which only involves $\{ {\Lambda}_j(u), \, j \in [N], u \in \calV \}$.
\begin{lemma}[Discrete-time likelihood on network]
\label{lemma:l-lambda-base-on-graph}
Under Assumption \ref{assump:A3}(A3) and Assumption \ref{assump:A4}, 
the log-likelihood defined in  \eqref{eq:logL-cont-time-uncert-on-graph}  can be written as
\begin{equation}\label{eq:l-lambda-base-on-graph}
\ell := \log \calL 
= \sum_{k=1}^n \log \left(  (e^{  h \bar{\Lambda}_{j(k)}} -1) \frac{ \Lambda_{j(k)} (u_k)}{ \bar{\Lambda}_{j(k)}}
		\right) -   h  \sum_{j=1}^N  \bar{\Lambda}_j,
\end{equation}
\end{lemma}

The parameters in the model are the baseline intensities $\mu(u)$ and the influence kernel $K_{i,j}(u',u)$, both of which are on the network, which determines the network conditional intensities $\Lambda_j(u)$ as in \eqref{eq:Lambdaj-base-on-graph}. Adopting the same time-grid setup in Section \ref{subsec:kernel-est-parametrization}, we consider the extended time grid from $-N'+1$ to $N$, where the max time lag for influence is $\tau_{\rm max} = N' h$.
Denoting the learnable parameters as $\theta$,  we have
\begin{equation}\label{eq:def-theta-network-Omega}
\theta = \{ \mu(u), \, u\in \calV \} \cup \{ K_{i,j}(u', u), \, u' , u \in \calV, \, 
	0 <j-i \le N', \,  -N' < i \le N, \, 0 < j  \le N\}.
\end{equation}
We similarly introduce the rearranged kernel matrix $\Psi$, where
\begin{equation}
\label{eq:def-Psi-tensor}
\Psi_{i,l}(u',u) = K_{i, i+l} (u',u), \quad \forall u', u \in \calV, \quad l=1, \cdots, N'.
\end{equation}
We will discuss the estimation of parameters from data in Section \ref{subsec:kernel-recovery-network}.

\paragraph{Prediction on future interval.}
Extending the argument in Section \ref{subsec:bernoulli-time-only}, we can predict the probability of having the next event happen during a future interval and at a location on $\calV$ once the parameter $\theta$ is learned from data. 
Specifically, recall the notations in the continuous-time setting in Section \ref{subsec:model-graph-cont-time}, we are given $n$ historical events and the goal is to predict the probability of the $({n+1})$-th event happening on the interval $( l,r]$, $r > l \ge  \Delta_{n}^r$.
Using the discrete-time model under (A3), $\Delta_n = I_{j(n)}$, and again suppose $l$ and $r$ are endpoints on the time grid, we have
\[
l = j_l h , \quad r = j_r h , \quad j_r > j_l \ge j(n).
\]
We then have
\begin{align}
& \Pr [ t_{n+1}\in  ( l= j_l h  , r = j_r h  ], \, u_{n+1} = u | \tilde{\calF}_{n}]
 = \int_l^r  \lambda(t,u) e^{- \int_{\Delta_{n}^r}^t \bar{\lambda}(s) ds} dt
	\quad \text{(by \eqref{eq:Pr-tk-cont-time-uncert-on-graph-form-a})} 	
	 \nonumber \\
&~~~
 = \sum_{j = j_l+1}^{j_r}  
	\frac{\Lambda_j(u)}{\bar{\Lambda}_j} ( 1- e^{- h \bar{\Lambda}_j} ) 
	e^{ - h \sum_{j' = j(n)+1}^{j-1}   \bar{\Lambda}_{j'}},
	\label{eq:predict-future-interval-network}
\end{align}
where the second equality is derived similarly to the proof of Lemma \ref{lemma:l-lambda-base-on-graph}, making use of Lemma \ref{lemma:quant-time-lambda(t,u)} under Assumption \ref{assump:A4}.
The probability of having no event till $l$ is 
\begin{align}
\Pr [ t_{n+1} > l = j_lh  | \tilde{\calF}_{n}]
& = 	e^{ - h \sum_{j' = j(n)+1}^{j_l }   \bar{\Lambda}_{j'}},
\end{align}
which is similar to the time-only case
where we replace ${\Lambda}_{j'}$ with  $\bar{\Lambda}_{j'}$ in the expression \eqref{eq:predict-future-2}.

\subsection{Spatial-temporal Bernoulli process on network}\label{subsec:bernoulli-graph}

We have derived the discrete time model  \eqref{eq:l-lambda-base-on-graph} and \eqref{eq:Lambdaj-base-on-graph} under 
(A3) and Assumption \ref{assump:A4}.
Here, we introduce an equivalent formulation of the model using a spatial-temporal  Bernoulli process on the graph $\calV$.
Below, we use the boldface notation to emphasize vectors. For example, 
$\boldsymbol{\mu} = \{ \mu(u), u \in \calV \} \in \R^{V}$ 
and $\mathbf{\Lambda}_j = \{ {\Lambda}_j(u), u \in \calV \} \in \R^V$
are length-$V$ vectors.
We use the boldface $\mathbf{x}$ to denote the vector, and $x(u)$ or $x_u$ stands for the $u$-th entry.
When to emphasize the dependence on parameter $\theta$, we write $\mathbf{\Lambda}_j [\theta]$ for the vector and ${\Lambda}_j [\theta](u)$ for the $u$-th entry.

Given a trajectory of discrete-time unit-uncertainty event data, define a vector-valued binary sequence $(\mathbf{y}_1, \cdots, \mathbf{y}_N)$, and for $t\in [N]$, $\mathbf{y}_t \in \{ 0,1 \}^V$, where
$y_t( u) = 1$ if there exists an event $(\Delta_i ,u_i)$  s.t. $\Delta_i = I_{t}$ and $u_i = u$;
$y_t( u) = 0$ otherwise. 
We also define 
\[
\bar{y}_t : = \sum_{u \in \calV} y_t(u), \quad t \in [N],  
\]
and $\bar{y}_t = 1$ indicates that within time interval $I_t$, there is an event happening somewhere on the network. 
Using the binary vectors $\mathbf{y}_t$, 
\eqref{eq:Lambdaj-base-on-graph} is equivalent to 
\begin{equation}\label{eq:Lambdaj-base-on-graph-bernoulli}
\Lambda_t [\theta] (u) = \mu(u) +  \sum_{ i < t}  \sum_{u' \in \calV } y_i(u')  K_{i, t}(u', u),
\end{equation}
and the log-likelihood \eqref{eq:l-lambda-base-on-graph} can be expressed as 
\begin{equation}\label{eq:l-lambda-base-on-graph-beroulli}
\ell 
= \sum_{t=1}^N \left\{ 
		\bar{y}_t \log \left(  (1- e^{- h \bar{\Lambda}_{t}} )  
					    \sum_{u \in \calV} y_t(u) \frac{ \Lambda_{t} (u) }{ \bar{\Lambda}_{t}} \right) 
		-   (1- \bar{y}_t)h \bar{\Lambda}_t \right\}.
\end{equation}

To ensure that the conditional intensity is physical, i.e. $\Lambda_t [\theta](u) > 0$ for any $u \in \calV$ and all time $t$, one can introduce an assumption similar to Assumption \ref{assump:c1-Lambda-on-data}
and require that, for positive constants $b < B$,
$\Lambda_t[\theta](u) \in [b,B]$ for all $u$ and $t$ a.s. when model parameter $\theta$ lies in certain domain,
including the true intensity $\Lambda_t^*(u)$.
As a result, $\mathbf{\Lambda}_t^* \in [b, B]^V \subset \R_+^V$ for all $t$.

Use the vector form of $\mathbf{\Lambda}_t$ and $\mathbf{y}_t$ which are in $\R^V$, we have the following lemma which characterizes the general linear model of predicting $y_t$ given history using $\mathbf{\Lambda}_t$, and this gives a counterpart of Lemma \ref{lemma:time-only-bernoulli-glm} on network.
In this lemma, $\mathbf{\Lambda}_t$ refers to the true intensity $\mathbf{\Lambda}_t^*$ induced by the true model parameters, which determine the law of the sequence $\{\mathbf{y}_t\}_t$.

\begin{lemma}\label{lemma:on-graph-bernoulli-glm}
Under Assumption \ref{assump:A3}(A3) and Assumption \ref{assump:A4}, 
and using the formulation of the Bernoulli process $\mathbf{y}_t$ on $\calV$,
for any $ t\in [N]$, $ \mathbf{\Lambda}_t  \in \sigma\{ \mathbf{y}_i, i\le  t-1 \}$ and 
\begin{equation}\label{eq:on-graph-bernoulli-glm}
\E  [ \mathbf{y}_t | \mathbf{y}_i, i \le  t-1 ]
= \left( \phi( h \bar{\Lambda}_{t} )\frac{ \Lambda_{t} (u) }{ \bar{\Lambda}_{t}} \right)_{u \in \calV}
= \Phi(  h \mathbf{\Lambda}_t  ), 
\end{equation}
where $\Phi: \R_+^V \to \R_+^V$ is defined by 
\begin{equation}\label{eq:def-Phi-mapping}
\Phi( \boldsymbol x ) := \frac{1- \exp\{ - \boldsymbol x^T \mathbf{ 1} \}}{ \boldsymbol x^T \mathbf{ 1} }  \boldsymbol x.
\end{equation}
Because $\boldsymbol x \in \R_+^V$ satisfies  $\boldsymbol x^T \mathbf{ 1} > 0$, \eqref{eq:def-Phi-mapping} is well-defined.
\end{lemma}

\subsection{Influence kernel recovery by VI and GD for networks}\label{subsec:kernel-recovery-network}

\paragraph{Log-likelihood of kernel parameter.}

Suppose we are given $M$ i.i.d. sequence of observed Bernoulli processes $\{ \mathbf{y}_t^{(m)}\}_{m=1}^M$ on the network $\calV$, which are of the same length indexed by $ -N' < t \le N$.
For simplicity, we assume that the baseline intensity vector $\boldsymbol \mu$ is known and fixed, 
and we will explain the learning of $\boldsymbol \mu$ from data in Section \ref{sec:algo}. 

We denote model parameter $\theta$ in \eqref{eq:def-theta-network-Omega} as $\theta = \{ \boldsymbol \mu\} \cup \theta_K$, 
where $\theta_K$ consists of the $NN' V^2$ many influence kernel parameters, again denoted by $z$. 
We write $\mathbf{\Lambda}_t[z] = \mathbf{\Lambda}_t [\boldsymbol \mu, z]$, i.e., with $\theta = [\boldsymbol \mu, z]$, and similarly for other quantities depending on $\theta$. 
Next, we introduce the data history vector $\eta_{t,u}$, doubly indexed by $t$ and $u$ and  in $\R^{NN' V^2}$, s.t.
\begin{equation}\label{eq:def-innerprod-etatu-K-graph}
\langle \eta_{t, u}, z \rangle = \sum_{ i = t-N'}^{t-1}  \sum_{u' \in \calV } y_i(u')  K_{i, t}(u', u).
\end{equation}
This allows \eqref{eq:Lambdaj-base-on-graph-bernoulli} to be represented as 
$\Lambda_t [z] (u)  = \mu(u) + \langle \eta_{t, u}, z \rangle$,
 and $\bar \Lambda_t [z] = \sum_{u \in \calV} \Lambda_t [z] (u)$. 
With this notation, by \eqref{eq:l-lambda-base-on-graph-beroulli}, the log-likelihood of parameter $\theta$ on the $m$-the sequence can be written as 
\begin{align}
\ell^{(m)}[z] 
& = \sum_{t=1}^N \left\{ 
		\bar{y}^{(m)}_t \log \left(  (1- e^{- h \bar{\Lambda}^{(m)}_{t}[z]} )  
		\sum_{u \in \calV} y_t^{(m)}(u) \frac{ \Lambda^{(m)}_{t}[z] (u) }{ \bar{\Lambda}^{(m)}_{t}[z]} \right) 
		-   (1- \bar{y}^{(m)}_t)h \bar{\Lambda}^{(m)}_t[z] \right\},  
		\label{eq:lm-z-graph}\\
& \Lambda_t^{(m)} [z] (u) 
 = \mu(u) + \langle \eta^{(m)}_{t, u}, z \rangle, \quad 		
 	\bar \Lambda_t^{(m)} [z] = \sum_{u \in \calV} \Lambda_t^{(m)} [z] (u),
	\label{eq:lambdatm-z-graph}
\end{align}
and the total likelihood is averaged over the $M$ trajectories.

\paragraph{VI scheme on network.}
With the mapping $\Phi$ defined as in \eqref{eq:def-Phi-mapping}, we can again conceptually define the per-trajectory $L^2$ loss as
\[
\ell_{\rm VI}^{(m)}( z )
  = \frac{1}{2 h }\sum_{t=1}^N  \| \Phi( h \mathbf{\Lambda}_t^{(m)}[z] ) - \mathbf{y}_t^{(m)} \|_2^2.
\]
For a vector $y \in \R^V$, we use both bracket $y(u)$ or subscript $y_u$ to denote the $u$-th entry of the vector. 
Following the stochastic update scheme in Section \ref{subsec:VI-recovery-timeonly}, the per-trajectory VI field can be written as 
\begin{align}
\hat G^{(m)}[ z ] 
& =  \sum_{t=1}^N \sum_{u \in \calV } \left(\Phi ( h \mathbf{\Lambda}_t^{(m)}[ z ]) -  \mathbf{y}_t^{(m)} \right)_u \eta_{t,u}^{(m)} 
		\nonumber \\
& = \sum_{t=1}^N \sum_{u \in \calV } \left( \phi( h \bar \Lambda_t^{(m)}[ z ]) \frac{\Lambda_t^{(m)}[ z ](u)}{\bar \Lambda_t^{(m)}[ z ]}  -  y_t^{(m)} (u) \right) \eta_{t,u}^{(m)}.
	\label{eq:formula-Gtheta-graph}
\end{align}
We update the solution $z_k$ at the $k$-th iteration by looping over the training trajectories similarly to before. 
It is possible to extend the recovery guarantee theory in Section  \ref{subsec:VI-recovery-timeonly} to the network case, which we leave for future work.

\paragraph{GD scheme on network.}
The GD dynamic solves for the MLE based on the likelihood model \eqref{eq:lm-z-graph}\eqref{eq:lambdatm-z-graph}.  Again, we focus on the stochastic GD update, 
and by differentiating \eqref{eq:lm-z-graph}, the per-trajectory gradient field can be computed directly to be
\begin{equation}\label{eq:expression-hatFm-graph}
\hat F^{(m)}(z) = 
\sum_{t=1}^N  \sum_{u \in \calV}
 	\left( \frac{ h ( \phi( h  \bar \Lambda_t^{(m)}[ z] ) - \bar y_t^{(m)} )}{ \phi( h \bar \Lambda_t^{(m)} [z] )}     
	+ \frac{\bar y^{(m)}_t}{\bar \Lambda_t^{(m)}[ z] } 
	- \frac{y^{(m)}_t(u)}{ \Lambda_t^{(m)}[ z](u)}
	\right)
	\eta_{t,u}^{(m)},
\end{equation}
where we used that when $y_t(u) =1$, $\sum_{u' \in \calV} y_t(u') \Lambda_t(u') = \Lambda_t(u)$ because when time $t$ is even, then only one node can have an event on it.

We see that \eqref{eq:expression-hatFm-graph} contains terms like 
${y_t(u)}/{ \Lambda_t(u)}$ which will be ${ \Lambda_t(u)}^{-1}$ at an event location $u$, 
and thus can be large when ${ \Lambda_t(u)}$ is small. 
In comparison, the VI field in \eqref{eq:formula-Gtheta-graph} does not have such terms
(the ratio $\Lambda_t(u)/\bar \Lambda_t$ is always between 0 and 1, which will not cause divergent magnitude).
This suggests that the VI training dynamic can be more stable than the GD dynamic, echoing the same phenomenon in the time-only case (see comments at the end of Section \ref{subsec:GD-shceme-timeonly}).

\paragraph{Special structures for spatial-temporal kernels.}

Similar to Section \ref{subsec:special-structure-psi}, one can incorporate special structures of the kernel structure into the inference procedure when such prior knowledge is available. For example, assuming time stationarity, one will have a time-invariant kernel 
$K_{i,j}(u',u) = \Psi_{i, j-i}(u', u) = \psi_{j-i}(u', u)$, $\forall u, u' \in \calV$,
and there are $N' V^2$ many kernel parameters to learn from data.
 The likelihood model \eqref{eq:lm-z-graph}\eqref{eq:lambdatm-z-graph} still holds, and the VI and GD optimization methods can also be extended. 
 
In the network case, because the spatial-temporal kernel parameters $K$ or $\Psi$ are four-way tensors, there are different ways to impose low-rank structures on the kernel. 
Motivated by the underlying temporal continuity of the influence kernel function, as has been explained in Section \ref{subsec:special-structure-psi}, one can expect approximate low-rankness of the discrete kernel over time, e.g., the matrix $[ \Psi_{i,l}(u',u) ]_{i,l}$ for each pair of $(u', u)$.
As an additional simplification, we concatenate these matrices over all pairs of $(u', u)$ to form a tall matrix of size $(N-N'+1)V^2 \times N'$, using the middle-chunk trick in Section \ref{subsec:special-structure-psi}, and impose low-rankness on the resulting matrix. 

The above two techniques exploit the temporal structures of the kernel, and it is also possible to impose constraints or regularizations based on the spatial information, i.e. graph topology of the network. 
In the real data application of this work, to discover a network structure using the Sepsis-Associated Derangements (SADs) data (Section \ref{sec:experiments}), we do not assume any prior knowledge of the network where graph nodes represent medical status. Instead, we will use the inferred spatial-temporal kernel from event data to reveal influence relations among the nodes. 
Thus, in this work, we do not explore any spatial structure or regularization of the kernel, which is possible and would certainly be useful in other application scenarios.

\section{Algorithm}\label{sec:algo}

The estimation methods using VI and GD have been introduced for the time-only case in Section \ref{sec:inference-GD-VI} and extended to the on-network case in Section \ref{subsec:kernel-recovery-network}. We provide additional algorithmic details in this section.

\subsection{Time-only case}\label{subsec:algo-time-only}

\begin{algorithm}[t!]
\DontPrintSemicolon
\SetAlgoNoLine
\caption{Stochastic batch-based training for time-only case}
\label{algorithm:train-time-only}

\SetKwInOut{Input}{Input}
\SetKwInOut{Output}{Output}

\Input{
Training trajectories $\{ y_t^{(m)} \}_{m=1}^{M}$, 
$\mu_0$ as initial value of $\mu$;
Parameters: 
intensity lower bound $b$, barrier weight $\delta_{b}$ ; 
batch size $M_B$, 
maximum number of epochs $k_{\rm max}$, 
learning rate schedule $\{ \gamma_k \}_{k}$;
Optional: singular value threshold $\tau_\SVD$, smoothness weight $\delta_s$
}

\Output{
Learned kernel $\theta_K$ and baseline $\mu$ }

\vspace{+5pt}

Initialize kernel $\theta_K \leftarrow 0$ 
and baseline $\mu \leftarrow \mu_0$

\For{$k = 1,\ldots, k_{\rm max}$}{


    \While{loop over the batches}{

    $\theta \leftarrow [\mu, \theta_K]$

    Load $M_B$ many training trajectories to form batch $Y_B =\{y^{(m)} \}_{m}$

    Compute intensities $\Lambda_t^{(m)}( \theta )$ in \eqref{eq:lambdatm-theta-timeonly} at $t=1, \cdots, N$ for each $y^{(m)}$ in $Y_B$, and let 
    $b^{(m)} \leftarrow \min_{ 1 \le t \le N} \Lambda_t^{(m)}$

    Split $Y_B = Y_{B,1}\cup Y_{B,2}$, where $y^{(m)}$ in $Y_{B,1}$ if $b^{(m)} < b$, and otherwise in $Y_{B,2}$

   \uIf{using VI update}{
    $g_B \leftarrow \frac{1}{M_B}
            \sum_{y^{(m)} \in Y_{B,2}} \hat G^{(m)}(\theta)$,
    $\hat G^{(m)}$ as in \eqref{eq:formula-Gtheta-timeonly}             
            }
    \ElseIf{using GD update}{
    $g_B \leftarrow \frac{1}{M_B}
            \sum_{y^{(m)} \in Y_{B,2}} \hat F^{(m)}(\theta)$,
     $\hat F^{(m)}$ as in \eqref{eq:expression-hatFm}
            }
            
    \If{$Y_{B,1}$ not empty}{
    $g_B \gets g_B + \delta_b \sum_{y^{(m)} \in Y_{B,1}} \partial_{\theta_K} B^{(m)}(\theta)$, $B^{(m)}$ as in \eqref{eq:def-Bm-barrier}
    } 

    Update 
    $\theta_K \gets \theta_K - \gamma_k g_B$
    
    Solve $\tilde \mu$ as root of \eqref{eq:eqn-elim-mu-timeonly} on $Y_B$ using bisection search 

    Update $\mu \gets 0.9\mu + 0.1\tilde \mu$
    }
    
    (Optional)
    \If{do smoothness regularization of kernel}{
    $\theta_K \gets \theta_K - \gamma_k \delta_s \partial_{\theta_K} S(\theta_K)$, $S$ as in \eqref{eq:def-Stheta-smoothness}
    }
    
}

(Optional) \If{do low-rank truncation of kernel}{
    Rearrange entries in $\theta_K$ into $\Psi$, and
    $\Psi' \leftarrow \Psi_{(N'+1):N,1:N'}$
    
   $r \leftarrow$ number of singular values of $\Psi'$ larger than $\tau_\SVD$

   $V \leftarrow $ first $r$ right singular vectors of $\Psi'$

   $\Psi \gets \Psi{V}{V}^\top$, and rearrange $\Psi$ into $\theta_K$
}
\KwRet{$\mu$, $\theta_K$}
\end{algorithm}

\paragraph{Stochastic optimization by batches.}
We theoretically proved in Section \ref{subsec:VI-recovery-timeonly} the (in-expectation) convergence of the stochastic scheme \eqref{eq:VI-SGD-scheme}, where one training trajectory is used in each step of updating the model parameter, i.e., the batch size is one. This choice is for the convenience of exposition for the theory.
In practice, the batch size in a stochastic optimization algorithm balances between the efficiency of parameter adaptation versus random fluctuation of the solution. When the batch size is small, the solution can develop large variations across batches. 

In our algorithm, we use a finite batch size $M_B$, which can be up to a few hundred. 
In each step of parameter update, we load $M_B$ many training trajectories, called a ``batch'' and denoted by $Y_B$. The batches are non-overlapping and loop over the training set, and when one pass of the training set is finished, it is called an ``epoch''. We random shuffle the training set indices at the beginning of each epoch. 
Following the convention of stochastic optimization literature, the parameter update step size $\gamma_k$ is called ``learning rate. Instead of decreasing $\gamma_k$ per batch (with batch size one) as theoretically in Theorem \ref{thm:VI-convergence}, we set a learning rate schedule indexed by the epochs till the maximum number of epochs. 
The stochastic optimization algorithm is summarized in Algorithm \ref{algorithm:train-time-only}.

\paragraph{Inferenece of baseline intensity $\mu$.}
To be able to infer the baseline $\mu$ from data, our idea is to eliminate the variable $\mu$ in the process of learning $\theta_K$, making use of the first-order condition of the log-likelihood objective.
Specifically, from \eqref{eq:lm-theta-timeonly} we have that, with $\theta = [\mu, \theta_K]$,
\begin{equation*}
\partial_\mu \ell^{(m)}( \theta)
= \sum_{t=1}^N  h \left(  \frac{   y_t^{(m)}   }{ 1- e^{-  h  \Lambda_t^{(m)}( \theta) }}   -1 \right),
\end{equation*}
and for fixed $\theta_K$, one can solve for the value of $\mu$ by the following 1D non-linear equation
\begin{equation}\label{eq:eqn-elim-mu-timeonly}
0 
= \sum_{m, \, y^{(m)} \in Y_B} \partial_\mu \ell^{(m)}( \theta)
=  \sum_{m, \, y^{(m)} \in Y_B}  \sum_{t=1}^N  h \left(  \frac{   y_t^{(m)}   }{ 1- e^{-  h  \Lambda_t^{(m)}( \theta) }}  -1 \right).
\end{equation}
The equation can be solved by {\it bisection search} due to the monotonicity of the function $1- e^{-  h  \Lambda_t^{(m)}( \theta) } $ with respect to $\mu$, recalling that $\Lambda_t^{(m)}( \theta)  = \mu + \langle \eta_t^{(m)}, \theta_K \rangle$. 

Strictly speaking, the first-order condition corresponds to averaging over all the $M$ training trajectories, and in \eqref{eq:eqn-elim-mu-timeonly} we use the batch average to approximate solving for $\mu$ given the current $\theta_K$ and update the baseline more diligently. 
In our algorithm, we update $\mu$ by a convex combination of the current $\mu$ with the solution $\tilde \mu$ from the batch $Y_B$ so as to implement an average over the batches. 
The initial value $\mu_0$ of $\mu$ is set to be the empirical frequency of having an event averaged over the time horizon, i.e., 
$\mu_0 = \sum_{m=1}^M n^{(m)}/(MNh)$ where $n^{(m)} = \sum_{t=1}^N y^{(m)}_t$ is the count of events on $ 1 \le t \le N$ of the $m$-th trajectory.
We observe the numerical convergence of this scheme to learn $\mu$ in practice, 
see Figure~\ref{fig:TULIK-nonstationary-kernel-time-only-mu-dynamics}.

\paragraph{Constraint and regularization of the kernel.}
Motivated by Assumption \ref{assump:c1-Lambda-on-data}(ii), 
we inforce the lower-boundedness of the model intensity $\Lambda^{(m)}_t \ge b$ for all $1 \le t \le N$  and all training trajectories, where the minimum intensity threshold $b$ is an algorithmic parameter. 
We introduce a (per-trajectory) barrier function $B^{(m)}$ to penalize the violations, defined as 
\begin{equation}\label{eq:def-Bm-barrier}
B^{(m)}(\theta) = \sum_{t=1}^N {\bf 1}_{  \{ \Lambda^{(m)}_t (\theta) < b \}} l_b( \Lambda^{(m)}_t (\theta) ),
\end{equation}
and is weighted by a penalty strength factor $\delta_b$. 
The barrier function $l_b: \R \to \R$ is differentiable, and then $\partial_{\theta_K}B^{(m)}(\theta) 
	= \sum_{t=1}^N {\bf 1}_{  \{ \Lambda^{(m)}_t (\theta) < b \}} l_b'( \Lambda^{(m)}_t (\theta) ) \eta^{(m)}_t$.
In experiments, we find that several choices of $l_b$ can work well, and in our experiments, we use
(i) log-barrier $l_b(x) = -b \log (x/b) $
or 
(ii) quadratic barrier $l_b(x) = \frac{1}{ 0.2   b}(x - b)^2 $.

The underlying continuity of the kernel function suggests the low-rankness of the kernel matrix, as has been explained in Section \ref{subsec:special-structure-psi}. 
To leverage such prior structural information, we can apply a low-rank truncation of the learned kernel matrix by the stochastic optimization using the submatrix trick in Section \ref{subsec:special-structure-psi}, see the end of Algorithm \ref{algorithm:train-time-only}. We only keep singular vectors for which the singular value is greater than  $\tau_{\rm SVD}$, and this threshold is an algorithmic parameter.

Another optional regularization of the kernel matrix is by regularizing the variation of the kernel matrix $\Psi$ across time. Specifically, writing $\theta_K$ into the form of matrix $\Psi$, one can adopt the smoothness penalty 
\begin{equation}\label{eq:def-Stheta-smoothness}
S(\theta_K) 
	= \frac{1}{2 h^2} \left( \sum_{i=-N'+1}^{N-1} \sum_{l=1}^{N'}  (\Psi_{i,l} - \Psi_{i+1,l})^2 
	+ \sum_{i=-N'+1}^{N} \sum_{l=1}^{N'-1} (\Psi_{i,l} - \Psi_{i,l+1})^2 \right)
\end{equation}
 weighted by a strength factor $\delta_s$. In our algorithm, we update the kernel using the gradient of $\partial_{\theta_K} S$ after each epoch. 
 The choices of the parameters $b$, $\delta_b$, $\delta_s$ 
 together with the optimization hyperparameters 
 depend on the problem and will be detailed in our experiments (Appendix \ref{apdx:additional-exp}).

\subsection{On-network case}\label{subsec:algo-on-network}

The methodology in Section \ref{subsec:algo-time-only} can be extended to the on-network model, implementing the VI and GD inference as has been introduced in Section \ref{subsec:kernel-recovery-network}. 
The setup of batch-based stochastic optimization, the constraint enforcement, and the regularization of the kernel matrices are similar, and a main difference lies in the learning of baseline vector $\boldsymbol \mu$, which we again utilize the first-order condition.
This is summarized in Algorithm \ref{apdx:algorithm:train-graph} in the Appendix, and we give technical details of the extensions below.

Similar to \eqref{eq:eqn-elim-mu-timeonly}, we propose to solve for $\boldsymbol \mu$ given kernel $\theta_K$ by requiring the derivative of the log-likelihood with respect to each $\mu(u)$ to vanish. That is, with $\theta = [\boldsymbol \mu, \theta_K]$ and again averaged on a batch of trajectories $Y_B$, 
\begin{align*}
0 & = \sum_{m, \, \mathbf{y}^{(m)} \in Y_B} \frac{\partial \ell^{(m)}[ \theta ]}{\partial \mu(u)} \nonumber \\
& =  \sum_{m, \, \mathbf{y}^{(m)} \in Y_B}  
	\sum_{t=1}^N  \left( 
	h \left(  \frac{  \bar y_t^{(m)}   }{ 1- e^{-  h  \bar \Lambda_t^{(m)}[ \theta] }}   -1 \right)
	- \frac{\bar y^{(m)}_t}{ \bar \Lambda^{(m)}_t [ \theta] }
	+ \frac{y^{(m)}_t(u)}{\Lambda^{(m)}_t[ \theta] (u)}
	\right),
\quad \forall u \in \calV,
\end{align*}
and we recall that  $\Lambda_t^{(m)} [\theta] (u) = \mu(u) + \langle \eta^{(m)}_{t, u}, \theta_K \rangle$.
These $V$ equations are non-linear because all three terms inside the summation involve $\boldsymbol \mu$ through
$\bar \Lambda_t^{(m)}$ and $\Lambda_t^{(m)}$. 
To simplify the algorithm, we propose a Gauss-Seidel-type iteration: given the current $\boldsymbol \mu$ and $\theta_K$ fixed, 
we solve for \textcolor{red}{$\tilde \mu(u)$} for each $u \in \calV$ by the equation
\begin{equation}\label{eq:eqn-elim-mu-graph}
0 = \sum_{m, \, \mathbf{y}^{(m)} \in Y_B}  
	\sum_{t=1}^N  \left( 
	h \left(  \frac{  \bar y_t^{(m)}   }{ 1- e^{-  h  \bar \Lambda_t^{(m)}[ \boldsymbol \mu, \theta_K] }}   -1 \right)
	- \frac{\bar y^{(m)}_t}{ \bar \Lambda^{(m)}_t [ \boldsymbol \mu, \theta_K ] }
	+ \frac{y^{(m)}_t(u)}{ \textcolor{red}{\tilde \mu(u)} + \langle \eta^{(m)}_{t, u}, \theta_K \rangle }
	\right).
\end{equation}
These $V$ equations are then decoupled, each one is a one-dimensional non-linear equation and bisection search can be applied due to the monotonicity with respect to $\tilde \mu(u)$. 
After solving the vector $\{\tilde \mu (u) \}_{u \in \calV}$ on a batch, we do a convex combination with the current $\boldsymbol \mu$, same as before. 
The vector $\mu$ is initialized by
$\mu_0(u) = \sum_{m=1}^M n^{(m)}(u)/(MNh)$ where $n^{(m)}(u) = \sum_{t=1}^N y^{(m)}_t(u)$ is the count of events happening at location $u$ along the $m$-th trajectory.
We empirically observe the convergence of our iteration scheme to learn $\boldsymbol \mu$ as shown in Figure~\ref{fig:TULIK-on-network-mu-dynamics}.

We also extend the constraints and regularizations on the time-only kernel to the spatial-temporal kernel on a network.
Specifically, the penalty to enforce lower boundedness of the model intensities applies to each node, i.e.
\begin{equation}\label{eq:def-Bm-barrier-network}
B^{(m)}[\theta] = \sum_{t=1}^N \sum_{u \in \calV} {\bf 1}_{  \{ \Lambda^{(m)}_t [\theta](u) < b \}} l_b( \Lambda^{(m)}_t [\theta](u) ).
\end{equation}
The optional low-rank truncation is applied to the tall concatenated $\Psi$ matrix across all pairs of nodes, as has been explained in Section \ref{subsec:kernel-recovery-network}.
The optional smoothness regularization is by applying the penalty $S$ in \eqref{eq:def-Stheta-smoothness} to the time-only kernel on each pair of nodes. 

Depending on the problem, e.g., the size of the problem, namely $N$, $N'$, and the graph size $V$, 
as well as the inference approach (GD or VI),
we may need different parameters $b$, $\delta_b$, $\delta_s$ 
as well as optimization hyperparameters like batch size and learning rate schedule.
The choices will be detailed in the experiments
(Appendix \ref{apdx:additional-exp}).

\section{Experiments}\label{sec:experiments}

In this section, we compare the proposed model with various baselines on simulated and real data.

\subsection{Simulated event data: Time-only}\label{subsec:time-only-example}

\paragraph{Dataset.} For time-only event data, we set a time horizon $[-\tau_{\rm max}, T]$ with $\tau_{\rm max}=4$ and $T = 16$. 
The model adopts a latent influence kernel function $k( t', t)$ defined as
\begin{equation}\label{eq:time-only-kernel}
   k(t',t) = \sum_{j=1}^{13} 0.3(2)^{-j}(\cos(2+1.3\pi(j+1)((t'-9)/15))+0.6)e^{-8((t-t') j)^2/25},
    \quad t > t', 
\end{equation}
which, after specifying $N$ and $N'$, will induce a kernel matrix $K$ in \eqref{eq:def-Kij-discrete-time}.
In this example, the influence kernel matrix $K$ will have both positive and negative entries, indicating that the history may elicit both excitatory and inhibitory effects between events.
Additionally, the influence kernel is time-varying, indicating that these effects are non-stationary over time.

 We use two types of discrete time grids: 
 (a)  $N'=8,N=32$,
 and 
 (b) $N'=80, N=320$. 
 The corresponding kernel matrices are visualized in Figure~\ref{fig:true-small-kernel-time-only} and Figure~\ref{fig:true-large-kernel-time-only} respectively
 in the form of $\Psi$ defined as in \eqref{eq:def-Psi-matrix}.  
 We use $\mu = 0.2$,
 and the event data are simulated by the law of a Bernoulli process as introduced in Section \ref{subsec:bernoulli-time-only}. 
 The training and testing sets consist of $16,000$ and $500$  trajectories, respectively.

\begin{figure}[b!]
\centering 
\begin{subfigure}[h]{0.3\linewidth}
\includegraphics[width=\linewidth]{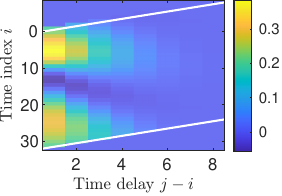}
\caption{true kernel}
\label{fig:true-small-kernel-time-only}
\end{subfigure}
\hspace{+0pt}
\begin{subfigure}[h]{0.3\linewidth}
\includegraphics[width=\linewidth]{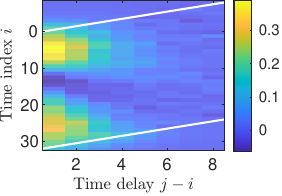}
\caption{{\texttt{TULIK-VI}}}
\label{fig:TULIK-VI-small-kernel-time-only}
\end{subfigure}
\hspace{+0pt}
\begin{subfigure}[h]{0.3\linewidth}
\includegraphics[width=\linewidth]{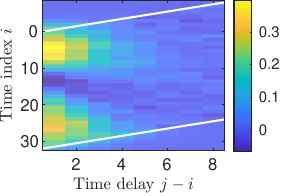}
\caption{{\texttt{TULIK-GD}}}
\label{fig:TULIK-GD-small-kernel-time-only}
\end{subfigure}
\vspace{-5pt}
\caption{Recovered kernels on simulated time-only event data with kernel function defined in \eqref{eq:time-only-kernel}, $N'=8$ and $N=32$. 
(a) The true kernel. (b) The kernel recovered by \texttt{TULIK-VI}. (c) The kernel recovered by \texttt{TULIK-GD}.
The white lines indicate the estimable parameters.
}
\label{fig:TULIK-small-kernel-time-only}
\end{figure}

\begin{figure}[b!]
\centering 
\begin{subfigure}[h]{0.80\linewidth}
\includegraphics[width=\linewidth]{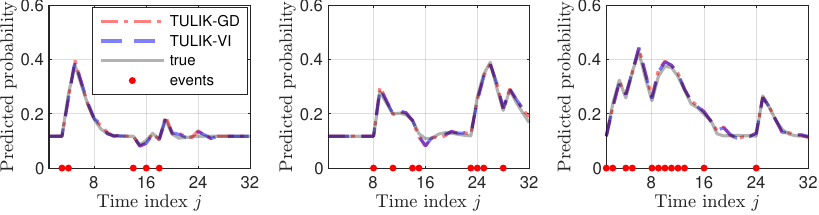}
\end{subfigure}
\vspace{-5pt}
\caption{
The probability predictions on three randomly selected testing sequences.
Data is the simulated time-only event data with kernel function defined in \eqref{eq:time-only-kernel}, $N'=8$ and $N=32$.
}
\label{fig:prob-pred-time-only-smallkernel}
\end{figure}

\begin{table}[th!]
    \caption{
    Relative Errors (REs) of learned $\mu$, 
    kernel matrix, 
    and probability predictions,
    on simulated time-only event data with the kernel function defined in \eqref{eq:time-only-kernel}, $N'=8$ and $N=32$.
    We report means of REs over 10 replicas, and standard deviations are in parentheses.  
    For probability predictions, the best performance is in boldface.  
    Note that $\mu$ is a scalar in the time-only example and its $\ell_1$, $\ell_2$, and $\ell_\infty$ errors are the same. 
    }
    \label{tab:time-only-smallkernel}
    \vspace{-5pt}
    \begin{subtable}[h]{1\textwidth}
        \centering
        \resizebox{1\columnwidth}{!}{%
	\begin{tabular}{c cc cc ccccc}
        \toprule
        \multicolumn{10}{c}{time-only event data with discrete time grids $N'=8$ and $N=32$}\\ 
        \cmidrule(lr){2-10}
        RE $\times 10^{-2}$ & \multicolumn{2}{c}{$\mu$} & \multicolumn{2}{c}{kernel} & \multicolumn{5}{c}{prediction}\\
        \cmidrule(lr){2-3}  \cmidrule(lr){4-5}  \cmidrule(lr){6-10} 
        & {\texttt{TULIK-VI}} & {\texttt{TULIK-GD}} & {\texttt{TULIK-VI}} & {\texttt{TULIK-GD}} & {\texttt{TULIK-VI}} & {\texttt{TULIK-GD}} 
        & {\texttt{GLM-L}} & {\texttt{GLM-S}} & {\texttt{HP-E}} \\
			 \cmidrule(lr){2-3} \cmidrule(lr){4-5}  \cmidrule(lr){6-10} 
\multirow{2}{*} {$\ell_1$} & 0.40 &0.47 &16.37 &18.39 &{\bf 2.85} &3.20 &3.93 &6.13 &37.08\\

 & (0.38) & (0.41) & (0.82) & (0.86) & (0.11) & (0.15) & (0.19) & (0.19) & (0.65)\\[3pt]

\multirow{2}{*} {$\ell_2$} & -- & -- &12.07 &13.46 &{\bf 3.96} &4.43 &5.25 &7.99 &53.06\\

 &  &  & (0.70) & (0.66) & (0.16) & (0.20) & (0.21) & (0.14) & (1.14)\\[3pt]

\multirow{2}{*} {$\ell_\infty$} & -- & -- &11.26 &12.35 &{\bf 6.02} &6.73 &8.09 &12.85 &92.83\\

 &  &  & (1.40) & (1.32) & (0.35) & (0.38) & (0.41) & (0.30) & (2.49)
   			\\\bottomrule
		\end{tabular}
  }%
    \end{subtable}
\end{table}

 \begin{figure}[t!]
\centering 
\begin{subfigure}[h]{0.3\linewidth}
\includegraphics[width=\linewidth]{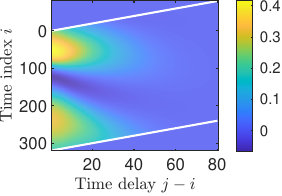}
\caption{true kernel}
\label{fig:true-large-kernel-time-only}
\end{subfigure}
\hspace{+0pt}
\begin{subfigure}[h]{0.3\linewidth}
\includegraphics[width=\linewidth]{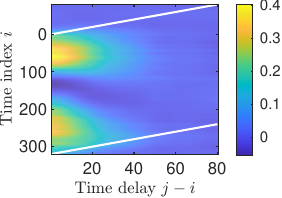}
\caption{{\texttt{TULIK-VI}}}
\label{fig:TULIK-VI-large-kernel-time-only}
\end{subfigure}
\hspace{+0pt}
\begin{subfigure}[h]{0.3\linewidth}
\includegraphics[width=\linewidth]{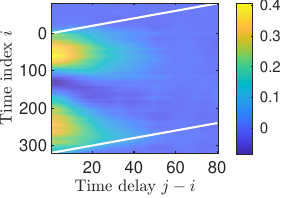}
\caption{{\texttt{TULIK-GD}}}
\label{fig:TULIK-GD-large-kernel-time-only}
\end{subfigure}
\vspace{-5pt}
\caption{
Recovered kernels on simulated time-only event data with kernel function defined in \eqref{eq:time-only-kernel},  $N'=80$ and $N=320$.
Same plots as in Figure~\ref{fig:TULIK-small-kernel-time-only}. 
The recovered kernel is after a low-rank truncation. }
\label{fig:TULIK-large-kernel-time-only}
\end{figure}

\begin{figure}[t!]
\centering 
\begin{subfigure}[h]{0.8\linewidth}
\includegraphics[width=\linewidth]{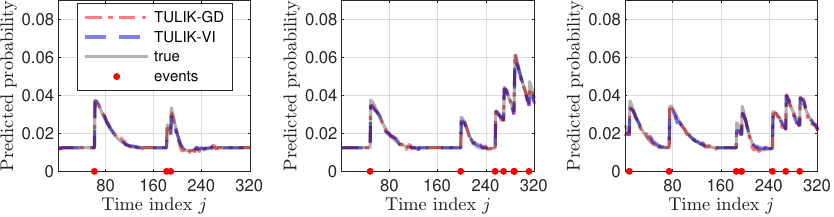}
\end{subfigure}
\vspace{-5pt}
\caption{
Same plots as in Figure \ref{fig:prob-pred-time-only-smallkernel}.
Data is the simulated time-only event data with kernel function defined in \eqref{eq:time-only-kernel}, $N'=80$ and $N=320$.}
\label{fig:prob-pred-time-only-largekernel}
\end{figure}

\begin{table}[t!]
    \caption{Relative Errors (REs) of learned $\mu$, 
    kernel matrix, 
    and probability predictions,
    on simulated time-only event data with the kernel function defined in \eqref{eq:time-only-kernel}, $N'=80$ and $N=320$.
    The format are the same as in Table~\ref{tab:time-only-smallkernel}.}
    \label{tab:time-only-largekernel}
    \vspace{-5pt}
    \begin{subtable}[h]{1\textwidth}
        \centering
        \resizebox{1\columnwidth}{!}{%
	\begin{tabular}{c cc cc ccccc}
        \toprule
        \multicolumn{10}{c}{time-only event data with discrete time grids $N'=80$ and $N=320$}\\ 
        \cmidrule(lr){2-10}
        RE $\times 10^{-2}$ & \multicolumn{2}{c}{$\mu$} & \multicolumn{2}{c}{Kernel} & \multicolumn{5}{c}{Prediction}\\
        \cmidrule(lr){2-3}  \cmidrule(lr){4-5}  \cmidrule(lr){6-10} 
        & {\texttt{TULIK-VI}} & {\texttt{TULIK-GD}} & {\texttt{TULIK-VI}} & {\texttt{TULIK-GD}} & {\texttt{TULIK-VI}} & {\texttt{TULIK-GD}} 
        & {\texttt{GLM-L}} & {\texttt{GLM-S}} & {\texttt{HP-E}} \\
			 \cmidrule(lr){2-3} \cmidrule(lr){4-5}  \cmidrule(lr){6-10} 
			\multirow{2}{*} {$\ell_1$} &1.46 &0.72 &14.49 &18.23 &{\bf3.34} &3.89 &37.11 &27.96 &72.88\\

 & (0.80) & (0.37) & (0.53) & (0.67) & (0.23) & (0.19) & (0.59) & (0.49) & (1.74)\\[3pt]

\multirow{2}{*} {$\ell_2$} & -- & -- &11.58 &13.39 &{\bf4.50} &5.31 &51.65 &39.79 &149.82\\

 &  &  & (0.50) & (0.52) & (0.19) & (0.25) & (0.48) & (0.79) & (5.44)\\[3pt]

\multirow{2}{*} {$\ell_\infty$} & -- & -- &17.50 &17.04 &{\bf8.64} &9.44 &98.82 &101.33 &393.29\\

 &  &  & (1.52) & (2.27) & (0.59) & (0.50) & (0.91) & (2.57) & (17.25)
			\\\bottomrule
		\end{tabular}
  }%
    \end{subtable}
\end{table}

\paragraph{Method.}

In the experimental results, we refer to our model as the Time Uncertain Latent Influence Kernel (\texttt{TULIK}) point processes, and the model estimated using VI and GD (following Algorithm \ref{algorithm:train-time-only} as introduced in Section~\ref{subsec:algo-time-only}) are referred to as \texttt{TULIK-VI} and \texttt{TULIK-GD}, respectively.
The choice of training hyperparameters is detailed in Appendix \ref{apdx:choice-hyperparameters-time-only}.
The training dynamic on one training set, including the evolution of the training (negative) log-likelihood and the convergence of the learned $\mu$ are illustrated in Figure \ref{fig:TULIK-nonstationary-kernel-time-only-dynamics}.

For comparison, we consider three alternative baselines: 
(1) Generalized Linear Model (GLM) \cite{juditsky2020convex}
with linear link function (\texttt{GLM-L}), 
where the link function is an identity function with range $[0,1]$; 
(2) GLM with Sigmoid link function (\texttt{GLM-S}); 
(3) Continuous-time Hawkes Process with Exponential kernel (\texttt{HP-E}) \cite{hawkes1971point,hawkes1971spectra,hawkes1974cluster}, see more in Section~\ref{subsec:prelim-hawkes}. 
The implementation details of the alternative baselines can be found in Appendix \ref{apdx:detail-baselines-time-only}.

\paragraph{Results.}
For data with $N'=8$ and $N=32$, both \texttt{TULIK-VI} and \texttt{TULIK-GD} can learn the true model parameters ($\mu$ and kernel matrix) well, resulting in good prediction performance, as shown quantitatively in Table \ref{tab:time-only-smallkernel}.
The proposed TULIK models outperform the other three baselines, where  \texttt{TULIK-VI} gains slightly better performance than  \texttt{TULIK-GD}. 
The recovered kernel matrices are plotted in Figure~\ref{fig:TULIK-small-kernel-time-only}, showing a good agreement with the true kernel. 
The predicted probability is also close to the ground truth, as reflected by the almost overlapping curves in Figure~\ref{fig:prob-pred-time-only-smallkernel}.

For the data with $N'=80$ and $N=320$, which results in much higher dimensional trainable parameters, the proposed TULIK model can still achieve reasonable recovery of kernel matrices and $\mu$'s and accurate probability predictions.
Figure~\ref{fig:TULIK-large-kernel-time-only} shows that both \texttt{TULIK-VI} and \texttt{TULIK-GD} produce correct shapes and magnitudes of kernel matrices, where we applied low-rank truncation to the learned kernel matrix. 
Figure~\ref{fig:prob-pred-time-only-largekernel} visualizes the probability predictions in good agreement with the ground truth. 
From the quantitative errors in Table~\ref{tab:time-only-largekernel}, our proposed model again achieves the best relative prediction errors, and VI slightly outperforms GD.   

\paragraph{Example of a stationary process.}
As has been introduced in Section \ref{subsec:special-structure-psi}, the proposed TULIK model can be applied to stationary processes with time uncertainty, where the kernel is time-invariant.
This is a special case of the more general time-varying kernel. 
We provide experiments on such a simplified case in Appendix \ref{adpx:time-only-stationary}. The results show that the proposed approach successfully recovers the model parameters and outperforms the alternative baselines.

\subsection{Simulated event data on networks}\label{subsec:on-network-example}

\begin{figure}[t]
\centering 
\begin{subfigure}[h]{0.25\linewidth}
\includegraphics[width=\linewidth]{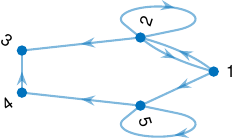}
\end{subfigure}
\caption{The visualization for the network structure of the on-network event data in Section~\ref{subsec:on-network-example}. } 
\label{fig:network-structure}
\end{figure}

\begin{figure}[b!]
\centering 
\begin{subfigure}[h]{0.3\linewidth}
\includegraphics[width=\linewidth]{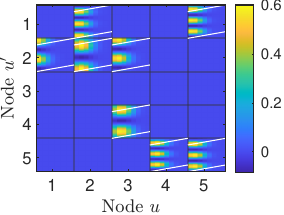}
\caption{true kernel}
\label{fig:true-kernel-on-network}
\end{subfigure}
\hspace{+6pt}
\begin{subfigure}[h]{0.305\linewidth}
\includegraphics[width=\linewidth]{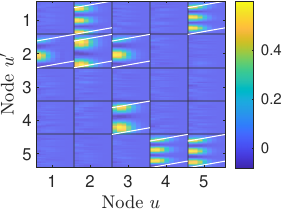}
\caption{{\texttt{TULIK-VI}}}
\label{fig:TULIK-VI-kernel-on-network}
\end{subfigure}
\hspace{+6pt}
\begin{subfigure}[h]{0.3\linewidth}
\includegraphics[width=\linewidth]{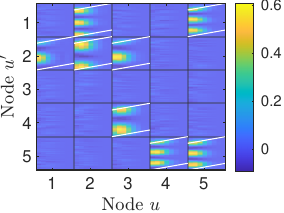}
\caption{{\texttt{TULIK-GD}}}
\label{fig:TULIK-GD-kernel-on-network}
\end{subfigure}
\caption{
Recovered kernel tensors on simulated on-network event data with kernel function defined in \eqref{eq:on-network-kernel}, $V=5$, $N'=8$ and $N=32$. 
(a) The true kernel. (b) The kernel recovered by \texttt{TULIK-VI}. (c) The kernel recovered by \texttt{TULIK-GD}.
The learned kernels are after low-rank truncation. 
The $(u',u)$-th block in the plot shows the kernel matrix $[\Psi_{i,l}(u',u)]_{i,l}$, for $u',u = 1,\cdots, 5$.
The white slashes box out the trainable parameters of the kernel matrices corresponding to the directed edges.} 
\label{fig:TULIK-kernel-on-network}
\end{figure}

\begin{figure}[b!]
\centering 
\begin{subfigure}[h]{0.80\linewidth}
\includegraphics[width=\linewidth]{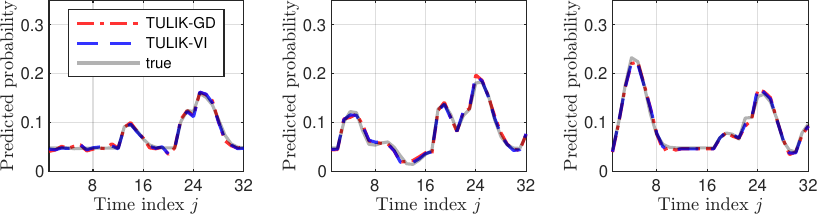}
\end{subfigure}
\caption{
The probability predictions on 3 randomly selected testing sequences on one node (node 3 in the graph in Figure \ref{fig:network-structure}).
Data is the simulated on-network event data with kernel function defined in \eqref{eq:on-network-kernel}, $V=5$, $N'=8$, and $N=32$.} 
\label{fig:prob-pred-on-network}
\end{figure}

\paragraph{Dataset.} For the on-network event data, we set a time horizon $[-\tau_{\rm max}, T]$ with $\tau_{\rm max}=0.8$ and $T = 3.2$. The network has $V=5$ nodes, in which 8 pairs of successors and predecessors are randomly sampled from $\calV$ to form 8 directed edges. The network structure is visualized in Figure~\ref{fig:network-structure}. For a directed edge $u^\prime \rightarrow u$, the marginal kernel (temporal pattern) is defined as 
\begin{equation}
\label{eq:on-network-kernel}
k(t',t;u',u) = 0.35\left(\cos\left(\omega_{u^\prime, u}(t+2)\right)+0.75\right)e^{-20\left(t-t'-h_{u^\prime,u}\right)^2},
\end{equation}
where $\omega_{u^\prime, u}$ independently follows $\Unif(2,6)$ and $h_{u^\prime, u}$ is independently drawn from $\Unif(0,0.2)$. Given $N'=8$ and $N=32$, we can induce the four-way kernel tensor $\Psi$ as in \eqref{eq:def-Psi-tensor} and visualize it in Figure~\ref{fig:true-kernel-on-network}. The non-zero block matrix on $u'$-th row and $u$-th column corresponds to a directed edge $u'\to u$. The zero block matrices indicate disconnected pairs of nodes. 

Given the kernel and $\mu(u)$ independently sampled from $\Unif(0.25,0.35)$ for $u\in\calV$, we can generate on-network Bernoulli process data as in Section \ref{subsec:bernoulli-graph}. The training data and testing data contain 40000 and 500 trajectories, respectively.

\paragraph{Method.} 
We train the proposed model using Algorithm \ref{apdx:algorithm:train-graph} (in Appendix) as introduced in Section~\ref{subsec:algo-on-network}. The choice of training hyperparameters is discussed in Appendix \ref{apdx:choice-hyperparameters-on-network}.
The convergence of the training (negative) log-likelihood and the learned $\mu$ vector is illustrated in Figure \ref{fig:TULIK-on-network-dynamics}. 
We compare the proposed method with the baselines, including \texttt{GLM-L}, \texttt{GLM-S}, and \texttt{HP-E}. The implementation details of the alternative baselines in the on-network setting are given in Appendix \ref{apdx:detail-baselines-on-network}.

\begin{table}[t!]
    \caption{
    Relative Errors (REs) of learned $\mu$, 
    kernel tensor, 
    and probability predictions,
    on simulated on-network event data with the kernel function defined in \eqref{eq:on-network-kernel}, $V=5$, $N'=8$, and $N=32$.
    We report means of REs over 10 replicas, and standard deviations are in parentheses.  
    For probability predictions, the best performance is in boldface.  
    }
    \label{tab:on-graph-10-node}
    \vspace{-5pt}
    \begin{subtable}[h]{1\textwidth}
        \centering
        \resizebox{1\columnwidth}{!}{%
\begin{tabular}{c cc cc ccccc}
        \toprule
        \multicolumn{10}{c}{on-network event data with discrete time grids $N'=8$ and $N=32$}\\ 
        \cmidrule(lr){2-10}
        RE $\times 10^{-2}$ & \multicolumn{2}{c}{$\mu$} & \multicolumn{2}{c}{Kernel} & \multicolumn{5}{c}{Prediction}\\
        \cmidrule(lr){2-3}  \cmidrule(lr){4-5}  \cmidrule(lr){6-10} 
        & {\texttt{TULIK-VI}} & {\texttt{TULIK-GD}} & {\texttt{TULIK-VI}} & {\texttt{TULIK-GD}} & {\texttt{TULIK-VI}} & {\texttt{TULIK-GD}} 
        & {GLM-L} & {GLM-S} & {HP-E} \\
			 \cmidrule(lr){2-3} \cmidrule(lr){4-5}  \cmidrule(lr){6-10} 
			\multirow{2}{*} {$\ell_1$} &7.28 &6.04 &19.11 &21.07 &{\bf6.36} &7.47 &9.29 &10.29 &31.55\\

 & (0.52) & (0.57) & (0.57) & (0.44) & (0.09) & (0.09) & (0.11) & (0.11) & (0.45)\\[3pt]

\multirow{2}{*} {$\ell_2$} &7.40 &6.19 &13.49 &14.70 &{\bf7.42} &8.76 &11.32 &13.94 &45.97\\

 & (0.56) & (0.60) & (0.33) & (0.24) & (0.09) & (0.09) & (0.12) & (0.14) & (0.41)\\[4pt]

\multirow{2}{*} {$\ell_\infty$} &8.15 &7.09 &17.22 &17.22 &{\bf9.62} &11.09 &15.01 &25.97 &72.12\\

 & (0.82) & (0.85) & (2.00) & (2.16) & (0.07) & (0.19) & (0.15) & (0.53) & (0.33)
			\\\bottomrule
		\end{tabular}
  }%
    \end{subtable}
\end{table}

\paragraph{Results.} 
Both \texttt{TULIK-VI} and \texttt{TULIK-GD} can learn the vector $\mu$ and the kernel tensor, leading to accurate prediction as quantitatively revealed in Table \ref{tab:on-graph-10-node}.
Similar to the time-only examples, the proposed TULIK models yield smaller prediction errors than the other baselines, where \texttt{TULIK-VI} is slightly better than \texttt{TULIK-GD}. 
After optimization loops, we apply low-rank truncation to the matrix $\mathbf{\Psi}$, and the learned kernel tensors are visualized in Figure~\ref{fig:TULIK-kernel-on-network}.
It can be seen that the kernel is well recovered, and in particular, the learned kernel reveals the correct causal network:  when $(u',u)$ is a missing (directed) edge in the network (Figure \ref{fig:network-structure}), the corresponding block in the learned kernel also takes almost zero values.  
The predicted probability is also consistent with the ground truth, as shown in Figure~\ref{fig:prob-pred-on-network} (the prediction is successful on all the nodes, where the plots on one node are shown).

\section{Real data}

\subsection{Sepsis-Associated-Derangements (SADs) data}
\label{subsec:more-on-sepsis}

\begin{table}[b!]
    \centering
        \caption{Selected 13 medical indices in SADs data}
    \label{tab:medical indices}
    \resizebox{0.55\columnwidth}{!}{%
\begin{tabular}{ccc}
\toprule Node & Abbreviation & Full Name \\
\midrule[0.3pt] 1 & RI Kernel & Renal Injury \\
 2 & EI & Electrolyte Imbalance \\
 3 & OCD & Oxygen Carrying Dysfunction\\
 4 & Shock & Shock \\
 5 & DCO & Diminished Cardiac Output \\
 6 & Cl & Coagulopathy\\
 7 & Chl & Cholestatsis \\
 8 & HI & Hepatocellular Inury\\
 9 & OD(Lab) & Oxygenation Dysfunction (Lab)\\
 10 & Inf(Lab) & Inflammation (Lab)\\
 11 & OD(vs) & Oxygenation Dysfunction (vital sign)\\
 12 & Inf(vs) & Inflammation (vital sign)\\
 13 & Sepsis & Sepsis \\
\bottomrule
\end{tabular}
}%

\end{table}

\paragraph{Dataset.} We consider the data to be publicly available through the 2019 PhysioNet Challenge \cite{physionetChallenge}, which consists of hospitalization records of Intensive Care Unit (ICU) patients. Each patient record includes three main types of variables: vital signs, laboratory (lab) measurements, and demographic information. The goal is to use this data to predict the onset of sepsis—a condition that arises when the body's response to infection causes tissue damage, organ failure, or death, and can quickly become life-threatening \cite{singer2016third}.

We use a processed version of this dataset, referred to as SADs \cite{wei2021inferringb,wei2023granger,10366499}, which was created by grouping all raw observations—including vital signs and lab measurements—into 13 medical indices using a standard procedure \cite{ABIM_ref_range}. These indices represent common physiological derangements associated with sepsis and are treated as nodes in our graph. Table~\ref{tab:medical indices} lists the abbreviations and names of the indices. Vital signs are recorded hourly, introducing a natural one-hour time uncertainty in the data. Lab test results are typically returned within 24 hours, leading to greater time uncertainty; that is, the test result indicates the patient's condition within 24 hours around the test time. For simplicity, we assume the test time uncertainty is also 1 hour.  Appendix~\ref{apdx:data-processing-sepsis} contains details on data pre-processing.

Given the 13 medical indices, the processed SAD data can be viewed as a point process on a network with $ V=13$ nodes, with a time uncertainty of 1 hour. Each patient’s data trajectory is a Boolean process on this network, which we refer to as a ``signal'' that takes values 0 or 1 at each node ($y_t(u) \in {0, 1}$). We consider the 24 hours following ICU admission and set the influence time lag, $\tau_{\max}$, to 4 hours. As a result, the time horizon is $[-\tau_{\rm max}, T] = [-4, 20]$, covering a total of 24 hours. The time grid has $N = 20$ and $N' = 4$, with unit uncertainty of one hour. 

The node ``Sepsis'' marks the onset of sepsis for each patient, according to the Sepsis-3 definition. For time points more than 6 hours before sepsis onset, the observation on the sepsis node is always 0; after sepsis onset, the observation remains positive. To support early prediction of sepsis onset, the positive labels are shifted 6 hours earlier, meaning that prediction within 6 6-hour period before and after the initial onset of sepsis is meaningful, consistent with the definition in Physionet Challenge \cite{physionetChallenge}.

In the experiments below, the processed dataset comprises 1,141 patient records, all of whom were diagnosed with sepsis, with varying onset times. We randomly select 1,070 trajectories for training and 71 trajectories for testing.

\begin{table}[t!]
     \caption{Sepsis prediction: True Positive Rate (TPR), True Negative Rate (TNR), and Balanced Accuracy (BA) from \texttt{TULIK-VI}. The standard deviations of TPR, TNR, and BA over 10 training replicas are in parentheses. The highest metrics are bolded. The training data and testing data consist of 1070 and 71 patient trajectories, respectively. }
    \label{tab:sepsis-metric}
        \begin{subtable}[h]{1\textwidth}
        \centering
        \begin{scriptsize}
		\begin{tabular}{cccccc}\toprule
     & {\texttt{TULIK-VI}} & {\texttt{TULIK-VI} (Stationary)} & {\texttt{GLM-L}} & {\texttt{GLM-S}} & {\texttt{HP-E}}\\
     \cmidrule(lr){2-6}
\multirow{2}{*} {TPR} & {\bf 0.6842} &0.4632 &0.5789 &0.6316 &0.0000\\

 & (0.0000) & (0.0211) & (0.0000) & (0.0000) & (0.0000)\\[3pt]

\multirow{2}{*} {TNR} & {\bf 0.6904} &0.4519 &0.5769 &0.6731 &0.0192\\

 & (0.0058) & (0.0197) & (0.0000) & (0.0000) & (0.0000)\\[3pt]

\multirow{2}{*} {BA} & {\bf 0.6873} &0.4575 &0.5779 &0.6523 &0.0096\\

 & (0.0029) & (0.0203) & (0.0000) & (0.0000) & (0.0000)\\
            
   \bottomrule
		\end{tabular}
\end{scriptsize}
    \end{subtable}
    \end{table}

\begin{figure}[t!]
\centering 
\begin{subfigure}[h]{0.24\linewidth}
\includegraphics[width=\linewidth]{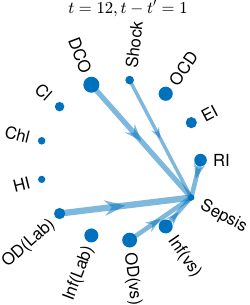}
\end{subfigure}
\begin{subfigure}[h]{0.24\linewidth}
\includegraphics[width=\linewidth]{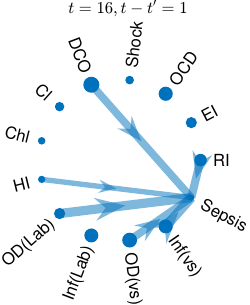}
\end{subfigure}
\begin{subfigure}[h]{0.24\linewidth}
\includegraphics[width=\linewidth]{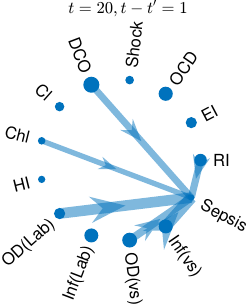}
\end{subfigure}

\caption{Temporally evolving influence networks. For clarity of visualization, only the top five strongest edges are shown. The time lag $t - t'$ is fixed at 1, corresponding to one-hour effects. The future time $t$ progresses from 12 to 20 in increments of 4. Node radii are proportional to their baseline intensities. Blue arrows indicate excitation effects, with arrow widths scaled to reflect their magnitudes. Marker sizes and line widths are consistent across all plots.
}
\label{fig:causal-networks}
\end{figure}

\paragraph{Method.} 
We use \texttt{TULIK-VI} on the SADs dataset, as it exhibits more stable training behavior than \texttt{TULIK-GD} in our case. As in Section~\ref{subsec:on-network-example}, we compare the proposed model with baseline methods, including \texttt{GLM-L}, \texttt{GLM-S}, and \texttt{HP-E}. Experimental details are provided in Appendices~\ref{apdx:choice-hyperparameters-sepsis} and \ref{apdx:eva-sepsis}. 

\paragraph{Results.} 
The performance of all methods is reported in Table~\ref{tab:sepsis-metric}. We evaluate True Positive Rate (TPR), True Negative Rate (TNR), and Balanced Accuracy (BA), on the testing data. These evaluation metrics are standard for sepsis prediction \cite{physionetChallenge}. As shown in the table, \texttt{TULIK-VI} consistently outperforms the baseline methods. 

In particular, \texttt{TULIK}, which allows for non-stationary modeling, significantly outperforms its stationary counterpart, indicating that the sepsis data likely contain substantial non-stationarity, which our model effectively captures. \texttt{TULIK}, with its modeling of time uncertainty and flexible triggering pattern modeling, also outperforms methods that do not account for time uncertainty, such as \texttt{GLM-L} and \texttt{GLM-S}, which use standard link functions, and continuous time \texttt{HP-E}, which employs a parametric exponential decaying triggering kernel. These results highlight the importance of incorporating both time uncertainty and nonparametric triggering structures in modeling sepsis dynamics.

To further interpret the results, we visualize the temporally evolving causal networks extracted from the kernel tensor learned by \texttt{TULIK-VI} in Figure~\ref{fig:causal-networks}, which indicates the inferred Granger causal relationships between medical indices. Node radii are proportional to the estimated baseline intensities $\mu(u)$; arrows represent excitatory effects, with arrow widths indicating the magnitude of influence. We present the five strongest edges at each time snapshot of the network.

The learned network structure aligns with medical knowledge. For example, the node ``DCO'' (Diminished Cardiac Output) consistently exhibits a strong excitatory influence on sepsis onset, an observation well-established in sepsis research \cite{10366499}. Additionally, as time progresses, more medical indices show a strong influence on sepsis onset, reflecting the development of multiple organ dysfunction syndrome in severe cases. Figure~\ref{fig:sepsis-non-linear-effects} shows that the effects of biological processes are nonlinear, suggesting that medical indices cannot deteriorate indefinitely, nor can their influence on sepsis onset grow without bound.

\begin{figure}[t!]
\centering 
\begin{subfigure}[h]{0.24\linewidth}
\includegraphics[width=\linewidth]{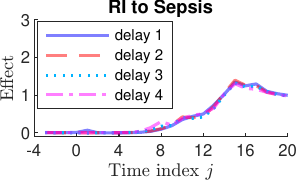}
\end{subfigure}
\begin{subfigure}[h]{0.24\linewidth}
\includegraphics[width=\linewidth]{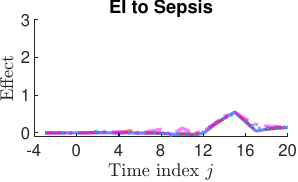}
\end{subfigure}
\begin{subfigure}[h]{0.24\linewidth}
\includegraphics[width=\linewidth]{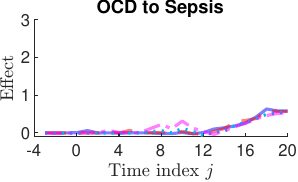}
\end{subfigure}
\begin{subfigure}[h]{0.24\linewidth}
\includegraphics[width=\linewidth]{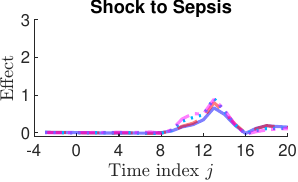}
\end{subfigure}

\vspace{+5pt}
\begin{subfigure}[h]{0.24\linewidth}
\includegraphics[width=\linewidth]{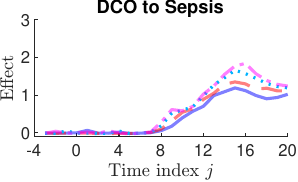}
\end{subfigure}
\begin{subfigure}[h]{0.24\linewidth}
\includegraphics[width=\linewidth]{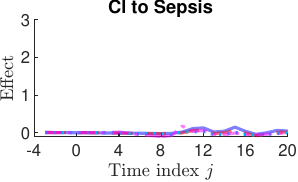}
\end{subfigure}
\begin{subfigure}[h]{0.24\linewidth}
\includegraphics[width=\linewidth]{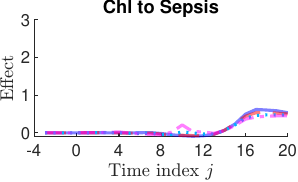}
\end{subfigure}
\begin{subfigure}[h]{0.24\linewidth}
\includegraphics[width=\linewidth]{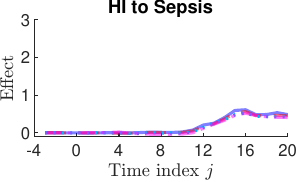}
\end{subfigure}

\vspace{+5pt}
\begin{subfigure}[h]{0.24\linewidth}
\includegraphics[width=\linewidth]{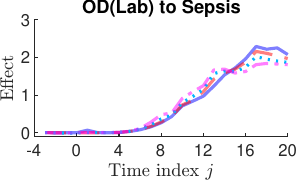}
\end{subfigure}
\begin{subfigure}[h]{0.24\linewidth}
\includegraphics[width=\linewidth]{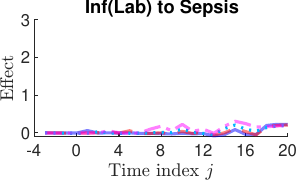}
\end{subfigure}
\begin{subfigure}[h]{0.24\linewidth}
\includegraphics[width=\linewidth]{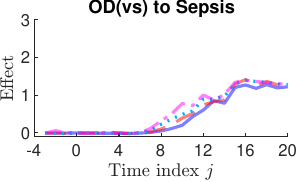}
\end{subfigure}
\begin{subfigure}[h]{0.24\linewidth}
\includegraphics[width=\linewidth]{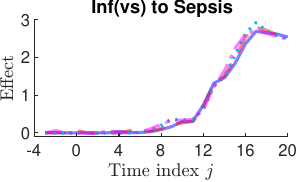}
\end{subfigure}

\vspace{-5pt}
\caption{The effects from all medical indices to the sepsis onset over time across different delays.}
\label{fig:sepsis-non-linear-effects}
\end{figure}

\subsection{Burglary crime data}\label{subsec:burglary-crime}

\paragraph{Dataset.}  To validate our method on another dataset, we consider a proprietary dataset provided by the Atlanta Police Department (APD), including over 213,000 burglary events recorded from the beginning of 2013 to the end of 2019, with timestamps and geographic coordinates (longitude and latitude). We focus on the downtown area of Atlanta, defined by the longitude range $[-84.400, -84.375]$ and latitude range $[33.745, 33.770]$. Figure~\ref{fig:atl-downtown-burglary} displays the spatial distribution of burglary events in this area.
Following the setup in \cite{juditsky2020convex}, we uniformly divide the downtown area into 16 sub-regions, resulting in $V = 16$ nodes for the point process model. We treat each 24-hour period, from 12:00 AM to 11:59 PM, as one trajectory. The time horizon is defined as $(-\tau_{\rm max}, T] = (-8, 16]$, incorporating a one-hour uncertainty. The time grid includes $N = 16$ and $N' = 8$ intervals, each representing one hour.
The processed dataset consists of 2,548 daily trajectories. We randomly select 2,250 for training and use the remaining 298 for testing.

\begin{figure}[b!]
\centering 
\begin{subfigure}[h]{0.6\linewidth}
\includegraphics[width=\linewidth]{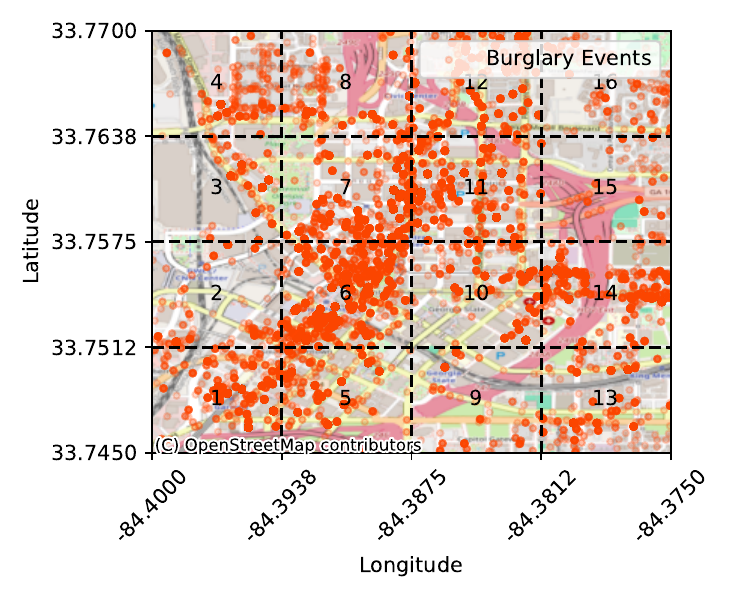}
\end{subfigure}

\caption{The distribution of the burglary crime events from 2013 to 2019 in the downtown Atlanta area. 
}
\label{fig:atl-downtown-burglary}
\end{figure}

\paragraph{Method.} We compared \texttt{TULIK-VI} with \texttt{GLM-L}, \texttt{GLM-S}, and \texttt{HP-E}. Details of the hyperparameter tuning procedure are provided in Appendix~\ref{apdx:choice-hyperparameters-burglary}.

\begin{table}[t!]
     \caption{Crime data modeling in downtown Atlanta: TPR, TNR, and BAs from TULIK-VI. The time horizon is 1 day with an uncertainty of 1 hour. The standard deviations of TPR, TNR, and BA over 10 training replicas are in parentheses. The largest metrics are bolded. }
    \label{tab:burglary-metric}
        \begin{subtable}[h]{1\textwidth}
        \centering
        \begin{scriptsize}
		\begin{tabular}{cccccc}\toprule
     & {TULIK-VI} & {TULIK-VI (Stationary)} & {GLM-L} & {GLM-S} & {HP}\\
     \cmidrule(lr){2-6}
\multirow{2}{*} {TPR} &{\bf 0.6705} &0.6686 &0.6186 &0.6594 &0.4533\\

 & (0.0008) & (0.0009) & (0.0004) & (0.0002) & (0.0031)\\[3pt]

\multirow{2}{*} {TNR} &{\bf0.6710} &0.6683 &0.6187 &0.6532 &0.4733\\

 & (0.0010) & (0.0008) & (0.0004) & (0.0000) & (0.0098)\\[3pt]

\multirow{2}{*} {BA} &{\bf0.6707} &0.6684 &0.6187 &0.6563 &0.4633\\

 & (0.0002) & (0.0004) & (0.0003) & (0.0001) & (0.0036)\\[3pt]
   \bottomrule
		\end{tabular}
\end{scriptsize}
    \end{subtable}
    \end{table}



\paragraph{Results.} 
Table~\ref{tab:burglary-metric} presents the prediction performance of all methods. We report TPR, TNR, and BA on the testing data. For the burglary crime data, TPR and TNR are computed by averaging across all nodes for $t \in (0, 16]$, treating event occurrences as positives and non-events as negatives. \texttt{TULIK-VI} achieves noticeably higher prediction accuracy compared to the baseline methods.

Again, we observe that \texttt{TULIK}, allowing for non-stationary modeling, outperforms its stationary counterpart. However, in this example, the performance gain is less pronounced than in the sepsis prediction case, suggesting that the burglary data may exhibit less non-stationarity. Nevertheless, \texttt{TULIK} still outperforms the alternative methods—\texttt{GLM-L}, \texttt{GLM-S}, and \texttt{HP-E}, the latter of which employs a parametric triggering kernel, highlighting the value of incorporating time uncertainty and flexible triggering structures.



\section{Discussion}

The work can be extended in several directions. 
First, although we focused on the unit-time uncertainty case in our algorithm development, our method can be extended to handle the arbitrary event uncertainty case in Section \ref{subsec:discrete-general-uncertainty}, applicable to both the time-only and on-network scenarios.
Second, it would be useful to allow the baseline intensity $\mu$ to vary with time $t$, which can be handled within our discrete-time framework by introducing additional learnable parameters. 
Meanwhile, it would be interesting to relax the requirement that each uncertainty window contains exactly one event (Assumption \ref{assump:disjoint_Delta}), which would yield a Poisson process rather than a Bernoulli process and necessitate the development of a new time-uncertainty model. 
Finally, one can apply the approach to real-world data applications, and practical challenges may call for further development of computational techniques, e.g., the use of a barrier to enforce positive intensity and the tuning of optimization hyperparameters. 

\section*{Acknowledgments}
The work was partially supported by NSF DMS-2134037.
YX was also partially supported by 
NSF DMS-2220495.
XC was also partially supported by 
NSF DMS-2237842
and Simons Foundation (grant ID: 814643).


\appendix

\setcounter{figure}{0} \renewcommand{\thefigure}{A.\arabic{figure}}
\setcounter{table}{0} \renewcommand{\thetable}{A.\arabic{table}}

\setcounter{equation}{0} \renewcommand{\theequation}{A.\arabic{equation}}
\setcounter{remark}{0} \renewcommand{\theremark}{A.\arabic{remark}}

\clearpage

\section{Proofs in Section \ref{sec:model-time-only}}\label{app:proofs-sec:3}

\begin{proof}[Proof of Lemma \ref{lemma:logL-time-only}]
By Assumption (A1), we know that $\Delta_k^r > \Delta_k^l \ge \Delta_{k-1}^r$, and then we have
\begin{align}
\Pr[ \Delta_k | \tilde{\calF}_{k-1} ] 
&= S_k (\Delta_k^l) - S_k( \Delta_k^r ) \nonumber  \\
&= \exp \left\{ - \int_{\Delta_{k-1}^r}^{\Delta_k^l} h_k(s) ds \right\}
 - \exp \left\{ - \int_{\Delta_{k-1}^r}^{\Delta_k^r} h_k(s) ds \right\} \nonumber  \\
& =  \exp\left\{ - \int_{\Delta_{k-1}^r}^{\Delta_k^r} h_k(s) ds \right\} 
  \left( \exp\left\{ \int_{\Delta_{k}^l}^{\Delta_k^r} h_k(s) ds\right\}  - 1 \right). \label{eq:Pr-Deltak-time-only}
\end{align}

This leads to the expression of the likelihood $\calL$ as follows, denoting $\Delta_{0}^r = 0$,
\begin{align}
\calL 
& = \Pr[ t_1\in\Delta_1, \ldots, t_n \in \Delta_n, t_{n+1} > T ] \\
& = \Pr[ \Delta_1  ]   
    \Pr[ \Delta_2 | \tilde{\calF}_{1}  ]
    \cdots 
    \Pr[ \Delta_n | \tilde{\calF}_{n-1}  ]
    S_{n+1}(T) \nonumber \\
& =
 \left(\prod_{k=1}^n
  \exp\{ - \int_{\Delta_{k-1}^r}^{\Delta_k^r} h_k(s) ds \} 
  \left( \exp\{ \int_{\Delta_{k}^l}^{\Delta_{k}^r} h_k(s) ds\}  - 1 \right) \right)     
  \exp\{ - \int_{\Delta_{n}^r}^T h_{n+1}(s) ds \}.
    \label{eq:likelihood-1}
\end{align}
Recall the definition of $\lambda(t)$ as in \eqref{eq:def-lambdat-time-only},
then \eqref{eq:likelihood-1} can be written by integrals of $\lambda(t)$, which leads to the expression of  $\calL$ as
\begin{align*}
\calL
& =  \left(\prod_{k=1}^n
  \exp\{ - \int_{\Delta_{k-1}^r}^{\Delta_k^r}  \lambda(u) du \} 
  \left( \exp\{ \int_{\Delta_{k}^l}^{\Delta_{k}^r}  \lambda(u) du\}  - 1 \right) \right)     
  \exp\{ - \int_{\Delta_{n}^r}^T \lambda(u) du \} \nonumber \\
& = \left(  \prod_{k=1}^n
  \left( \exp\{ \int_{\Delta_{k}^l}^{\Delta_{k}^r}  \lambda(u) du\}  - 1 \right) \right) 
  \exp\{ - \int_{0}^{T}  \lambda(u) du \},
\end{align*}
Taking the log proves \eqref{eq:l-lambda-cont-time-uncert}.
\end{proof}

\begin{proof}[Proof of Lemma \ref{lemma:Lambdaj-Kij-lemma}]
By \eqref{eq:lambda-phi-cont-time-uncert} and under Assumption (A2), we have
\begin{equation*}
    \lambda(t) = \mu + \sum_{k, t_k^r < j } \frac{1}{ h (t_k^r - t_k^l+1)} \sum_{i=t_k^l}^{t_k^r} \int_{I_{i}} k( t', t) dt',
    \quad \forall t \in I_j.
\end{equation*} 
By the definition of $\Lambda_j$ in \eqref{eq:def-Lambdaj-discrete-time}, 
\begin{equation*}
    \Lambda_j 
    = \mu + \sum_{k, \, t_k^r < j } \frac{1}{ (t_k^r - t_k^l+1)}  \sum_{i=t_k^l}^{t_k^r}  \frac{1}{ h^2 }\int_{I_{i}} \int_{I_j} k( t', t) dt' dt.
\end{equation*}
The expression \eqref{eq:Lambdaj-Kij-lemma} then follows by the definition of $K_{i,j}$ in \eqref{eq:def-Kij-discrete-time}.
\end{proof}

\begin{proof}[Proof of Lemma \ref{lemma:time-only-bernoulli-glm}]
Under (A3),  the definition of $\Lambda_t$ in  \eqref{eq:model-base-uncert-bernoulli} shows that $\Lambda_t \in \sigma\{ y_i, i\le t-1 \}$.
The expression of log-likelihood as in \eqref{eq:model-base-uncert-bernoulli} gives that 
\[
\Pr[ y_1, \cdots, y_N] =  e^{\ell} 
=  \prod_{j=1}^N  
\left( 
(1-y_j) e^{- h  \Lambda_j } 
+ y_j   ( 1 - e^{ - h  \Lambda_j} ) 
\right).
\]
This gives that, for any $t \in [ N]$,
\[
\Pr[ y_1, \cdots, y_t]
 =   \prod_{j=1, \, y_j = 1}^t  ( 1 - e^{ - h  \Lambda_j} ) 
   \prod_{j=1, \, y_j = 0}^t e^{- h  \Lambda_j }. 
\]
Thus, by definition, $\Pr [y_t = 1 |  y_i, i\le t-1 ] = 1 - e^{ - h  \Lambda_t}$. 
\end{proof}

\section{Proofs in Section \ref{sec:inference-GD-VI}}\label{app:proofs-sec:4}

\subsection{Proofs in Section \ref{subsec:VI-recovery-timeonly}}
\begin{proof}[Proof of Lemma \ref{lemma:c1-Lambda-on-data-implication}]
Consider when $t= -N'+1$, since there is no history, $\Lambda_{-N'+1}^* = \mu$, and then  \eqref{eq:assump-Lamdbat-star} implies that $ b \le \mu \le B$.

The claim that any realization of the binary sequence $\{y_t, -N' <  t \le N\}$ happens w.p. $> 0$ is equivalent to that 
any realization of the binary sequence
$\{y_j, -N'+1 \le j \le t\}$, for any $ -N' < t \le N$, happens w.p. $> 0$.
We prove this claim by induction. 
First, for $t= -N'+1$, $\Pr [ y_{-N'+1}  =0] = e^{-h \mu} \in (e^{-hB}, e^{-hb})$. Thus both $\Pr [ y_{-N'+1}  =0]$ and $\Pr [ y_{-N'+1}  =1] >0$. 
Suppose the claim holds for $t-1$, 
\begin{align*}
\Pr [ (y_{ -N'+1 }, \cdots, y_{t}) ]
& = \Pr [ (y_{ -N'+1 }, \cdots, y_{t-1}) ] \Pr [ y_t | y_i, \,  i \le t-1 ] \\
& = \begin{cases}
\Pr [ (y_{ -N'+1 }, \cdots, y_{t-1}) ] e^{-h \Lambda_t^*}, & y_t =0 \\
\Pr [ (y_{ -N'+1 }, \cdots, y_{t-1}) ] (1-e^{-h \Lambda_t^*}), &  y_t =1,
\end{cases}
\end{align*}
where the second inequality is by Lemma \ref{lemma:time-only-bernoulli-glm}.
By induction hypothesis, $\Pr [ (y_{ -N'+1 }, \cdots, y_{t-1}) ]  > 0 $, 
then \eqref{eq:assump-Lamdbat-star} holds a.s. means that 
$ \Lambda_t^*$ as a function evaluated at the outcome $(y_{ -N'+1 }, \cdots, y_{t-1})$ needs to satisfy that $ b \le \Lambda_t^* \le B$. This means that both $e^{-h \Lambda_t^*}$ and $1-e^{-h \Lambda_t^*} > 0$. 
As a result, $\Pr [ (y_{ -N'+1 }, \cdots, y_{t}) ] > 0$ in both cases of $y_t = 0$ or 1. 
This proves the claim for $t$ based on the induction hypothesis on $t-1$.
We thus have proved that the claim holds for any $t$.

Finally, because  the conditional intensity $ \Lambda_t^*$ is a random variable determined by the outcome $(y_{ -N'+1 }, \cdots, y_{t-1})$, any any realization of this binary sequence up to time $t-1$ happens w.p. $>0$, then $ b \le \Lambda_t^* \le B$ a.s. means that for any realization of $(y_{ -N'+1 }, \cdots, y_{t-1})$ this inequality holds. 
\end{proof}

\begin{proof}[Proof of Lemma \ref{lemma:eta-eta-positivity}]
For any $v \in \R^{N N'}$, we also write $v$ in the pattern of matrix $K$ indexed by $v_{i,t}$, 
$ -N' < i \le N$, $1 \le t \le N$, cf. \eqref{eq:def-theta-timeonly-Omega}, then
\begin{equation}\label{eq:eta-eta-pos-proof-1}
\left\langle v,  \left[ \E  \sum_{t= 1}^N \eta_t \otimes \eta_t  \right] v \right\rangle
= \E \sum_{t=1}^N \langle \eta_t, v\rangle^2
= \E \sum_{t=1}^N \left(  \sum_{ i = t- N' }^{t-1}  y_i  v_{i, t} \right)^2,
\end{equation}
where the second equality is by  \eqref{eq:def-innerprod-etat-K}.
For each $t$, we lower bound $\E \left(  \sum_{ i = t- N' }^{t-1}  y_i  v_{i, t} \right)^2$ by enumerating over the realizations of
$(y_{t-N'}, \cdots, y_{t-1} )$, a binary sequence of length $N'$, that equals 
$(0, \cdots, 1, \cdots, 0)$, namely only $y_i = 1$ for some $i$ and the other $y_j$'s are all zero.
By the derivation in the proof of Lemma \ref{lemma:time-only-bernoulli-glm}, we know that
\begin{align*}
 \Pr[ (y_{t-N'}, \cdots, y_{t-1} ) = (0, \cdots, \underset{i-th}{1}, \cdots, 0) | y_{i}, i< t-N']
=   ( 1 - e^{ - h  \Lambda_i^*} ) 
   \prod_{j=t-N', \, j \neq i}^{t-1} e^{- h  \Lambda_j^* },
\end{align*}
and then, applying Lemma \ref{lemma:c1-Lambda-on-data-implication}, we have
\begin{equation}\label{eq:Pr-ysubseq-lower-bound-proof1}
 \Pr[ (y_{t-N'}, \cdots, y_{t-1} ) = (0, \cdots, \underset{i-th}{1}, \cdots, 0)] 
 \ge e^{-hB(N'-1)} (1- e^{-hb}) = \rho.
\end{equation}
As a result, for each $t$,
\begin{align*}
\E \left(  \sum_{ i = t- N' }^{t-1}  y_i  v_{i, t} \right)^2
& \ge \sum_{i = t-N'}^{t-1} v_{i,t}^2 \Pr[  (y_{t-N'}, \cdots, y_{t-1} ) = (0, \cdots, \underset{i-th}{1}, \cdots, 0) ] \\
& \ge \rho \sum_{i = t-N'}^{t-1} v_{i,t}^2,
\end{align*}
and then 
\[
 \E \sum_{t=1}^N \left(  \sum_{ i = t- N' }^{t-1}  y_i  v_{i, t} \right)^2 
 \ge   \rho \sum_{t=1}^N  \sum_{i = t-N'}^{t-1} v_{i,t}^2 
 = \rho  \|v\|_2^2.
 \]
 This proves the lemma by putting back to \eqref{eq:eta-eta-pos-proof-1}.
\end{proof}

\begin{proof}[Proof of Lemma \ref{lemma:G-kappa-monotone}]
By definition of $G(z)$ in \eqref{eq:def-Gz-VI}, for any $z, \tilde z \in \Theta_K$,
\begin{align} 
&  \langle G( z) - G(\tilde z), z - \tilde z \rangle
 = \E \sum_{t=1}^N \left( \phi( h \Lambda_t( z)) - \phi( h \Lambda_t(\tilde z)) \right) \langle \eta_t , z - \tilde z \rangle 
	 \nonumber \\
&~~~
=  \E \sum_{t=1}^N \left( \phi( h \Lambda_t( z)) - \phi( h \Lambda_t(\tilde z)) \right)  (\Lambda_t(z) - \Lambda_t(\tilde z)),
\label{eq:proof-G-mono-1} 
\end{align}
where in the second equality we used that $\Lambda_t ( z ) = \mu + \langle \eta_t , z \rangle$.
Observe an elementary fact about the function $\phi(x) = 1-e^{-x}$: for any $\beta >0 $,
\begin{equation}
(\phi(x) - \phi(x'))(x-x') \ge e^{-\beta} (x-x')^2, \quad \forall x, x' \in (0, \beta).
\end{equation}
The claim can be verified by the Mean Value Theorem.
Using this fact and that $\Lambda_t(z), \Lambda_t(\tilde z)  \in [b,B]$ a.s. under Assumption \ref{assump:c1-Lambda-on-data}(ii), \eqref{eq:proof-G-mono-1} continues as
\[
 \langle G(z) - G(\tilde z), z - \tilde z \rangle
 \ge \E \sum_{t=1}^N e^{-hB} h \langle  \eta_t , z - \tilde z \rangle^2
 = e^{-hB} h \left\langle  z - \tilde z,  \left[ \E \sum_{t=1}^N   \eta_t  \otimes \eta_t \right] (z - \tilde z) \right\rangle.
 \]
By Lemma \ref{lemma:eta-eta-positivity} which holds under Assumption \ref{assump:c1-Lambda-on-data}(i), we have
\[
 \langle G(z) - G(\tilde z), z - \tilde z \rangle
 \ge e^{-hB} h \rho \| z - \tilde z \|^2.
\]
It remains to show that $ e^{-hB} h \rho  \ge \kappa$ to finish the proof of the lemma.  By definition, 
\[
e^{-hB} h \rho 
=  h  (1-e^{-hb}) e^{-hB N'}
\ge h^2 b e^{-hb} e^{-hB N'} = \kappa,
\]
by that $hN' = \tau_{\rm max}$, 
where in the inequality we used the fact that $ x e^{-x} \le 1-e^{-x}$, $\forall x \ge 0$.
\end{proof}

\vspace{5pt}
\noindent
$\bullet$ \textbf{ Proof of Theorem \ref{thm:VI-convergence} }
\vspace{5pt}

The proof of Theorem \ref{thm:VI-convergence}  uses the following technical lemma on the boundedness of the (per-trajectory) VI vector field.
\begin{lemma}\label{lemma:boundedness-VI}
Under Assumption \ref{assump:c1-Lambda-on-data}, 
let  $C$ be as defined in Theorem \ref{thm:VI-convergence},
\begin{itemize}
\item[(i)] Define $G_1(z) :=  \E \sum_{t=1}^N \phi( h \Lambda_t(z)) \eta_t $, 
then $\forall z \in \Theta_K$, $\| G_1(z) \|_2 \le C$.
\item[(ii)] $\forall z \in \Theta_K$, 
$\E \| \hat G^{(m)}( z )  \|_2^2 =  \E \| \sum_{t=1}^N \left(\phi( h \Lambda_t (z)) - y_t \right) \eta_t \|_2^2 
\le 4 C^2$.
\end{itemize}
\end{lemma}
\begin{proof}[Proof of Lemma \ref{lemma:boundedness-VI}]
We first prove (i), which provides useful intermediate estimates to prove (ii).
The norm $\| \cdot \|$ always denotes 2-norm.

\vspace{5pt}
Proof of (i): $\| G_1(z) \|^2 = \| \E \sum_{t=1}^N \phi( h \Lambda_t(z)) \eta_t \|^2 \le \E \| \sum_{t=1}^N \phi( h \Lambda_t(z)) \eta_t \|^2 $ by Jensen's inequality.
By definition of $\eta_t$ as a vector in $\R^{NN'}$, we know that 
$\langle \eta_t, \eta_{t'} \rangle =0$ when $t \neq t'$.
Meanwhile,
$\|\eta_t\|^2 = \sum_{i = t-N'}^{t-1} y_i$.
Thus,
\[
\E \| \sum_{t=1}^N \phi( h \Lambda_t(z)) \eta_t \|^2 
= \E \sum_{t=1}^N \phi( h \Lambda_t(z))^2  \sum_{i = t-N'}^{t-1} y_i
\le ( (hB) \wedge 1)^2 \E \sum_{t=1}^N  \sum_{i = t-N'}^{t-1} y_i, 
\]
where the 2nd inequality is by that, under  Assumption \ref{assump:c1-Lambda-on-data}(ii),
\[
0 \le \phi( h \Lambda_t(z)) = 1-e^{- h \Lambda_t(z)} 
\le  1- e^{-hB} \le  (hB) \wedge 1 \quad a.s.,
\]
using the elementary relation that $1-e^{-x} \le x$, $\forall x \ge 0$.
We claim that 
\begin{equation}\label{eq:claim-bound-Eyi-proof(i)}
\E \sum_{t=1}^N  \sum_{i = t-N'}^{t-1} y_i
\le N N'   ( (hB) \wedge 1).
\end{equation}
If true, then putting together, we have
\begin{align*}
\| G_1(z) \|^2
& \le  \E \| \sum_{t=1}^N \phi( h \Lambda_t(z)) \eta_t \|^2
\le ( (hB) \wedge 1)^2 \E \sum_{t=1}^N  \sum_{i = t-N'}^{t-1} y_i \\
& \le  ( (hB) \wedge 1)^3 N N'
\le ( (hB) \wedge 1)^2 (N+N')^2,
\end{align*}
namely $\| G_1(z) \| \le  ( (hB) \wedge 1) (N + N')$ which equals $C$ since 
$ hB(N + N')  = B (T+ \tau_{\rm max})$.

\vspace{5pt}
Proof of Claim \eqref{eq:claim-bound-Eyi-proof(i)}:
Observe that for any $t$, we always have
\begin{equation}\label{eq:Eyt-upperbound-as-proof1}
\E [ y_t | y_i, i <t ] = 1-e^{ - h \Lambda_t^*}
\le   1- e^{- hB}  \le (hB) \wedge 1
\end{equation} 
 by Lemma \ref{lemma:time-only-bernoulli-glm} and Lemma \ref{lemma:c1-Lambda-on-data-implication},
 and then
\begin{align*}
\E y_t 
 = \E \left(  \E[ y_t | y_{i}, i < t]  \right)
 \le  (hB) \wedge 1.
\end{align*}
By linearity of expectation, we have 
\[
\E \sum_{t=1}^N  \sum_{i = t-N'}^{t-1} y_i
\le \sum_{t=1}^N  \sum_{i = t-N'}^{t-1}    (hB) \wedge 1,
\]
which proves the claim.

\vspace{5pt}
Proof of (ii): 
Again using that $\langle \eta_t, \eta_{t'} \rangle =0$ when $t \neq t'$, 
and $\|\eta_t\|^2 = \sum_{i = t-N'}^{t-1} y_i$,
we have
\begin{align}
&  \E \| \sum_{t=1}^N \left(\phi( h \Lambda_t (z)) - y_t \right) \eta_t \|^2 
=   \E  \sum_{t=1}^N  \left(\phi( h \Lambda_t (z)) - y_t  \right)^2  \sum_{i = t-N'}^{t-1} y_i  \nonumber \\
& ~~~
\le  \E  \sum_{t=1}^N  2 \left( \phi( h \Lambda_t (z))^2 + y_t  \right)  \sum_{i = t-N'}^{t-1} y_i 
	\quad \text{(by that $y_t^2 = y_t$)}  \nonumber \\
& ~~~
= 2 \left(  \E  \sum_{t=1}^N  \phi( h \Lambda_t (z))^2  \sum_{i = t-N'}^{t-1} y_i 
  + 	       \E  \sum_{t=1}^N   y_t  \sum_{i = t-N'}^{t-1} y_i   \right).
   \label{eq:to-bound-proof(ii)-1}
\end{align}
In the proof of (i), we have already shown that $\E  \sum_{t=1}^N  \phi( h \Lambda_t (z))^2  \sum_{i = t-N'}^{t-1} y_i  \le C^2$.
To handle the second term in the bracket,
by \eqref{eq:Eyt-upperbound-as-proof1} we have
\begin{align*}
 \E  \sum_{t=1}^N   y_t  \sum_{i = t-N'}^{t-1} y_i 
&  =  \E  \sum_{t=1}^N   \E [ y_t | y_i, i <t ]   \sum_{i = t-N'}^{t-1} y_i   \\
& \le ( (hB) \wedge 1)  \E  \sum_{t=1}^N   \sum_{i = t-N'}^{t-1} y_i  \\
& \le ( (hB) \wedge 1)^2 N N', 
	\quad \text{(by Claim \eqref{eq:claim-bound-Eyi-proof(i)})} 
\end{align*}
which again is upper bounded by $C^2$ as has been shown in the proof of (i). 
Putting together, we have that \eqref{eq:to-bound-proof(ii)-1} is upper bounded by $4 C^2$.
\end{proof}

\begin{proof}[Proof of Theorem \ref{thm:VI-convergence}]
Due to the projection operator in \eqref{eq:VI-SGD-scheme}, we have that $z_k \in \Theta_K$ for all $k$.
Let $\| \cdot \|$ always denote the 2-norm.
By definition of the projection operator \eqref{eq:def-proj-operator}, the convexity of $\Theta_K$ and that $z^* \in \Theta_K$,
we have
\[
\| z_k - z^* \|^2 
\le \| z_{k-1} - \gamma_k \hat G^{(k)}(z_{k-1}) - z^* \|^2,
\]
and by expanding the r.h.s. we have
\[
\| z_k - z^* \|^2 
\le \| z_{k-1} -z^*\|^2 
	- 2  \gamma_k \langle \hat G^{(k)}(z_{k-1}), z_{k-1} - z^* \rangle
	+ \gamma_k^2 \|\hat G^{(k)}(z_{k-1}) \|^2.
\]
Taking expectation on both sides, and defining $d_k := \frac{1}{2} \E \| z_k - z^*\|^2$, we have
\begin{equation}\label{eq:dk-VI-converge-proof-1}
d_k \le d_{k-1} - \gamma_k \E \langle \hat G^{(k)}(z_{k-1}), z_{k-1} - z^* \rangle 
	+ \frac{1}{2}\gamma_k^2 \E \|\hat G^{(k)}(z_{k-1}) \|^2.
\end{equation}
The third term on the r.h.s. is upper bounded by $2 \gamma_k^2 C^2$ by Lemma \ref{lemma:boundedness-VI}(ii).

To handle the second term on the r.h.s. of \eqref{eq:dk-VI-converge-proof-1}, first note that $z_{k-1}$ is in the $\sigma$-algebra generated by up to ($k-1$)-th trajectories, 
and then
\begin{align*}
\E \langle \hat G^{(k)}(z_{k-1}), z_{k-1} - z^* \rangle
& =  \E \langle  \E [ \hat G^{(k)}(z_{k-1}) | y^{(m)}, m=1, \cdots k-1 ] , z_{k-1} - z^* \rangle \\
& =  \E  \langle G(z_{k-1}), z_{k-1} - z^* \rangle,
\end{align*}
where we used that $\E [ \hat G^{(k)}(z_{k-1}) | y^{(m)}, m=1, \cdots k-1 ]  = G(z_{k-1})$ because $y^{(k)}$ is independent from the previous $k-1$ trajectories. 
Next, recall \eqref{eq:G(zstar)=0}, we have
\begin{align*}
 \langle G(z_{k-1}), z_{k-1} - z^* \rangle 
 =  \langle G(z_{k-1}) - G(z^*), z_{k-1} - z^* \rangle  
 \ge \kappa \| z_{k-1} - z^*\|^2,
\end{align*}
where the last inequality is by Lemma \ref{lemma:G-kappa-monotone} and that both $z_{k-1}$ and $z^*$ are in $\Theta_K$. Taking expectations on both sides gives
\[
\E \langle G(z_{k-1}), z_{k-1} - z^* \rangle  \ge 2 \kappa d_{k-1}.
\]

Putting back to \eqref{eq:dk-VI-converge-proof-1}, we have
\begin{equation}\label{eq:bound-dk-VI-converge-proof1}
d_k \le ( 1- 2 \kappa \gamma_k )  d_{k-1}
	+ 2 \gamma_k^2 C^2,
\end{equation}
where, by our choice of $\gamma_k$, $\kappa \gamma_k = 1/(k+1) \le 1/2$ for all $k\ge 1$, 
and thus, the factor $1- 2 \kappa \gamma_k$ is non-negative. 
We now claim that  for all $k$,
\begin{equation}\label{eq:claim-VI-converge-proof}
d_k \le \frac{2 C^2}{ \kappa^2} \frac{1}{k+1},
\end{equation}
which, if true, proves the theorem. The proof of \eqref{eq:claim-VI-converge-proof} follows the argument to prove the same claim in the proof of \cite[Proposition 2]{juditsky2019signal}, and we include the details for completeness.

We prove  \eqref{eq:claim-VI-converge-proof} by induction. 
When $k=0$,  we have 
$d_0 \le D^2/2$ where $D > 0$ is the diameter of the domain $\Theta_K$.
Suppose $D = \|z_+ - z_-\|$, where $z_+, z_- \in \Theta_K$.
We then have
\begin{align*}
\kappa D^2 
& = \kappa \| z_+ - z_-\|^2
\le \langle G(z_+) - G(z_-), z_+ - z_- \rangle  
	\quad \text{(by Lemma \ref{lemma:G-kappa-monotone})} \\
& = \langle \E \sum_{t=1}^N (\phi( h \Lambda_t(z_+ ) ) - \phi(h \Lambda_t( z_- )  )) \eta_t, z_+ - z_- \rangle \\
& =  \langle G_1(z_+) - G_1(z_-), z_+ - z_- \rangle  
	\quad \text{(by definition of $G_1$ in Lemma \ref{lemma:boundedness-VI}(i))} \\
& \le ( \| G_1(z_+)\| + \| G_1(z_-)\| )  \|  z_+ - z_-  \| \\
& \le 2C D. 
	\quad \text{(by Lemma \ref{lemma:boundedness-VI}(i))}
\end{align*}
This gives that $D \le 2C/\kappa$, and then we have 
\[
d_0 \le \frac{D^2}{2} \le  \frac{2 C^2}{ \kappa^2} =: S,
\]
 which proves  \eqref{eq:claim-VI-converge-proof}  when $k=0$.
 
 Suppose the claim holds up to $k-1$; we prove it also holds for $k$. 
 Because $k \ge 1$, we have $1- 2 \kappa \gamma_k \ge 0$,  
 then, by \eqref{eq:bound-dk-VI-converge-proof1},
 and  that $d_{k-1} \le S/k$ (the induction hypothesis),
\begin{align*}
d_k 
& \le ( 1- 2 \kappa \gamma_k )  \frac{S}{k}
	+ 2 \gamma_k^2 C^2 \\
& \le \frac{S}{k} (1- \frac{2}{k+1}) + \frac{S}{(k+1)^2} 
	\quad \text{(by definitions of $\gamma_k$ and $S$)}\\
& = \frac{S}{k+1}( \frac{k-1}{k} + \frac{1}{k+1} ) \le \frac{S}{k+1}.
\end{align*}
This proves the claim  \eqref{eq:claim-VI-converge-proof} by induction. 
\end{proof}

\subsection{Proofs in Section \ref{subsec:GD-shceme-timeonly}}
\begin{proof}[Proof of Proposition \ref{prop:GD-strongly-convex}]
Direct computation gives that
\begin{equation} \label{eq:expression-E-Lzz}
\partial^2_{z z} \E (-  L(z) ) =
\E \sum_{t=1}^N \frac{h^2}{ \phi( h  \Lambda_t( z) )^2 } e^{-h    \Lambda_t( z) } y_t \eta_t \otimes \eta_t.
\end{equation}
For any $z \in \Theta_K$, by Assumption \ref{assump:c1-Lambda-on-data}(ii),
$b \le  \Lambda_t( z) \le B$ a.s.,
and then $
\frac{e^{- h \Lambda_t (z)}}{(1- e^{- h \Lambda_t (z)})^2}  \ge \frac{e^{-hB }}{(1- e^{-hB })^2}$ a.s.
This gives that 
\[
\partial^2_{z z} \E (-  L(z) ) 
\succeq  h^2  \frac{e^{-hB }}{(1- e^{-hB })^2} \E \sum_{t=1}^N  y_t \eta_t \otimes \eta_t.
\]
Meanwhile, because $\eta_t \in \sigma\{ y_i, i<t\}$, 
\begin{align*}
\E \sum_{t=1}^N  y_t \eta_t \otimes \eta_t
& =  \E \sum_{t=1}^N  \E [y_t | y_i, i< t] \eta_t \otimes \eta_t \\
& =  \E \sum_{t=1}^N  (1-e^{- h \Lambda_t^*}) \eta_t \otimes \eta_t 
	\quad \text{(by Lemma \ref{lemma:time-only-bernoulli-glm})} \\
& \succeq (1- e^{-hb})   \E \sum_{t=1}^N 	 \eta_t \otimes \eta_t,
	\quad \text{(by Lemma \ref{lemma:c1-Lambda-on-data-implication})} \\
& \succeq (1- e^{-hb})    \rho I_{NN'}.
		\quad \text{(by Lemma \ref{lemma:eta-eta-positivity})}
\end{align*}
This means that $\E (-  L(z) ) $ is $\lambda$-strongly convex on $\Theta_K$ where
$\lambda =  h^2  \frac{e^{-hB }}{(1- e^{-hB })^2} (1- e^{-hb})    \rho$.
To prove the proposition, it suffices to show that $\lambda \ge \kappa'$.
 This can be verified by the definition of $\rho$ in Lemma \ref{lemma:eta-eta-positivity}
 and using that $ x e^{-x} \le 1-e^{-x} \le x$, $\forall x \ge 0$. 
\end{proof}

\subsection{Proofs in Section \ref{subsec:special-structure-psi}}
\begin{proof}[Proof of Lemma \ref{lemma:G-kappa-monotone-stationary}]
We claim that under Assumption \ref{assump:c1-Lambda-on-data}(i),
\begin{equation}\label{eq:claim-xi-xi-positivity-stationray}
\E  \sum_{t= 1}^N \xi_t \otimes \xi_t \succeq  \rho_s I_{N'},
\quad \rho_s := bT e^{-hb} e^{-hB (N'-1)}.
\end{equation}
If true, then using the same argument as in the proof of Lemma \ref{lemma:G-kappa-monotone} one can show that for any $\psi, \tilde \psi \in \Theta_\psi$,
\begin{align*}
\langle G_s(\psi) - G_s(\tilde \psi), \psi - \tilde \psi \rangle
& \ge e^{-hB} h  \left\langle  \psi - \tilde \psi,  \left[ \E \sum_{t=1}^N   \xi_t  \otimes \xi_t \right] (\psi - \tilde \psi) \right\rangle \\
& \ge e^{-hB} h  \rho_s \| \psi - \tilde \psi \|_2^2,
\end{align*}
which proves that $ G_s$ is $(e^{-hB} h  \rho_s)$-monotone on $\Theta_\psi$, and that is the same as declared in the lemma.

It remains to prove the claim \eqref{eq:claim-xi-xi-positivity-stationray} to prove the lemma. For any $v \in \R^{N'}$,
\[
\left\langle v,  \left[ \E  \sum_{t= 1}^N \xi_t \otimes \xi_t  \right] v \right\rangle
= \E \sum_{t=1}^N \langle \xi_t, v\rangle^2
= \E \sum_{t=1}^N \left(  \sum_{ i = t- N' }^{t-1}  y_i  v_{t-i} \right)^2.
\]
To lower bound $\E \left(  \sum_{ i = t- N' }^{t-1}  y_i  v_{t-i} \right)^2$ for each $t$,
again we have \eqref{eq:Pr-ysubseq-lower-bound-proof1} due to Lemma \ref{lemma:c1-Lambda-on-data-implication},
and then
\[
\E \left(  \sum_{ i = t- N' }^{t-1}  y_i  v_{t-i} \right)^2
\ge \rho \sum_{ i = t- N' }^{t-1} v_{t-l}^2
= \rho \| v\|_2^2,
\quad \rho = e^{-hB(N'-1)} (1- e^{-hb}) .
\]
Thus,
$\E \sum_{t=1}^N \left(  \sum_{ i = t- N' }^{t-1}  y_i  v_{t-i} \right)^2
\ge N \rho \| v\|_2^2 $.
This proves that $\E  \sum_{t= 1}^N \xi_t \otimes \xi_t \succeq \lambda I_{N'}$ where
\[
\lambda 
= N \rho 
= N e^{-hB(N'-1)} (1- e^{-hb}) 
\ge N e^{-hB(N'-1)} hb  e^{-hb} 
=  \rho_s,  
\]
where in the inequality we used 
 $ x e^{-x} \le 1-e^{-x} $, $\forall x \ge 0$. 
\end{proof}

\section{Proofs in Section \ref{sec:model-with-graph}}\label{app:proofs-sec:5}
\begin{proof}[Proof of Lemma \ref{lemma:logL-cont-time-uncert-on-graph}]
Using the expression \eqref{eq:Pr-tk-cont-time-uncert-on-graph}, we have that, denoting $\Delta_{0}^r = 0$,
\begin{align*}
\calL 
& = \Pr[ ( \Delta_1, u_1)  ]   
    \Pr[ ( \Delta_2, u_2) | \tilde{\calF}_{1}  ]
    \cdots 
    \Pr[ (\Delta_n, u_n) | \tilde{\calF}_{n-1}  ]
    S_{n+1}(T) \nonumber  \\
&=    \prod_{k=1}^n \left(  e^{ - \int_{\Delta_{k-1}^r}^{\Delta_{k}^r}  \bar{\lambda} (s) ds}  
	\int_{\Delta_k^l}^{\Delta_k^r} \lambda( t, u) e^{ \int_t^{\Delta_{k}^r}  \bar{\lambda} (s) ds} dt \right)
	e^{ - \int_{\Delta_{n}^r}^{T}  \bar{\lambda} (s) ds} \nonumber \\
& = 	e^{ - \int_{0}^{T}  \bar{\lambda} (s) ds}  
	\prod_{k=1}^n \left( \int_{\Delta_k^l}^{\Delta_k^r} \lambda( t, u_k) e^{ \int_t^{\Delta_{k}^r}  \bar{\lambda} (s) ds} dt
		\right).
\end{align*}
This proves the expression \eqref{eq:logL-cont-time-uncert-on-graph}.
\end{proof}

\begin{proof}[Proof of Lemma \ref{lemma:quant-time-lambda(t,u)}]
Under (A3), $\Delta_i = I_{j(i)}$ for some $j(i) \in [N]$, and then  for $t \in I_j $, 
by \eqref{eq:lambda-phi-cont-time-on-graph},
\begin{align*}
   \lambda[k](t, u) 
 & = \mu(u) + \sum_{i, j(i) < j } \frac{1}{ h  }  \int_{I_{j(i)}}  k( t', t; u_i, u) dt',
\end{align*}
because $\Delta_i^r < t$ if and only if $j(i) < j$.
By Assumption \ref{assump:A4},  for $t' \in I_{j(i)}$ and $t \in I_j$,  $k(t', t; u_i, u) \equiv K_{ j(i), j} (u_i, u)$.
Thus,
\[
   \lambda[k](t, u)  =  \mu(u) + \sum_{i, j(i) < j }  K_{ j(i), j} (u_i, u), \quad \forall t \in I_j.
\]
The r.h.s. is a constant independent of $t$, and this proves the lemma. 
\end{proof}
\begin{proof}[Proof of Lemma \ref{lemma:l-lambda-base-on-graph}]
Under (A3), we have \eqref{eq:logL-cont-time-uncert-on-graph-A3} hold, namely
\[
\ell = \sum_{k=1}^n \log \left( 
\int_{I_{j(k)}^l}^{I_{j(k)}^r} \lambda( t, u_k) e^{ \int_t^{I_{j(k)}^r}  \bar{\lambda} (s) ds} dt
		\right) - \int_0^T \bar{\lambda}(s) ds.
\]		
 Applying Lemma \ref{lemma:quant-time-lambda(t,u)}, 
we have 
\[
\int_{I_{j(k)}^l}^{I_{j(k)}^r} \lambda( t, u_k) e^{ \int_t^{I_{j(k)}^r}  \bar{\lambda} (s) ds} dt
= \Lambda_{j(k)} (u_k)  \int_{0}^{ h}  e^{  \bar{\Lambda}_{j(k)}  (h- t)} dt
=  \frac{ \Lambda_{j(k)} (u_k) }{ \bar{\Lambda}_{j(k)} }  (   e^{ h  \bar{ \Lambda}_{j(k)}  } - 1), 
\]
and meanwhile,
\[
\int_0^T \bar{\lambda}(s) ds   = \sum_{j=1}^N h \bar{\Lambda}_j. 
\]
This proves the expression of $\ell$ as \eqref{eq:l-lambda-base-on-graph}.
\end{proof}

\begin{proof}[Proof of Lemma \ref{lemma:on-graph-bernoulli-glm}]
The expression \eqref{eq:Lambdaj-base-on-graph-bernoulli} implies that $\mathbf{\Lambda}_t $ is in $\sigma\{ \mathbf{y}_i, i\le  t-1 \}$. 
From the expression of the log-likelihood \eqref{eq:l-lambda-base-on-graph-beroulli},
one can verify that 
\[
\Pr [ \bar{y}_t = 1| \mathbf{y}_i, i\le  t-1 ] = 1- e^{- h \bar{\Lambda}_{t}},
\]
and for each $u\in\calV$, 
\[
\Pr [  y_t(u) =1 | \mathbf{y}_i, i\le  t-1, \bar{y}_t = 1] = \frac{ \Lambda_{t} (u) }{ \bar{\Lambda}_{t}}.
\]
As a result,
\[
\Pr [ {y}_t (u) = 1|  \mathbf{y}_i, i\le  t-1 ] 
= ( 1- e^{- h \bar{\Lambda}_{t}} )\frac{ \Lambda_{t} (u) }{ \bar{\Lambda}_{t}}
= ( 1- e^{- h  \mathbf{\Lambda}_t^T \mathbf{1}  } )\frac{ \Lambda_{t} (u) }{ \mathbf{\Lambda}_t^T \mathbf{1}  },
\] 
where in the second equality we used that $ \bar{\Lambda}_t = \mathbf{\Lambda}_t^T \mathbf{1} $ by definition \eqref{eq:def-bar-Lambda}.
Writing in vector form, this gives
$\E [ \mathbf{y}_t | \mathbf{y}_i, i\le  t-1 ] = \Phi(\boldsymbol x)$ with $\boldsymbol x = h \mathbf{\Lambda}_t \in \R_+^V$.  
\end{proof}

\section{Additional experimental details }
\label{apdx:additional-exp}

\begin{algorithm}
\DontPrintSemicolon
\SetAlgoNoLine
\caption{ Stochastic batch-based training for network case}
\label{apdx:algorithm:train-graph}
\SetKwInOut{Input}{Input}
\SetKwInOut{Output}{Output}

\Input{
Training trajectories $\{ \mathbf{y}_t^{(m)} \}_{m=1}^{M}$, 
$\boldsymbol \mu_0$ as initial value of $\boldsymbol \mu$;
Parameters: 
intensity lower bound $b$, barrier weight $\delta_{b}$ ; 
batch size $M_B$, 
maximum number of epochs $k_{\rm max}$, 
learning rate schedule $\{ \gamma_k \}_{k}$;
Optional: singular value threshold $\tau_\SVD$, smoothness weight $\delta_s$
}

\Output{
Learned kernel $\theta_K$ and baseline $\boldsymbol \mu$ }

\vspace{+5pt}

Initialize kernel $\theta_K \leftarrow 0$  and baseline on graph $\boldsymbol \mu \leftarrow \boldsymbol \mu_0$

\For{$k = 1,\ldots, k_{\rm max}$}{

	 \While{loop over the batches}{
	 
	 $\theta \leftarrow [\boldsymbol \mu, \theta_K]$
	
	 Load $M_B$ many training trajectories to form batch $Y_B =\{\mathbf{y}^{(m)} \}_{m}$
	
	  Compute intensities $\mathbf{\Lambda}_t^{(m)}[ \theta ]$ in \eqref{eq:lambdatm-z-graph} at $t=1, \cdots, N$ for  each $\mathbf{y}^{(m)}$ in $Y_B$, and let   $b^{(m)} \leftarrow \min_{ 1 \le t \le N, \, u \in \calV} \Lambda_t^{(m)}(u)$
	  
	   Split $Y_B = Y_{B,1}\cup Y_{B,2}$, where $\mathbf{y}^{(m)}$ in $Y_{B,1}$ if $b^{(m)} < b$, and otherwise in $Y_{B,2}$

	   \uIf{using VI update}{
	    $g_B \leftarrow \frac{1}{M_B}
        		    \sum_{\mathbf{y}^{(m)} \in Y_{B,2}} \hat G^{(m)}[\theta]$,
	    	$\hat G^{(m)}$ as in \eqref{eq:formula-Gtheta-graph}             
            }
    \ElseIf{using GD update}{
    $g_B \leftarrow \frac{1}{M_B}
            \sum_{\mathbf{y}^{(m)} \in Y_{B,2}} \hat F^{(m)}[\theta]$,
     $\hat F^{(m)}$ as in \eqref{eq:expression-hatFm-graph}
            }

    \If{$Y_{B,1}$ not empty}{
    $g_B \gets g_B + \delta_b \sum_{ \mathbf{y}^{(m)} \in Y_{B,1}} \partial_{\theta_K} B^{(m)}(\theta)$, $B^{(m)}$ as in \eqref{eq:def-Bm-barrier-network}
    } 
    
     Update 
    $\theta_K \gets \theta_K - \gamma_k g_B$

    For each $u \in\calV$, solve $\tilde \mu(u)$ as root of \eqref{eq:eqn-elim-mu-graph} on $Y_B$ using bisection search

    Update $\mu(u) \gets 0.9 \mu(u) +0.1\tilde \mu(u)$, $\forall u \in \calV$
    }

	(Optional)
    \If{do smoothness regularization of kernel}{
    $\theta_K \gets \theta_K - \gamma_k \delta_s \partial_{\theta_K} \sum_{u', u \in \calV} S( [ \Psi_{i,l} (u',u)]_{i, l} )$, 
    	$S$ as in \eqref{eq:def-Stheta-smoothness}
    }
 }

(Optional) \If{do low-rank truncation of kernel}{

	Rearrange entries in $\theta_K$ into $\{ [ \Psi_{i,l}(u',u)]_{i,l}, \, u', u \in \calV \}$  and 
    	concatenate matrices 
	$\mathbf{\Psi} \gets \left[ \Psi(u',u)\right]_{u', u \in \calV}$, $\mathbf{\Psi}$ has size $(N+N')V^2 \times N'$

	For each $(u',u)$, $\Psi'(u',u) \leftarrow \Psi_{(N'+1):N,1:N'}(u',u)$, and 
    	concatenate matrices 
	$\mathbf{\Psi}' \gets \left[ \Psi'(u',u)\right]_{u', u \in \calV}$, $\mathbf{\Psi}'$ has size $(N-N'+1)V^2 \times N'$
      
    	$r \leftarrow $ number of singular values of $\mathbf\Psi'$ larger than $\tau_\SVD$

	 $V \leftarrow $ first $r$ right singular vectors of $\mathbf\Psi'$
	 
	  $\mathbf{\Psi} \gets \mathbf{\Psi}{V}{V}^\top$, and rearrange $\mathbf{\Psi}$ into $\theta_K$
}

\KwRet{$\boldsymbol \mu$, $\theta_K$}
\end{algorithm}

\subsection{On the time-only data in Section \ref{subsec:time-only-example}}

\subsubsection{Choice of hyperparameters}\label{apdx:choice-hyperparameters-time-only}

In the experiment with $N=32$ and $N' = 8$:
The batch size $M_B=400$,
and the maximum number of epochs $k_{\rm max} = 300$. 
Learning rate: 
for \texttt{TULIK-VI}, the schedule is set as $\gamma_k = 0.4$ if $k\le 100$ and $\gamma_k = 0.2$ if $ 100 < k \le 300 $;
for \texttt{TULIK-GD}, the schedule is set as $\gamma_k = 0.2$ if $k\le 100$ and $\gamma_k = 0.1$ if  $100 < k \le 300 $.
Regularization:
We use the quadratic barrier in \eqref{eq:def-Bm-barrier}, the intensity lower bound $b=0.01$, 
the barrier weight $\delta_b=0.1$.
We apply the smoothness penalty \eqref{eq:def-Stheta-smoothness}
with smoothness weight $\delta_s = 0.08$. 

In the experiment with $N = 320$ and $N'=80$:
all the choices of optimization hyperparameters are the same as the experiment with $N = 32$ and $N' = 8$ except for  $\delta_s = 0.004$.
We apply low-rank truncation after the 
stochastic optimization 
loops.
We select the singular value threshold $\tau_{\rm SVD} \in\{0.2, 0.4, \ldots, 1.2\}$ by minimizing the $\ell_1$ probabilistic prediction errors on a validation set with $500$ trajectories.
This gives $\tau_{\rm SVD} = 0.6$ and $0.8$ for \texttt{TULIK-VI} and \texttt{TULIK-GD} respectively.
The truncated kernel matrices have rank 3 for both cases and are plotted in Figure \ref{fig:TULIK-large-kernel-time-only}.

\subsubsection{Training dynamics for the data with \texorpdfstring{$N=32$ and $N'=8$}{N=32andNprime=8}}\label{apdx:training-dynamics-time-only}
\begin{figure}[b!]
\centering 
\begin{subfigure}[h]{0.4\linewidth}
\includegraphics[width=\linewidth]{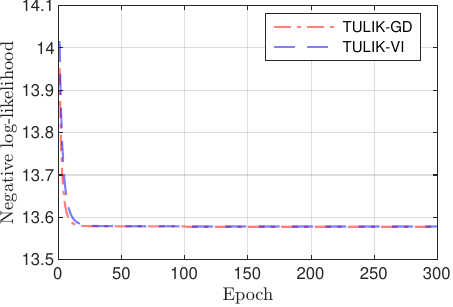}
\caption{negative log-likelihood over epochs}
\label{fig:TULIK-nonstationary-kernel-time-only-nll}
\end{subfigure}
\hspace{+10pt}
\begin{subfigure}[h]{0.4\linewidth}
\includegraphics[width=\linewidth]{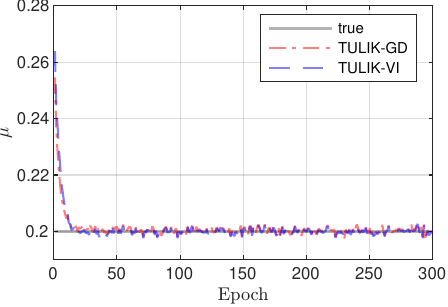}
\caption{trained $\mu$ over epochs}
\label{fig:TULIK-nonstationary-kernel-time-only-mu-dynamics}
\end{subfigure}
\hspace{+0pt}
\vspace{-5pt}
\caption{
Training dynamics on simulated time-only event data with kernel function defined in \eqref{eq:time-only-kernel}, $N'=8$ and $N=32$. 
(a) The training negative log likelihood over epochs of \texttt{TULIK-GD} and \texttt{TULIK-VI}. (b) The $\mu$ learned by \texttt{TULIK-GD} and \texttt{TULIK-VI} over epochs. The true value of $\mu$ is marked by the solid line at $0.2$.
}
\label{fig:TULIK-nonstationary-kernel-time-only-dynamics}
\end{figure}
Figure~\ref{fig:TULIK-nonstationary-kernel-time-only-nll} shows the descending dynamics of training negative log-likelihood functions for both \texttt{TULIK-VI} and \texttt{TULIK-GD}. Meanwhile, Figure~\ref{fig:TULIK-nonstationary-kernel-time-only-mu-dynamics} shows that the proposed scheme to learn $\mu$ also numerically converges to the truth $\mu=0.2$ in both \texttt{TULIK-VI} and \texttt{TULIK-GD}. 

\subsubsection{Alternative baselines}\label{apdx:detail-baselines-time-only}

Note that in the GLM model \cite{juditsky2020convex}, one can allow \texttt{GLM-L} and \texttt{GLM-S} to have non-stationary coefficients. 
In GLM, the standard approach to infer parameters is to construct the likelihood function and recover the parameters as MLE \cite{mccullagh2019generalized}. 
In the work proposing \texttt{GLM-S} and \texttt{GLM-I}, GLMs are used as part of the model for the Bernoulli processes and they train models using VI approach, which is originally introduced in \cite{juditsky2019signal} to adapt to problems with weaker assumptions than the convexity of the likelihood function.

For GLM-based baselines, the $L^2$ loss is formally defined as 
$L_{\rm GLM}(\theta)   = \frac{1}{M} \sum_{m=1}^M \ell_{\rm GLM}^{(m)}(\theta)$, where 
\[
\ell_{\rm GLM}^{(m)}(\theta)
  = \frac{1}{2}\sum_{t=1}^N  ( \phi_{\rm GLM}(  \Lambda_t^{(m)}(\theta)) - y_t^{(m)} )^2,
\]
and the VI field can be derived similarly as in Section \ref{subsec:VI-recovery-timeonly}. The link function $\phi_{\rm GLM}(x)$ can be selected from the linear function $x\mathbb I_{[0,1]}(x)$, where $\mathbb I_{[0,1]}(x)$ is an indicator function on $[0,1]$, or the sigmoid function $1/(1+e^{-x})$, corresponding to \texttt{GLM-L} or \texttt{GLM-S}. 

The \texttt{HP-E} method trains on the continuous-time data, which, in this experiment, are converted from data with uncertainty by assigning $t_i = jh$ for the $i$-th $y_j=1$.   For the learning of \texttt{HP-E}, since we are using the exponential time-invariant kernel defined in Section~\ref{subsec:prelim-hawkes}, we solve the gradient for $\mu$, $\alpha$, and $\beta$ and optimize using SGD.

\subsection{Simulated time-only event data with a time-invariant kernel}\label{adpx:time-only-stationary}

\begin{figure}[b!]
\centering 
\begin{subfigure}[h]{0.45\linewidth}
\includegraphics[width=\linewidth]{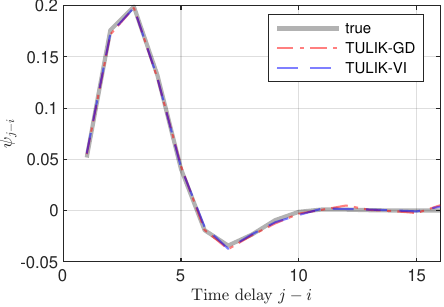}
\end{subfigure}
\vspace{-10pt}
\caption{True time-invariant kernel and recovered kernels on simulated time-only event data with $N'=16$ and $N=32$. 
True time-invariant kernel vector $\psi$ is visualized as a solid curve. The kernels recovered by \texttt{TULIK-VI} and \texttt{TULIK-GD} are shown as a dashed curve and a dash-dotted curve, respectively.}
\label{fig:stationary-kernel-time-only}
\end{figure}

\begin{figure}[b!]
\centering 
\begin{subfigure}[h]{0.80\linewidth}
\includegraphics[width=\linewidth]{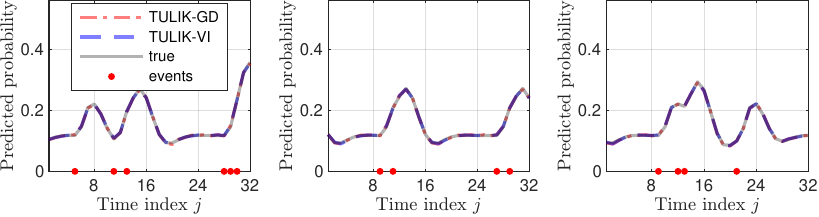}
\end{subfigure}
\vspace{-5pt}
\caption{
The probability predictions on 3 randomly selected testing sequences.
Data is the simulated time-only event data with time-invariant kernel vector $\psi$ defined in Figure~\ref{fig:stationary-kernel-time-only} with $N'=16$ and $N=32$.}
\label{fig:prob-pred-time-only-stationary-kernel}
\end{figure}

\begin{table}[th!]
    \caption{
    Relative Errors (REs) of learned $\mu$, 
    kernel matrix, 
    and probability predictions,
    on simulated time-only event data with the kernel vector $\psi$ defined in Figure~\ref{fig:stationary-kernel-time-only} with $N'=16$ and $N=32$.
    The format are the same as in Table~\ref{tab:time-only-smallkernel}.
    }
    \label{tab:time-only-stationary-kernel}
    \vspace{-5pt}
    \begin{subtable}[h]{1\textwidth}
        \centering
        \resizebox{1\columnwidth}{!}{%
	\begin{tabular}{c cc cc ccccc}
        \toprule
        \multicolumn{10}{c}{time-only event data with discrete time grids $N'=16$ and $N=32$}\\ 
        \cmidrule(lr){2-10}
        RE $\times 10^{-2}$ & \multicolumn{2}{c}{$\mu$} & \multicolumn{2}{c}{kernel} & \multicolumn{5}{c}{prediction}\\
        \cmidrule(lr){2-3}  \cmidrule(lr){4-5}  \cmidrule(lr){6-10} 
        & {\texttt{TULIK-VI}} & {\texttt{TULIK-GD}} & {\texttt{TULIK-VI}} & {\texttt{TULIK-GD}} & {\texttt{TULIK-VI}} & {\texttt{TULIK-GD}} 
        & {GLM-L} & {GLM-S} & {HP-E} \\
			 \cmidrule(lr){2-3} \cmidrule(lr){4-5}  \cmidrule(lr){6-10} 
			 \multirow{2}{*} {$\ell_1$} &0.70 &0.60 &10.73 &9.99 &{2.00} &{\bf1.89} &{2.28} &5.26 &47.19\\

 & (0.58) & (0.44) & (3.49) & (2.80) & (0.64) & (0.51) & (0.51) & (0.16) & (0.94)\\[3pt]

\multirow{2}{*} {$\ell_2$} & -- &0.60 &7.52 &7.02 &{2.43} &{\bf2.28} &{2.80} &6.36 &52.16\\

 &  & (0.44) & (2.51) & (2.06) & (0.77) & (0.62) & (0.63) & (0.31) & (0.94)\\[3pt]

\multirow{2}{*} {$\ell_\infty$} & -- &0.60 &5.86 &5.46 &{3.58} &{\bf3.36} &{4.32} &9.68 &66.53\\

 &  & (0.44) & (2.36) & (2.11) & (1.10) & (0.90) & (1.12) & (1.07) & (0.95)
			\\\bottomrule
		\end{tabular}
  }%
    \end{subtable}
\end{table}

\paragraph{Dataset.}

We set discrete time grids as $N'=16$ and $N=32$ and define a true time-invariant kernel vector $\psi$ visualized as the gray curve in Figure~\ref{fig:stationary-kernel-time-only}. Given $\psi$ and $\mu=0.2$, we generate Bernoulli process data from \eqref{eq:time-only-bernoulli-glm}. The training and testing data consist of $4800$ and $500$ trajectories, respectively.

\paragraph{Method.}
We compute the proposed model following Algorithm \ref{algorithm:train-time-only} and use the vector fields for the time-invariant kernel, which have been mentioned in Section~\ref{subsec:special-structure-psi}.
The choice of training hyperparameters is as follows:
We set the batch size $M_B=400$, and the maximum number of epochs $k_{\rm max} = 60$. 
Learning rate: for \texttt{TULIK-VI}, the learning rate schedule is set as $\gamma_k = 0.4$ if $k\le 20$ and $\gamma_k = 0.2$ if $20 <k\le 60$; for \texttt{TULIK-GD}, the learning rate schedule is set as $\gamma_k = 0.2$ if $k\le 20$ and $\gamma_k = 0.1$ if $20 <k\le 60$.
Regularization: We use the quadratic barrier in \eqref{eq:def-Bm-barrier}, the intensity lower bound $b=0.01$, the barrier weight $\delta_b=0.1$.  We apply the smoothness penalty \eqref{eq:def-Stheta-smoothness}
with smoothness weight $\delta_s = 0.004$.

We compare the proposed methods with the baselines including GLM-L, GLM-S, and HP-E, which can also be computed on the stationary example.

\paragraph{Results.}
In this time-invariant kernel example, the kernel $\psi$ and the scalar $\mu$ are well recovered, as illustrated by Figure~\ref{fig:stationary-kernel-time-only}  and the small REs in Table~\ref{tab:time-only-stationary-kernel}. 
The probability predictions are also accurate, as shown in Figure~\ref{fig:prob-pred-time-only-stationary-kernel} and Table~\ref{tab:time-only-stationary-kernel}. 
In comparison, TULIK methods outperform the other baselines, and in this case, the advantage over \texttt{GLM-L} is mild (within one standard deviation).

\subsection{On the on-network data in Section \ref{subsec:on-network-example}}

\subsubsection{Choice of hyperparameters}\label{apdx:choice-hyperparameters-on-network}

We use the batch size $M_B=800$ and the maximum number of epochs $k_{\rm max} = 150$. 
Learning rate:   
for \texttt{TULIK-VI}, the learning rate schedule is set as $\gamma_k = 0.4$ if $k\le 50$ and $\gamma_k = 0.2$ if  $50 < k \le 150 $; 
for \texttt{TULIK-GD}, the learning rate schedule is set as $\gamma_k = 0.2$ if $k\le 50$ and $\gamma_k = 0.1$ if  $50 < k \le 150 $. 
Regularization: 
we use the on-network quadratic barrier \eqref{eq:def-Bm-barrier-network} with the intensity lower bound $b=0.03$ and the barrier weight $\delta_b=0.1$.
We adopt the smoothness weight $\delta_s=0.004$. 

After the 
stochastic optimization
loops, we apply low-rank truncation to the kernel matrix as in Algorithm \ref{apdx:algorithm:train-graph}. We select $\tau_{\rm SVD}$ from $\{0.2, 0.4, \ldots, 1.2\}$ by  minimizing the $\ell_1$ probabilistic prediction errors on a validation set with $500$ trajectories, and obtain $\tau_{\rm SVD} = 0.8$ and $1.0$ for \texttt{TULIK-VI} and \texttt{TULIK-GD} respectively.
The truncated matrices $\mathbf{\Psi}$ have rank-2 in both cases and are visualized in Figure \ref{fig:TULIK-kernel-on-network}.

\subsubsection{Training dynamics}\label{apdx:TULIK-on-network}

Figure~\ref{fig:TULIK-on-network-nll} shows the convergence of training log-likelihood functions for both \texttt{TULIK-VI} and \texttt{TULIK-GD}. 
Additionally, Figure~\ref{fig:TULIK-on-network-mu-dynamics} validates the on-network approach to learning $\mu$ also numerically converges around $7\%$ relative errors. Such relative errors correspond to absolute errors on each node around $0.01$.

\begin{figure}[t!]
\centering 
\begin{subfigure}[h]{0.4\linewidth}
\includegraphics[width=\linewidth]{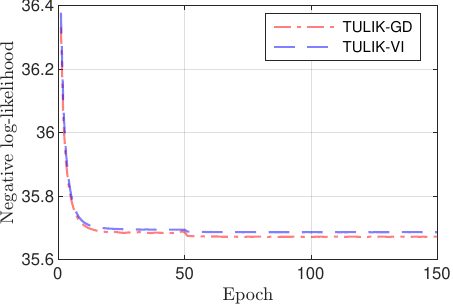}
\caption{negative log-likelihood over epochs}
\label{fig:TULIK-on-network-nll}
\end{subfigure}
\hspace{+10pt}
\begin{subfigure}[h]{0.4\linewidth}
\includegraphics[width=\linewidth]{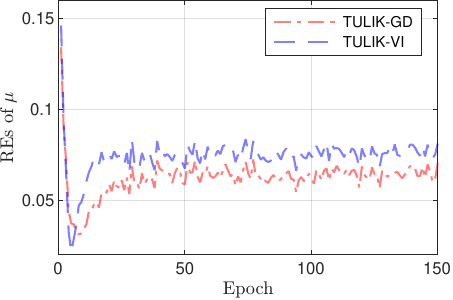}
\caption{trained $\mu$ over epochs}
\label{fig:TULIK-on-network-mu-dynamics}
\end{subfigure}
\hspace{+0pt}
\vspace{-5pt}
\caption{For simulated on-network event data in Section~\ref{subsec:on-network-example}, we train the model on 40000 trajectories using both \texttt{TULIK-VI} and \texttt{TULIK-GD}. The results from \texttt{TULIK-VI} and \texttt{TULIK-GD} are visualized as dashed curves and dash-dotted curves, respectively. (a) The dynamics of training negative log-likelihood over epochs. (c) The dynamics of REs of learned $\mu$ over epochs. 
}
\label{fig:TULIK-on-network-dynamics}
\end{figure}

\begin{figure}[t!]
\centering 
\begin{subfigure}[h]{0.4\linewidth}
\includegraphics[width=\linewidth]{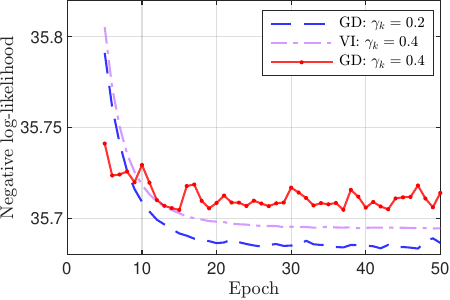}
\end{subfigure}
\hspace{+10pt}
\hspace{+0pt}
\vspace{-5pt}
\caption{
Data is the on-network event data in Section~\ref{subsec:on-network-example}.
On the same training set, we increase the learning rate of GD to be 0.4 which is the same as VI, then the training negative log-likelihood oscillates and converges worse, showing instability in training. 
}
\label{fig:TULIK-on-graph-GD-instability}
\end{figure}

\subsubsection{Alternative baselines}\label{apdx:detail-baselines-on-network}
Note that for on-network event data, in both GLM-L and GLM-S, instead of the link functions, they use link vector field, which is the same link function as defined in time-only examples but implemented on each node, to compute the prediction probability. 
Their per-trajectory $L^2$ loss can be formally written as
\[
\ell_{\rm GLM}^{(m)}( z )
  = \frac{1}{2 }\sum_{t=1}^N  \| \Phi_{\rm GLM}( \mathbf{\Lambda}_t^{(m)}[z] ) - \mathbf{y}_t^{(m)} \|_2^2,
\]
where $\Phi_{\rm GLM}(\mathbf {x}) = \left(\phi_{\rm GLM}(x(u))\right)_{u\in\calV} $. 
The VI monotone field used to update $z$ can be derived similarly as in Section \ref{subsec:kernel-recovery-network}.

For the learning of continuous-time multivariate HP-E, the kernel function becomes $k(t',t;u',u) = \alpha(u',u)\beta e^{-\beta(t-t')}$ and we perform SGD to learn $\boldsymbol{\mu}, (\alpha(u',u))_{u',u\in\calV}$ and $\beta$.

\subsubsection{Instability in GD training}\label{apdx:GD-instability}

Figure~\ref{fig:TULIK-on-graph-GD-instability} shows the dynamics of training negative log-likelihoods 
(from epoch 5 to epoch 50)
for VI with a learning rate $0.4$,
and GD with a learning rate $0.2$ and 0.4 respectively. 
It can be seen that, when the learning rate of GD is set to be the same as VI, the learning dynamic shows instability. 


\subsection{More on SADs data experiment in Section \ref{subsec:more-on-sepsis}}\label{apdx:SADs}

We provide additional details for the SADs data and its results.

\subsubsection{Data pre-processing}\label{apdx:data-processing-sepsis}

\begin{table}[t!]
\caption{SAD constructions based on raw observations thresholding and measurements grouping. For some measurement, if its raw observation exceed corresponding threshold, then its risk score will be added to the total risk score of corresponding SAD.}\label{tab:SAD-grouping}
\vspace{-15pt}
\begin{center}
\begin{small}
\resizebox{1\textwidth}{!}{%
\begin{tabular}{lllcc}
\toprule[1pt]
SAD name & Measurement name & Physionet name & Threshold & Risk score \\ 
\midrule[0.3pt]
\textbf{Renal Injury}	&	creatinine	&	Creatinine	&	$>$1.3	&	0.667	\\
	&	potassium	&	Potassium	&	$>$5.0	&	0.067	\\
	&	phosphorus	&	Phosphate	&	$>$4.5	&	0.067	\\
	&	bicarb (hco3)	&	HCO3	&	$>$26	&	0.067	\\
	&	blood urea nitrogen (bun)	&	BUN	&	$>$20	&	0.133	\\
\cmidrule(l){2-5}
\textbf{Electrolyte}	&	calcium	&	Calcium	&	$>$10.5	&	0.167	\\
\textbf{Imbalance}	&	chloride	&	Chloride	&	$<$98 or $>$106	&	0.667	\\
	&	magnesium	&	Magnesium	&	$<$1.6	&	0.167	\\
	\cmidrule(l){2-5}
\textbf{Oxygen Carrying}	&	hematocrit	&	Hct	&	$<$37	&	0.500	\\
\textbf{Dysfunction}	&	hemoglobin	&	Hgb	&	$<$12	&	0.500	\\
	\cmidrule(l){2-5}
\textbf{Shock}	&	base excess	&	BaseExcess	&	$<-$3	&	0.100	\\
	&	lactic acid	&	Lactate	&	$>$2.0	&	0.150	\\
	&	ph	&	pH	&	$<$7.32	&	0.750	\\
	\cmidrule(l){2-5}
\textbf{Diminished }	&	sbp cuff	&	SBP	&	$<$120	&	0.250	\\
\textbf{Cardiac Output}	&	dbp cuff	&	DBP	&	$<$80	&	0.250	\\
	&	map cuff	&	MAP	&	$<$65	&	0.500	\\
	\cmidrule(l){2-5}
\textbf{Coagulopathy}	&	partial prothrombin time (ptt)	&	PTT	&	$>$35 	&	0.250	\\
	&	fibrinogen	&	Fibrinogen	&	$<$233	&	0.250	\\
	&	platelets	&	Platelets	&	$<$150,000	&	0.500	\\
	\cmidrule(l){2-5}
\textbf{Cholestasis}	&	bilirubin direct	&	Bilirubin direct	&	$>$0.3	&	0.100	\\
	&	bilirubin total	&	Bilirubin total	&	$>$1.0	&	0.500	\\
	&	alkaline phosphatase	&	Alkalinephos	&	$>$120	&	0.400	\\
	\cmidrule(l){2-5}
\textbf{Hepatocellular Injury}	&	aspartate aminotransferase (ast)	&	AST	&	$>$40	&	1.000	\\
	\cmidrule(l){2-5}
\textbf{Oxygenation }	 &	saturation of oxygen (sao2)	&	SaO2	&	$<$92 \%	&	0.500	\\
\textbf{Dysfunction (lab)}	 &	end tidal co2	&	EtCO2	&	$<$35 or $>$45	&	0.250	\\
	&	partial pressure of carbon dioxide (paco2)	&	PaCO2	&	$<$35 or $>$45	&	0.250	\\
	\cmidrule(l){2-5}
\textbf{Inflammation (lab)}	&	glucose	&	Glucose	&	$>$125	&	0.200	\\
	&	white blood cell count	&	WBC	&	$<$4,000 or $>$12,000	&	0.800	\\
	\cmidrule(l){2-5}
\textbf{Oxygenation }	&	unassisted resp rate	&	Resp	&	$>$20	&	0.167	\\
\textbf{Dysfunction (vital)}	&	spo2	&	O2Sat	&	$<$92 \%	&	0.333	\\
	&	fio2	&	FiO2	&	$>$21 \%	&	0.500	\\
	\cmidrule(l){2-5}
\textbf{Inflammation (vital)}	&	temperature	&	Temp	&	$<$36 or  $>$38	&	0.800	\\
	&	pulse	&	HR	&	$>$90	&	0.200	\\
\bottomrule[1pt]
\end{tabular}
}
\end{small}
\end{center}
\vspace{-0.25in}
\end{table}

Table~\ref{tab:SAD-grouping} introduces the rules to group measurements into SADs and compute risk scores for SADs. 
The Bernoulli process $y_t(u)$ on the 13-node network is obtained from the risk scores on the first 12 nodes, and using spesis onset-or-not on the last node.
Specifically, for the first 12 nodes, risk scores are assigned to indicate the severity of corresponding medical symptoms. Then for $u\in\{1,\cdots,12\}$, namely the nodes except for sepsis, $y_t(u)=1$ means the risk score of node $u$ increases at hour $t$. 

\subsubsection{Evaluation metrics}\label{apdx:eva-sepsis}

We evaluate the accuracy of sepsis onset prediction using the TPR, TNR, and BA. Since all patients in the processed data were eventually diagnosed as sepsis patients, the positive label cannot be defined simply based on sepsis onset. Instead, following the definition in \cite{physionetChallenge}, we consider a meaningful early prediction window \((t-12, t]\) for a sepsis onset at time \(t\). A patient is classified as {\it positive} if this prediction window intersects with the time horizon \((0, 20]\), ensuring that early prediction is feasible within the given time horizon.  
Specifically, patients with sepsis onset at \(t < 32\) are labeled as positive, while those with onset at \(t \geq 32\) are labeled as negative, as meaningful early prediction is not possible within the current time horizon for very late onsets.  
For a positive patient, a prediction is considered a true positive if the first predicted sepsis onset occurs within the early prediction window \((t-12, t] \cap (0, 20]\); otherwise, it is classified as a false positive.  
For a negative patient, a prediction is considered a true negative if no sepsis onset is predicted within \((0, 20]\); otherwise, it is classified as a false negative.

\subsubsection{Choice of hyperparameters}\label{apdx:choice-hyperparameters-sepsis}
\begin{figure}[t!]
\centering 
\begin{subfigure}[h]{0.65\linewidth}
\includegraphics[width=\linewidth]{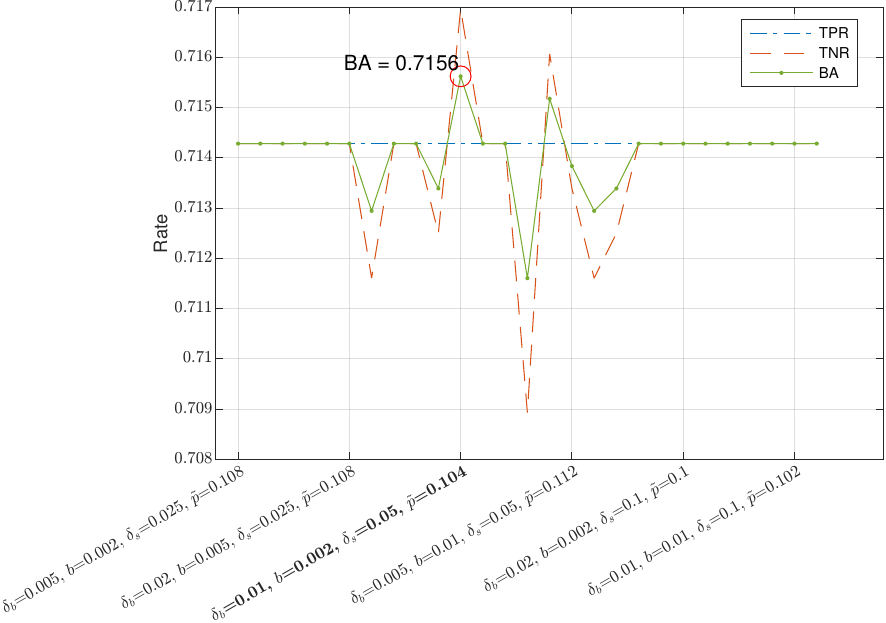}
\end{subfigure}

\vspace{-10pt}
\caption{ After grid search in Section~\ref{subsec:more-on-sepsis}, the validation curve for SADs data learned by \texttt{TULIK-VI}. The curves with circle marker, with cross marker and with dot marker show TPR, TNR and BA over combinations of hyperparameters. The hyperparameters yielding largest BA on the validation set are bolded and the corresponding metrics on each curve are marked by red circles and texts. 
}
\label{fig:sepsis-cv}
\end{figure}
We initialize the kernel tensor by filling all zero. The initialization of on-network $\mu$ can refer to Section~\ref{subsec:algo-on-network}. 
To get reasonable results, we randomly leave one validation set containing 70 trajectories out of the training data. The remaining 1000 training trajectories will be used to update the kernel tensor and $\mu$.  We set the batch size $M_B=107$, the maximum number of epochs $k_{\rm max} = 800$, and the learning rate schedule $\gamma_k = 0.4$ for all $k\le 800$.
We grid search the intensity lower bound $b\in\{0.002, 0.005,0.01\}$, the barrier weight $\delta_b\in\{0.005, 0.01, 0.02\}$ in the quadratic loss, and the smoothness weight $\delta_s\in\{0.05, 0.1, 0.2\}$. Figure~\ref{fig:sepsis-cv} shows the validation curve for the search of optimal hyperparameters in the grid. We pick the hyperparameters that have the largest BA across different combinations. During the training, violations occur more frequently than the simulations but most of them will be corrected by the barrier penalty. For each set of hyperparameters, we compute the probability predictions on the validation set after training. Denote $\hat p_t(u)$ as the prediction of $\E  [ y_t(u) | \mathbf{y}_i, i \le  t-1 ]$ in each method. We search for the probability threshold $\tilde p$ between $[0,1]$ to convert $\hat p_t(u)>\tilde p$ into ones and $\hat p_t(u)\le \tilde p$ into zeros to minimize the absolute difference between TPR and TNR on the validation set.  Then we select the hyperparameters yielding largest BA on the validation set for the eventual training on the whole training data. 


%
%
%

\subsection{On burglary crime data in Section \ref{subsec:burglary-crime}}\label{apdx:burglary}


We introduce more details for the burglary data experiments.

\subsubsection{Data pre-processing}\label{apdx:data-processing-burglary}
As shown in Figure~\ref{fig:atl-downtown-burglary}, the Atlanta downtown area is uniformly divided into 16 sub-regions, leading to 16 nodes in the point process model. The indices of the sub-region follow a latitude-major ascending order (bottom to top, left to right in Figure~\ref{fig:atl-downtown-burglary}).
The Bernoulli process $y_t(u)$ on the 16-node network is determined by burglary-or-not within the hour $t$ at the sub-region $u$.

\subsubsection{Choice of hyperparameters}\label{apdx:choice-hyperparameters-burglary}

\begin{figure}[t!]
\centering 
\begin{subfigure}[h]{0.65\linewidth}
\includegraphics[width=\linewidth]{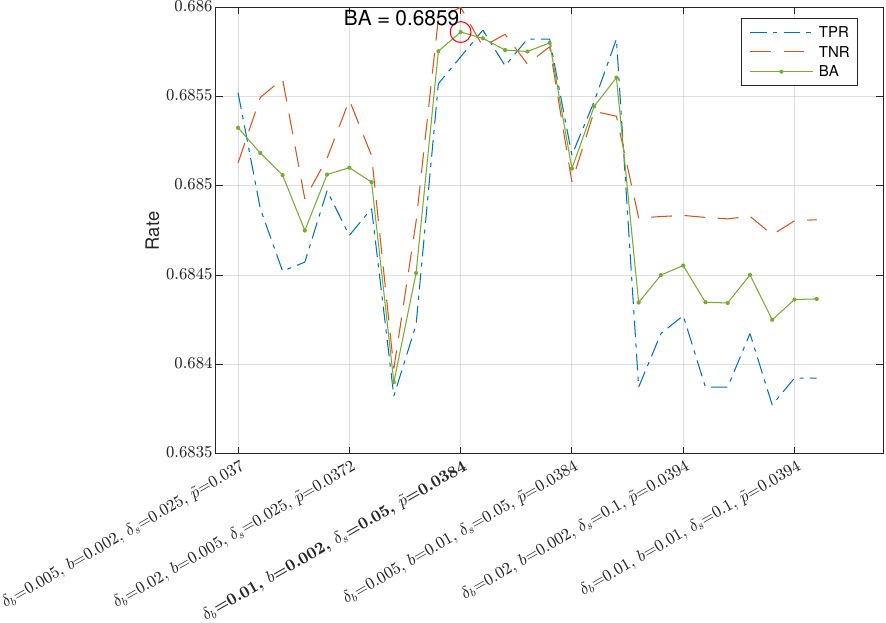}
\end{subfigure}

\vspace{-10pt}
\caption{ Same plot as in Figure~\ref{fig:sepsis-cv}. After grid search in Section~\ref{subsec:burglary-crime}, the validation curve for burglary crime data learned by \texttt{TULIK-VI}. 
}
\label{fig:burglary-cv}
\end{figure}
We initialize the kernel tensor as a zero tensor. The initialization of on-network $\mu$ is introduced in Section~\ref{subsec:algo-on-network}. 
To obtain reasonable results, we randomly leave out a validation set with 250 trajectories from the training data. The remaining 2000 training trajectories are used to learn the kernel tensor and $\mu$. We set the batch size to $M_B = 400$, the maximum number of epochs to $k_{\rm max} = 800$, and the learning rate schedule to $\gamma_k = 0.2$ for all $k \leq 800$.  
We perform a grid search over the intensity lower bound $b \in \{0.002, 0.005, 0.01\}$, the barrier weight $\delta_b \in \{0.005, 0.01, 0.02\}$ in the quadratic loss, and the smoothness weight $\delta_s \in \{0.05, 0.1, 0.2\}$. Figure~\ref{fig:burglary-cv} presents the validation curve for the grid search of optimal hyperparameters. The hyperparameters that yield the highest BA across different combinations are selected.  
The search for the probability threshold to convert the probability predictions to binary observations follows the same process as described in Section~\ref{apdx:choice-hyperparameters-sepsis}.

\end{document}